\documentclass[twoside,11pt]{article}

\usepackage{bbm}
\usepackage[table]{xcolor}
\definecolor{darkred}{rgb}{0.6,0,0}
\definecolor{darkgreen}{rgb}{0,0.5,0}
\definecolor{darkblue}{rgb}{0,0,0.5}
\definecolor{notered}{HTML}{d62728}
\definecolor{notegreen}{HTML}{2ca02c}
\definecolor{noteblue}{HTML}{1f77b4}
\definecolor{notepurple}{HTML}{7e1e9c}
\usepackage{soul}
\usepackage[framemethod=TikZ]{mdframed}
\usepackage{tcolorbox}
\usepackage{multirow} 
\usepackage{booktabs} 
\usepackage{threeparttable}
\usepackage{longtable}
\usepackage{makecell}
\usepackage{enumitem}
\usepackage[ruled,vlined,linesnumbered]{algorithm2e}
\usepackage{titletoc}

\usepackage{graphicx}
\usepackage{wrapfig}
\usepackage{subcaption}
\usepackage[preprint]{jmlr2e}

\usepackage{amsmath,amsfonts,bm}
\usepackage{mathtools}
\usepackage{nicefrac}

\def\1{\bm{1}}

\def\rvg{{\mathbf{g}}}

\def\rvx{{\mathbf{x}}}
\def\rvy{{\mathbf{y}}}

\DeclareMathAlphabet{\mathsfit}{\encodingdefault}{\sfdefault}{m}{sl}
\SetMathAlphabet{\mathsfit}{bold}{\encodingdefault}{\sfdefault}{bx}{n}

\def\gA{{\mathcal{A}}}
\def\gB{{\mathcal{B}}}

\def\gD{{\mathcal{D}}}
\def\gE{{\mathcal{E}}}

\def\gO{{\mathcal{O}}}

\def\gS{{\mathcal{S}}}

\newcommand{\E}{\mathbb{E}}

\newcommand{\R}{\mathbb{R}}

\newcommand{\norm}[1]{\left\lVert#1\right\rVert} 
\newcommand{\Norm}[1]{\lVert#1\rVert} 
\newcommand{\abs}[1]{\left\lvert#1\right\rvert}
\renewcommand{\Pr}{\mathbb{P}}
\newcommand{\transpose}{\mathrm{T}}
\newcommand{\Term}[2]{\text{Term}_{#1}\ \text{in}\ #2}

\newtheorem{assumption}{Assumption}

\usepackage[noabbrev, capitalize]{cleveref}

\usepackage{lastpage}
\jmlrheading{26}{2025}{1-\pageref{LastPage}}{5/24; Revised 11/24}{3/25}{24-0668}{Yipeng Li and Xinchen Lyu}

\ShortHeadings{Sharp Bounds for Sequential Federated Learning}{Li and Lyu}
\firstpageno{1}

\begin{document}

\title{Sharp Bounds for Sequential Federated Learning \\on Heterogeneous Data}

\author{\name Yipeng Li \email liyipeng@bupt.edu.cn \\
       \addr National Engineering Research Center for Mobile Network Technologies\\
       Beijing University of Posts and Telecommunications\\
       Beijing, 100876, China
       \AND
       \name Xinchen Lyu\thanks{Corresponding author.} \email lvxinchen@bupt.edu.cn \\
       \addr National Engineering Research Center for Mobile Network Technologies\\
       Beijing University of Posts and Telecommunications\\
       Beijing, 100876, China}

\editor{Peter Richt\'{a}rik}

\maketitle

\begin{abstract}
	There are two paradigms in Federated Learning (FL): parallel FL (PFL), where models are trained in a parallel manner across clients, and sequential FL (SFL), where models are trained in a sequential manner across clients. Specifically, in PFL, clients perform local updates independently and send the updated model parameters to a global server for aggregation; in SFL, one client starts its local updates only after receiving the model parameters from the previous client in the sequence. In contrast to that of PFL, the convergence theory of SFL on heterogeneous data is still lacking. To resolve the theoretical dilemma of SFL, we establish sharp convergence guarantees for SFL on heterogeneous data with both upper and lower bounds. Specifically, we derive the upper bounds for the strongly convex, general convex and non-convex objective functions, and construct the matching lower bounds for the strongly convex and general convex objective functions. Then, we compare the upper bounds of SFL with those of PFL, showing that SFL outperforms PFL on heterogeneous data (at least, when the level of heterogeneity is relatively high). Experimental results validate the counterintuitive theoretical finding.
\end{abstract}


\section{Introduction}

Federated Learning (FL) \citep{mcmahan2017communication, chang2018distributed} is a popular distributed machine learning paradigm, where multiple clients collaborate to train a global model, while preserving data privacy and security. Commonly, FL can be categorized into two types: (i) parallel FL (PFL), where models are trained in a parallel manner across clients, with periodic aggregation, such as Federated Averaging (FedAvg) \citep{mcmahan2017communication} and Local SGD \citep{stich2019local}, and (ii) sequential FL (SFL), where models are trained in a sequential manner across clients, such as Cyclic Weight Transfer (CWT) \citep{chang2018distributed} and peer-to-peer FL \citep{yuan2023peer}. A simple illustration of SFL and PFL is presented in Figure~\ref{fig:SFL}.

\begin{figure}[h]
	\centering
	\includegraphics[width=0.9\linewidth]{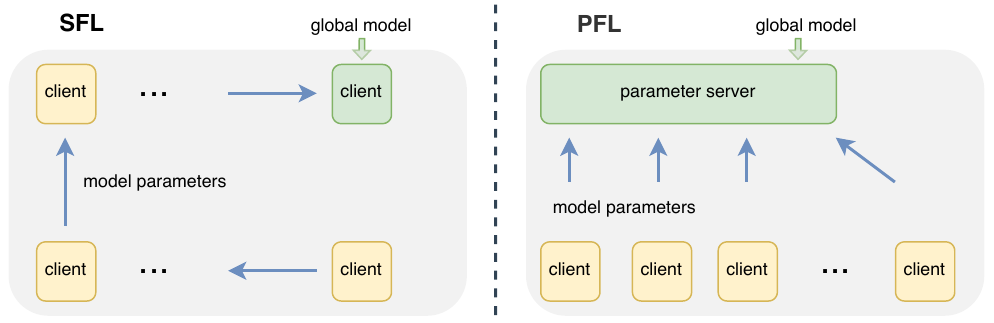}
	\caption{Illustration of SFL and PFL.}
	\label{fig:SFL}
\end{figure}

SFL has recently attracted much attention in the FL community \citep{lee2020tornadoaggregate, yuan2024decentralized} with various applications in medicine \citep{chang2018distributed, huang2024experimental}, automated driving \citep{yuan2023peer} and so on. Compared with PFL, SFL shows the advantages in the following aspects: First, as one of the most popular decentralized FL paradigms \citep{yuan2024decentralized}, \textit{SFL operates in a peer-to-peer manner, avoiding the reliance on a central parameter server.} This is a more practical option for medical applications, as establishing a central server is a costly endeavor \citep{huang2024experimental}. In addition, it avoids single points of failure and bottlenecks in communication, computation, and storage resources found in central servers \citep{yuan2024decentralized}. Second, \textit{SFL is regarded as a network-resilient, communication- and computation-efficient alternative to PFL} \citep{yuan2024decentralized,yan2024sequential}. Recently, SFL has been increasingly combined with PFL to complement each other. For example, in Clustered Federated Learning, where clients are grouped into multiple clusters, \citet{zaccone2022speeding, chen2023fedseq} adopts SFL for clients within each cluster to speed up the training process and reduce communication overhead; \citet{chen2020fedcluster, yan2024sequential} adopts SFL for inter-cluster training to make frequent updates and overcome the shortcomings of network dynamics and instability. Third, \textit{SFL has played a great role in Split Learning (SL) \citep{gupta2018distributed, thapa2022splitfed}}, an emerging distributed learning technology at the network edge, where the full model is split into client-side and server-side portions to alleviate the excessive computation overhead for resource-constrained devices. In SL, one popular way to enable multi-client training is the sequential model training manner, where clients collaborate with the server to perform the updates sequentially (one by one), as the full model is divided between the client side and the server side. Our theory on the convergence of SFL is also applicable to such sequential SL \citep{li2023convergence}.

Both PFL and SFL suffer from ``data heterogeneity'', one of the most persistent problems in FL. Up to now, there have been numerous works to study the convergence of PFL on heterogeneous data \citep{li2020convergence, khaled2020tighter, koloskova2020unified,woodworth2020minibatch}. These theoretical works not only helped understand the effect of heterogeneity, but also spawned new algorithms like SCAFFOLD \citep{karimireddy2020scaffold,mishchenko2022proxskip}. In contrast, the convergence of SFL on heterogeneous data has not been well studied. Recent works \citep{cho2023convergence,malinovsky2023federated} studied the convergence of FL with cyclic client participation, which can be seen as an extension of SFL. However, its convergence analysis is still in an infancy stage, and existing works do not cover the SFL setups in this paper (see Section~\ref{sec:related work}). Our earlier conference paper \citep{li2023convergence} proved the upper bounds of SFL (which is shown to be applicable to SL). However, notably, the lower bounds for SFL are still missing in existing works. The lack of theoretical study can hinder further development of SFL and even SL.

To resolve the theoretical dilemma of SFL, this paper, extending our conference paper \citep{li2023convergence},\footnote{The conference paper is available at \url{https://arxiv.org/abs/2311.03154}. Notably, by default, we use the latest arXiv version for all the references.} aims to establish sharp convergence guarantees for SFL with both upper and lower bounds. In the case of homogeneous data, this task is trivial, as SFL is reduced to SGD (Stochastic Gradient Descent). However, in the case of heterogeneous data, it is more challenging than existing works, including PFL and SGD-RR (Random Reshuffling), primarily due to the following reasons:
\begin{enumerate}[label=({\roman*})]
	\item Sequential and shuffling training manner across clients (vs. PFL). In PFL, the local updates at each client only depend on the randomness of the current client within each training round. However, in SFL, the local updates additionally depend on the randomness of all previous clients.
	\item Multiple local update steps at each client (vs. SGD-RR). In contrast to its with-replacement sibling SGD, SGD-RR samples data samples ``without replacement'' and then performs one step of GD (Gradient Descent) on each data sample. Similarly, SFL samples clients without replacement and then performs multiple steps of SGD at each client. In fact, SGD-RR can be regarded as a special case of SFL.
\end{enumerate}

In this paper, we establish the convergence guarantees for SFL (Algorithm~\ref{algorithm1}), and then compare them with those of PFL (Algorithm~\ref{algorithm2}). The main contributions are as follows:
\begin{itemize}
	\item We derive the upper bounds of SFL for the strongly convex, general convex and non-convex cases on heterogeneous data with the standard assumptions in FL in Subsection~\ref{subsec:upper bound} (Theorem~\ref{thm:SFL} and Corollary~\ref{cor:SFL}).
	\item We construct the lower bounds of SFL for the strongly convex and general convex cases in Subsection~\ref{subsec:lower bound} (Theorem~\ref{thm:lower bound} and Theorem~\ref{thm:lower bound2}). They match the derived upper bounds for the large number of training rounds.
	\item We compare the upper bounds of SFL with those of PFL in Subsections~\ref{subsec:comparison-1} and~\ref{subsec:comparison-2}. In the convex cases,\footnote{For clarity, we use the term ``the convex cases'' to collectively refer to both the strongly convex case and the general convex case in this paper.} the comparison results show a subtle difference under different heterogeneity assumptions. That is, under Assumption~\ref{asm:heterogeneity:optimum}, the upper bounds of SFL are better than those of PFL strictly, while under Assumption~\ref{asm:heterogeneity:max}, the upper bounds of SFL are still better unless the level of heterogeneity is very low. In the non-convex case under Assumption~\ref{asm:heterogeneity:average}, the upper bounds of SFL are better without exception.
	\item The comparison results imply that SFL outperforms PFL in heterogeneous settings (at least, when the level of heterogeneity is relatively high). We then validate this counterintuitive result with experiments on quadratic functions (Subsection~\ref{subsec:exp:QF}), logistic regression (Subsection~\ref{subsec:exp:LR}) and deep neural networks (Subsection~\ref{subsec:exp:DNN}).
\end{itemize}

\section{Related Work}\label{sec:related work}
The most relevant research topics are the convergence analyses of PFL and SGD-RR. 

So far, there have been a wealth of works to study the upper bounds of PFL on data heterogeneity \citep{li2020convergence, khaled2020tighter, karimireddy2020scaffold, koloskova2020unified, woodworth2020minibatch}, system heterogeneity \citep{wang2020tackling}, partial client participation \citep{li2020convergence, yang2021achieving, wang2022unified} and other variants \citep{karimireddy2020scaffold, wang2020tackling, reddi2021adaptive}. The lower bounds of PFL have also been studied in \cite{woodworth2020local, woodworth2020minibatch, yun2022minibatch, glasgow2022sharp}. In this work, we make a comparison between the upper bounds of PFL and those of SFL on heterogeneous data (see Subsections~\ref{subsec:comparison-1} and \ref{subsec:comparison-2}). 

SGD-RR has been gaining significant attention as a more practical alternative to SGD. \cite{nagaraj2019sgd,ahn2020sgd, mishchenko2020random,nguyen2021unified,lu2022general, koloskova2024convergence} have proved the upper bounds and \cite{safran2020good, safran2021random, rajput2020closing, cha2023tighter} have proved the lower bounds of SGD-RR. In particular, to the best of our knowledge, \cite{mishchenko2020random} provided the tightest upper bounds and \cite{cha2023tighter} provided the tightest lower bounds of SGD-RR. In this work, we adopt them to exam the tightness of the convergence bounds of SFL (see Subsection~\ref{subsec:upper bound}).

Recently, the shuffling-based method, SGD-RR, has been applied to FL. One line of these works is Local RR (or FedRR) \citep{mishchenko2022proximal,yun2022minibatch, horvath2022fedshuffle, sadiev2023federated, malinovsky2023federated}, which adopts SGD-RR (instead of SGD) as the local solver. Another line is FL with cyclic client participation \citep{eichner2019semi, wang2022unified, cho2023convergence, malinovsky2023federated}, which can be seen as an extension of SFL. However, its convergence analysis is still in an infancy stage, and existing works do not cover the SFL setups in this paper. \citet{eichner2019semi} considered that clients can form different blocks due to diurnal variation, and propose training seperate models for each of these blocks. This differs from our setting, which aims to train a single global model. In \cite{wang2022unified, cho2023convergence}, when it reduces to SFL, the client training order is deterministic (not random), and thus the analyses cannot be directly extended to our setting. In \cite{malinovsky2023federated}, although their bounds are slightly tighter on the optimization term with SGD-RR as the local solver, their analysis is limited to the case where the number of local steps equals the size of the local data set. Most importantly, \cite{cho2023convergence} considered upper bounds for PL objective functions and \cite{malinovsky2023federated} considered upper bounds for strongly convex objective functions,\footnote{PL condition can be thought as a non-convex generalization of strong convexity.} while we consider both upper bounds (for both convex and non-convex cases) and lower bounds. Detailed comparisons are in Appendix~\ref{app:comparison-shuffling-variance} and \cite{li2023convergence}.

\section{Setup}
\textit{Notation.} We let $[n] \coloneqq \{1,2,\ldots, n\}$ for $n\in \mathbb{N}^+$ and $\{x_i\}_{i\in \gS} \coloneqq \{x_i : i \in \gS\}$ for any set $\gS$. We use $\abs{\gS}$ to denote the size of any set $\gS$. We use $\lesssim$ to denote ``less than'' up to some absolute constants and polylogarithmic factors, and $\gtrsim$ and $\asymp$ are defined likewise. We also use the big O notations, $\tilde O$, $\gO$, $\Omega$, where $\gO$, $\Omega$ hide numerical constants, $\tilde\gO$ hides numerical constants and polylogarithmic factors. We use $\norm{\cdot}$ to denote the $\text{L}_2$-norm for both vectors and matrices. More notations are in Table~\ref{tab:notations}.

\textit{Problem formulation.} The basic FL problem is to minimize a global objective function:
\begin{align*}
	\min_{\rvx\in \R^d} \left\{ F(\rvx) \coloneqq \frac{1}{M}\sum_{m=1}^M \left(F_m(\rvx)\coloneqq \E_{\xi \sim \gD_m}\left[f_m(\rvx; \xi)\right] \right) \right\},
\end{align*}
where $F_m$ and $f_m$ denote the local objective function and the local component function of Client $m$ ($m \in [M]$), respectively. The local objective function is the average of the local component functions, $F_m(\rvx)=\frac{1}{\abs{\gD_m}}\sum_{i\in \gD_m}f_m(\rvx;\xi_m^i)$, when the local data set $\gD_m$ contains a finite number of data samples.

\textit{Update rule of SFL (Algorithm~\ref{algorithm1}).} At the beginning of each training round, the indices $\pi_1, \pi_2, \ldots, \pi_M$ are sampled without replacement from $[M]$ randomly as the clients' training order. Within a round, each client (i) initializes its model with the latest parameters from its previous client, (ii) performs $K$ steps of local updates over its local data set, and (iii) passes the updated parameters to the next client. This process continues until all clients finish their local training. Let $\rvx_{m,k}^{(r)}$ denote the local parameters of the $m$-th client (that is, Client $\pi_m$) after $k$ local steps in the $r$-th round, and $\rvx^{(r)}$ denote the global parameters in the $r$-th round. Then, if SGD is chosen as the local solver (with a constant learning rate $\eta$), the update rule of SFL is as follows:
\begin{align*}
	&\text{Local update}:\ \rvx_{m,k+1}^{(r)} = \rvx_{m,k}^{(r)} - \eta \rvg_{\pi_m,k}^{(r)},\quad \text{initializing}\ \rvx_{m,0}^{(r)} =
	\begin{cases}
		\rvx^{(r)}, &m=1\\
		\rvx_{m-1,K}^{(r)}, &m>1
	\end{cases},\\
	&\text{Global model}:\ \rvx^{(r+1)} = \rvx_{M,K}^{(r)}.
\end{align*}
Here we use $\rvg_{\pi_m,k}^{(r)} \coloneqq \nabla f_{\pi_m}(\rvx_{m,k}^{(r)};\xi_{m,k}^{(r)})$ to denote the stochastic gradient generated at the $m$-th client for its $k+1$-th local update in the $r$-th round.

\textit{Update rule of PFL (Algorithm~\ref{algorithm2}).} Within a round, each client (i) initializes its model with the global parameters, (ii) performs $K$ steps of local updates, and (iii) sends the updated parameters to the central server. The server will aggregate the local parameters to generate the global parameters. With the the same notations as those of SFL, the update rule of PFL is as follows:
\begin{align*}
	&\text{Local update}:\ \rvx_{m,k+1}^{(r)} = \rvx_{m,k}^{(r)} - \eta \rvg_{m,k}^{(r)},\quad \text{initializing}\ \rvx_{m,0}^{(r)} = \rvx^{(r)},\forall m \in [M]\\
	&\text{Global model}:\ \rvx^{(r+1)} = \frac{1}{M} \sum_{m=1}^M \rvx_{m,K}^{(r)}.
\end{align*}

\textit{The mechanism of ``two learning rates''.} \citet{karimireddy2020scaffold} has proven that the mechanism of ``two learning rates'' can improve the convergence rate of PFL. This mechanism involves using a client-specific learning rate for local updates and a different server-specific learning rate for global updates on the server. In fact, \textit{This mechanism can also be applied to SFL. Theoretically, it can achieve the same improvement as that in PFL \citep{karimireddy2020scaffold}.} We show how to implement this mechanism in SFL, and compare the upper bounds with those of PFL \citep{karimireddy2020scaffold} in Appendix~\ref{app:two-step-sizes}.

\noindent
\begin{minipage}[t]{0.5\linewidth}
	\DecMargin{0.5em}
	\begin{algorithm}[H]
		\DontPrintSemicolon
		\caption{Sequential FL}
		\label{algorithm1}
		\KwOut{$\{\rvx^{(r)}\}$}
		\For{$r = 0, \ldots, R-1$}{
			Sample a permutation $\pi_1, \pi_2, \ldots, \pi_{M}$ of $\{1,2,\ldots,M\}$\;
			\For{$m = 1,\ldots,M$ {\bf\textcolor{red}{in sequence}}}{
				$\rvx_{m,0}^{(r)} =
				\begin{cases}
					\rvx^{(r)}, &m=1\\
					\rvx_{m-1,K}^{(r)}, &m>1
				\end{cases}$\;
				\For{$k = 0,\ldots, K-1$}{
					$\rvx_{m,k+1}^{(r)} = \rvx_{m,k}^{(r)} - \eta \rvg_{\pi_m, k}^{(r)}$\;
				}
			}
			Global model: $\rvx^{(r+1)} = \rvx_{M,K}^{(r)}$\;
		}
	\end{algorithm}
	\IncMargin{0.5em}
\end{minipage}
\hfill
\begin{minipage}[t]{0.5\linewidth}
	\DecMargin{0.5em}
	\begin{algorithm}[H]
		\DontPrintSemicolon
		\caption{Parallel FL}
		\label{algorithm2}
		\KwOut{$\{\rvx^{(r)}\}$}
		\For{$r = 0,\ldots, R-1$}{
			\For{$m = 1,\ldots,M$ {\bf \textcolor{red}{in parallel}}}{
				$\rvx_{m,0}^{(r)} = \rvx^{(r)}$\;
				\For{$k = 0,\ldots, K-1$}{
					$\rvx_{m,k+1}^{(r)} = \rvx_{m,k}^{(r)} - \eta \rvg_{m,k}^{(r)}$\;
				}
			}
			Global model: $\rvx^{(r+1)} = \frac{1}{M} \sum_{m=1}^M \rvx_{m,K}^{(r)}$\;
		}
	\end{algorithm}
	\IncMargin{0.5em}
\end{minipage}

\section{Convergence Analysis of SFL}
We consider three typical cases: the strongly convex case, the general convex case and the non-convex case, where all the local objective functions $F_1,F_2,\ldots, F_M$ are $\mu$-strongly convex, general convex (see Definition~\ref{def:strong convexity}) and non-convex, respectively.

\subsection{Assumptions}\label{subsec:assumption}

We assume that (i) $F$ is lower bounded by $F^\ast$ for all cases and there exists a global minimizer $\rvx^\ast$ such that $F(\rvx^\ast)=F^\ast$ for the convex cases; (ii) all local objective functions are differentiable and smooth (see Definition~\ref{def:smoothness}). Furthermore, we need to make assumptions on the diversities: (iii) the assumptions on the stochasticity bounding the diversity of local component functions $\{f_m(\cdot;\xi_m^i)\}_i^{\lvert\gD_m\rvert}$ with respect to $i$ inside each client (Assumptions~\ref{asm:stochasticity:optimum}~and~\ref{asm:stochasticity:uniform}); (iv) the assumptions on the heterogeneity bounding the diversity of local objective functions $\{F_m\}_{m}^M$ with respect to $m$ across clients (Assumptions~\ref{asm:heterogeneity:optimum},~\ref{asm:heterogeneity:average}~and~\ref{asm:heterogeneity:max}).

\begin{definition}\label{def:strong convexity}
	A differentiable function $F$ is $\mu$-strongly convex if for all $\rvx, \rvy \in \R^d$,
	\begin{align*}
		F(\rvx) -F(\rvy) - \left\langle \nabla F(\rvy), \rvx-\rvy \right\rangle\geq \frac{\mu}{2}\norm{\rvx-\rvy}^2.
	\end{align*}
	If $\mu = 0$, we say that $F$ is general convex.
\end{definition}

\begin{definition}\label{def:smoothness}
	A differentiable function $F$ is $L$-smooth if for all $\rvx,\rvy \in \R^d$,
	\begin{align*}
		\norm{\nabla F(\rvx) - \nabla F(\rvy)} \leq L \norm{\rvx - \rvy}.
	\end{align*}
\end{definition}

\textit{Assumptions on the stochasticity.} Since both Algorithms~\ref{algorithm1} and \ref{algorithm2} use SGD (data samples are chosen with replacement) as the local solver, the stochastic gradient generated at each client is an (conditionally) unbiased estimate of the gradient of the local objective function, $\E_{\xi \sim \gD_m}[\nabla f_m(\rvx; \xi)\mid\rvx]=\nabla F_m(\rvx)$. In the FL literature, there are two common assumptions, Assumptions~\ref{asm:stochasticity:optimum} and \ref{asm:stochasticity:uniform}, to bound the stochasticity, where $\sigma_\ast$, $\sigma$ measure the level of stochasticity. Assumptions~\ref{asm:stochasticity:optimum} only assumes the bounded stochasticity at the optimum, and therefore it is weaker than Assumption~\ref{asm:stochasticity:uniform}. However, if using Assumption~\ref{asm:stochasticity:optimum}, we need to assume that each local component function $f_m(\rvx;\xi)$ is smooth, rather than merely assuming that each local objective function $F_m(\rvx)$ is smooth \citep{khaled2020tighter, koloskova2020unified}. Besides, we prioritize studying the effects of heterogeneity. For these two reasons, we use Assumption~\ref{asm:stochasticity:uniform} for all cases in this paper.
\begin{assumption}\label{asm:stochasticity:optimum}
	There exists a constant $\sigma_\ast$ such that for the global minimizer $\rvx^\ast \in \R^d$,
	\begin{align*}
		\E_{\xi\sim \gD_m}\norm{\nabla f_m(\rvx^\ast;\xi) - \nabla F_m(\rvx^\ast)}^2 \leq \sigma_\ast^2.
	\end{align*}
\end{assumption}
\begin{assumption}\label{asm:stochasticity:uniform}
	There exists a constant $\sigma$ such that for all $\rvx \in \R^d$,
	\begin{align*}
		\textstyle
		\E_{\xi\sim \gD_m}\norm{\nabla f_m(\rvx;\xi) - \nabla F_m(\rvx)}^2 \leq \sigma^2.
	\end{align*}
\end{assumption}

\textit{Assumptions on the heterogeneity.} Now we make assumptions on the diversity of the local objective functions in Assumption~\ref{asm:heterogeneity:optimum},~\ref{asm:heterogeneity:average} and~\ref{asm:heterogeneity:max}, also known as the heterogeneity in FL. For the convex cases, we use Assumption~\ref{asm:heterogeneity:optimum} as \cite{koloskova2020unified} did, which assumes the bounded diversity only at the optimum. Assumption~\ref{asm:heterogeneity:average} is made for the non-convex case, where the constants $\beta$ and $\zeta$ measure the heterogeneity of the local objective functions. Assumption~\ref{asm:heterogeneity:max}, the strongest assumption, is only made in Subsection~\ref{subsec:comparison-2}. Notably, that all the local objective functions are identical (that is, no heterogeneity) means that $\zeta_\ast,\beta, \zeta,\hat\zeta$ equal zero in these assumptions. Yet the reverse may not be true, as they only assume the first-order relationships \citep{karimireddy2020scaffold}.

\begin{assumption}\label{asm:heterogeneity:optimum}
	Let $\rvx^\ast \in \R^d$ be a minimizer of the global objective function $F$. Define
	\begin{align*}
		\textstyle
		\zeta_\ast^2 \coloneqq \frac{1}{M}\sum_{m=1}^M \norm{\nabla F_m(\rvx^\ast)}^2,
	\end{align*}
	where $\zeta_\ast^2$ is assumed to be bounded.
\end{assumption}

\begin{assumption}\label{asm:heterogeneity:average}
	There exist bounded constants $\beta^2$ and $\zeta^2$ such that for all $\rvx \in \R^d$,
	\begin{align*}
		\textstyle
		\frac{1}{M}\sum_{m=1}^M \norm{\nabla F_m(\rvx)-\nabla F(\rvx)}^2 \leq \beta^2\norm{\nabla F(\rvx)}^2 + \zeta^2.
	\end{align*}
\end{assumption}

\begin{assumption}\label{asm:heterogeneity:max}
	There exists a bounded constant $\hat\zeta^2$ such that for all $\rvx\in \R^d$,
	\begin{align*}
		\underset{m}{\max} \norm{\nabla F_m(\rvx)-\nabla F(\rvx)}^2 \leq \hat\zeta^2.
	\end{align*}
\end{assumption}

\subsection{Upper Bounds of SFL}\label{subsec:upper bound}

\begin{theorem}\label{thm:SFL}
	Let all the local objectives be $L$-smooth (Definition~\ref{def:smoothness}). For SFL (Algorithm~\ref{algorithm1}), there exist a constant effective learning rate $\tilde\eta \coloneqq \eta MK$ and weights $\{w_r\}_{r\geq 0}$, such that the weighted average of the global model parameters $\bar{\rvx}^{(R)}\coloneqq \frac{\sum_{r=0}^{R}w_r\rvx^{(r)}}{\sum_{r=0}^Rw_r}$ satisfies the following upper bounds:
	\setlist[itemize]{label=}
	\begin{itemize}[leftmargin=0.5em]
		\item \textbf{Strongly convex}: Under Assumptions~\ref{asm:stochasticity:uniform}, \ref{asm:heterogeneity:optimum}, there exist $\tilde\eta \leq \frac{1}{6L}$ and $w_r=(1-\frac{\mu\tilde\eta}{2})^{-(r+1)}$, such that for $R\geq 6\kappa$ ($\kappa \coloneqq \nicefrac{L}{\mu}$),
		\begin{flalign*}
			\E\left[F(\bar\rvx^{(R)})-F(\rvx^\ast)\right] \leq \frac{9}{2}\mu D^2 \exp\left(-\frac{\mu\tilde\eta R}{2} \right)+\frac{12\tilde{\eta}\sigma^2}{MK}+\frac{18L\tilde{\eta}^2\sigma^2}{MK}+\frac{18L\tilde{\eta}^2\zeta_\ast^2}{M}.\label{eq:thm:strongly convex} &&
		\end{flalign*}
		\item \textbf{General convex}: Under Assumptions~\ref{asm:stochasticity:uniform}, \ref{asm:heterogeneity:optimum}, there exist $\tilde\eta \leq \frac{1}{6L}$ and $w_r=1$, such that
		\begin{flalign*}
			\E\left[F(\bar\rvx^{(R)})-F(\rvx^\ast)\right] \leq \frac{3D^2}{\tilde\eta R}+\frac{12\tilde{\eta}\sigma^2}{MK}+\frac{18L\tilde{\eta}^2\sigma^2}{MK}+\frac{18L\tilde{\eta}^2\zeta_\ast^2}{M}. &&
		\end{flalign*}
		\item \textbf{Non-convex}: Under Assumptions~\ref{asm:stochasticity:uniform}, \ref{asm:heterogeneity:average}, there exist $\tilde\eta\leq \frac{1}{6L(1+\beta^2/M)}$ and $w_r=1$, such that
		\begin{flalign*}
			\min_{0\leq r\leq R} \E\left[\Norm{\nabla F(\rvx^{(r)})}^2\right] \leq \frac{10A}{\tilde\eta R} + \frac{20L\tilde\eta\sigma^2}{MK} + \frac{75L^2\tilde\eta^2\sigma^2}{4MK} + \frac{75L^2\tilde\eta^2\zeta^2}{4M}. &&
		\end{flalign*}
	\end{itemize}
	Here $D\coloneqq\norm{x^{(0)}-x^\ast}$ for the convex cases and $A \coloneqq F(\rvx^{(0)}) - F^\ast$ for the non-convex case.
\end{theorem}
\begin{proof}
	We provide intuitive proof sketches of Theorem~\ref{thm:SFL} as done in \cite{karimireddy2020scaffold}. Ideally, we want to update the model with the gradients of the global objective function. For any local gradient in some training round of SFL, it can be decomposed into two vectors (or estimated by Taylor formula),
	\begin{align*}
		\nabla F_m(\rvx_{m,k}) \approx \left(\nabla F_m (\rvx) + \nabla^2 F_m(\rvx) (\rvx_{m,k} - \rvx) \right).
	\end{align*}
	Then, the global update of SFL can be written as
	\begin{flalign*}
		\Delta_{\text{SFL}}&= -\eta  \sum_{m=1}^M \sum_{k=0}^{K-1} \left\{\nabla F_m (\rvx_{m,k})\approx \left(\nabla F_m (\rvx) + \nabla^2 F_m(\rvx) (\rvx_{m,k} - \rvx) \right) \right\}&&\\
		&= \underbrace{-\eta MK\nabla F(\rvx)}_{\text{optimization vector}} \underbrace{- \eta \sum_{m=1}^M \sum_{k=0}^{K-1}\nabla^2 F_m(\rvx) (\rvx_{m,k} - \rvx)}_{\text{error vector}}.&&
	\end{flalign*}
	The optimization vector is beneficial while the error vector is detrimental. Thus, our goal is to suppress the error vector. Theorem~\ref{thm:SFL} is aimed to prove that $\sum_{m=1}^M \sum_{k=0}^{K-1}\norm{\rvx_{m,k} - \rvx}^2$ is bounded ($\nabla^2 F_m(\rvx) \approx L$). Intuitively, for there are about $mK$ update steps between $\rvx$ and $\rvx_{m,k}$, it is estimated to be $\gO\left(\sum_{m=1}^M \sum_{k=0}^{K-1} (\eta \sqrt{m} K \zeta)^2 \right) = \gO\left(\eta^2 M^2 K^3 \zeta^2 \right)$, where $\sqrt{m}$ is due to the shuffling-based manner. The formal proofs are in \cite{li2023convergence}.
\end{proof}

The effective learning rate $\tilde\eta \coloneqq \eta MK$ is used in the upper bounds as done in \cite{karimireddy2020scaffold, wang2020tackling}. All these upper bounds consist of two parts: the optimization part (the first term) and the error part (the last three terms). Setting $\tilde\eta$ larger makes the optimization part vanishes at a higher rate, yet causes the error part to be larger. This implies that we need to choose an appropriate $\tilde\eta$ to achieve a balance between these two parts, which is actually done in Corollary~\ref{cor:SFL}. Here we choose the appropriate learning rate with a prior knowledge of the total training rounds $R$ \citep{karimireddy2020scaffold}.

\begin{corollary}\label{cor:SFL}
	By choosing an appropriate learning rate for the results of Theorem~\ref{thm:SFL}, we can obtain the upper bounds of SFL:
	\setlist[itemize]{label=}
	\begin{itemize}[leftmargin=0.5em]
		\item \textbf{Strongly convex}: If $\tilde\eta=\eta MK \asymp  \min \{\frac{1}{L},\frac{1}{\mu R} \}$ for Theorem~\ref{thm:SFL}, then
		\begin{flalign*}
			\E\left[F(\bar\rvx^{(R)})-F(\rvx^\ast)\right] = \tilde\gO\left(\frac{\sigma^2}{\mu MKR} + \frac{L\sigma^2}{\mu^2MKR^2} + \frac{L\zeta_\ast^2}{\mu^2MR^2} + \mu D^2 \exp\left(\frac{-\mu R}{L}\right)\right).&&
		\end{flalign*}
		\item \textbf{General convex}: If $\tilde\eta=\eta MK \asymp  \min \{\frac{1}{L},\frac{D}{c_1^{1/2}R^{1/2}},\frac{D^{2/3}}{c_2^{1/3}R^{2/3}} \}$ with $c_1 \asymp \frac{\sigma^2}{MK}$ and $c_2 \asymp \frac{L\sigma^2}{MK} + \frac{L\zeta^2}{M}$ for Theorem~\ref{thm:SFL}, then
		\begin{flalign*}
			\E\left[F(\bar\rvx^{(R)})-F(\rvx^\ast)\right] = \gO\left(\frac{\sigma D}{\sqrt{MKR}} + \frac{\left(L\sigma^2D^4\right)^{1/3}}{(MK)^{1/3}R^{2/3}} + \frac{\left(L\zeta_\ast^2D^4\right)^{1/3}}{M^{1/3}R^{2/3}} + \frac{LD^2}{R}\right). &&
		\end{flalign*}
		\item \textbf{Non-convex}: If $\tilde\eta=\eta MK \asymp  \min \{\frac{1}{L(1+\beta^2/M)},\frac{A^{1/2}}{c_1^{1/2}R^{1/2}},\frac{A^{1/3}}{c_2^{1/3}R^{2/3}} \}$ with $c_1 \asymp \frac{L\sigma^2}{MK}$ and $c_2 \asymp \frac{L^2\sigma^2}{MK} + \frac{L^2\zeta^2}{M}$ for Theorem~\ref{thm:SFL}, then
		\begin{flalign*}
			\min_{0\leq r\leq R} \E\left[\Norm{\nabla F(\rvx^{(r)})}^2\right] = \gO\left(\frac{ \left(L\sigma^2A\right)^{1/2}}{\sqrt{MKR}} + \frac{\left(L^2\sigma^2A^2\right)^{1/3}}{(MK)^{1/3}R^{2/3}} + \frac{\left(L^2\zeta^2A^2\right)^{1/3}}{M^{1/3}R^{2/3}} + \frac{LA(1+\frac{\beta^2}{M})}{R}\right). &&
		\end{flalign*}
	\end{itemize}
	Here $D\coloneqq\norm{x^{(0)}-x^\ast}$ for the convex cases and $A \coloneqq F(\rvx^{(0)}) - F^\ast$ for the non-convex case.
\end{corollary}

Similar to Theorem~\ref{thm:SFL}, all these upper bounds consist of two parts, the optimization part (the last term), and the error part (the first three terms). Specifically, the first two terms (containing $\sigma$) is called stochasticity terms, the third term (containing $\zeta_\ast$, $\zeta$) is called heterogeneity terms, the last term is called optimization terms.

Generally, for a sufficiently large number of training rounds $R$, the convergence rate is determined by the first term for all cases, resulting in rates of $\tilde\gO(1/MKR)$, $\gO(1/\sqrt{MKR})$, $\gO(1/\sqrt{MKR})$ for the strongly convex, general convex and non-convex cases, respectively.

Recall that SGD-RR can be seen as one special case of SFL, where one step of GD is performed on each local objective $F_m$, which implies $K=1$ and $\sigma=0$. We now compare the upper bounds of SFL with those of SGD-RR to exam the tightness. As shown in \cite{mishchenko2020random}'s Corollaries~1, 2, 3, the upper bounds of SGD-RR are $\tilde\gO\left({\color{blue}\frac{L}{\mu}}\left(\frac{L\zeta_\ast^2}{\mu^2MR^2} + \mu D^2\exp\left(\frac{-\mu {\color{red}\boldsymbol{M}}R}{L}\right)\right)\right)$, $\gO\left(\frac{\left(L\zeta_\ast^2D^4\right)^{1/3}}{M^{1/3}R^{2/3}} + \frac{LD^2}{R}\right)$, $\gO\left(\frac{\left(L^2\zeta^2A^2\right)^{1/3}}{M^{1/3}R^{2/3}} + \frac{L A}{R}\right)$ for the strongly convex, general convex and non-convex cases, respectively. We see that our bounds match those of SGD-RR in the general convex, and non-convex cases. For the strongly convex case, the bound of SGD-RR shows an advantage on the optimization term (marked in red). This advantage is due to the advanced technique of Shuffling Variance introduced by \citet[][Definition~2]{mishchenko2020random}. We leave the investigation on introducing this advanced technique to SFL for future work.

\subsection{Lower Bounds of SFL}\label{subsec:lower bound}

The lower bounds of SFL are stated in Theorem~\ref{thm:lower bound} and Theorem~\ref{thm:lower bound2}. Theorem~\ref{thm:lower bound} is for arbitrary learning rates $\eta>0$ and Theorem~\ref{thm:lower bound2} is for small learning rates $0 < \eta \lesssim \frac{1}{LMK}$.

\begin{theorem}\label{thm:lower bound}
	There exist a multi-dimensional global objective function, whose local objective functions are $\mu$-strongly convex (Definition~\ref{def:strong convexity}) and $L$-smooth (Definition~\ref{def:smoothness}), and satisfy Assumptions~\ref{asm:stochasticity:uniform}~and~\ref{asm:heterogeneity:average}, and an initialization point $\rvx^{(0)}$ such that for any $\eta>0$ and {$R\geq 1$}, $M\geq 4$, $K\geq 1$, the last-round global parameters $\rvx^{(R)}$ satisfy
	\begin{align*}
		\E\left[ F(\rvx^{(R)}) - F(\rvx^\ast) \right] = \Omega \left( \frac{\sigma^2}{\mu MKR} + \frac{\zeta^2}{\mu MR^2}\right).
	\end{align*}
\end{theorem}

\begin{proof}
	The proofs are in Appendices~\ref{app:lower bound},~\ref{app:lower bound-stochasticity}~and~\ref{app:lower bound-heterogeneity}. We construct the lower bounds for stochasticity terms and heterogeneity terms, separately. For stochasticity, we let all the local objective functions be the same, that is $F=F_1=\cdots=F_M$, and then aim to derive the lower bound for SGD (see Appendix~\ref{app:lower bound-stochasticity}). For heterogeneity, we let the local component functions be the same inside each client $F_m = f_m(\cdot;\xi_m^1)= \cdots= f_m(\cdot;\xi_m^{|\gD_m|})$ for each $m$, and then aim to extend the works of SGD-RR to SFL, that is, from performing one update step to performing multiple update steps on each local objective function (see Appendix~\ref{app:lower bound-heterogeneity}). We use the techniques in \cite{woodworth2020minibatch}'s Theorem~2, \cite{yun2022minibatch}'s Theorem~4 and Proposition 5 and \cite{cha2023tighter}'s Theorem 3.1.
\end{proof}

This lower bound, which holds for arbitrary $\eta>0$, matches the error terms in the strongly convex case in Corollary~\ref{cor:SFL}, up to a factor of $\kappa$ and some polylogarithmic factors. In the next theorem, for small learning rates $\eta \lesssim \frac{1}{LMK}$, we can remove the gap of $\kappa$.

\begin{theorem}\label{thm:lower bound2}
	Under the same conditions of Theorem~\ref{thm:lower bound} (unless explicitly stated), there exist a multi-dimensional global objective function and an initialization point, such that for $\eta\leq \frac{1}{101LMK}$, the arbitrary weighted average global parameters $\bar\rvx^{(R)}$ satisfy the lower bounds:
	\setlist[itemize]{label=}
	\begin{itemize}[leftmargin=0.5em]
		\item \textbf{Strongly convex}: If $R\geq \frac{1}{1010}\kappa$ and $\kappa \geq 1010$, then
		\begin{flalign*}
			\E\left[F(\bar\rvx^{(R)})-F(\rvx^\ast)\right] = \Omega \left( \frac{\sigma^2}{\mu MKR} + \frac{L\sigma^2}{\mu^2 MKR^2} + \frac{L\zeta^2}{\mu^2 MR^2}\right).&&
		\end{flalign*}
		\item \textbf{General convex}: If $R\geq 51^3\max\left\{ \frac{\sigma}{LM^{1/2}K^{1/2}D},  \frac{L^2 MKD^2}{\sigma^2}, \frac{\zeta}{LM^{1/2}D},  \frac{L^2 MD^2}{\zeta^2}\right\}$, then
		\begin{flalign*}
			\E\left[F(\bar\rvx^{(R)})-F(\rvx^\ast)\right] = \Omega\left(\frac{\sigma D}{\sqrt{MKR}} + \frac{\left(L\sigma^2D^4\right)^{1/3}}{(MK)^{1/3}R^{2/3}} + \frac{\left(L\zeta^2D^4\right)^{1/3}}{M^{1/3}R^{2/3}}\right). &&
		\end{flalign*}
	\end{itemize}
\end{theorem}
\begin{proof}
	The proofs are in Appendix~\ref{app:subsec:cor-lower bound}. We use similar techniques in \cite{woodworth2020minibatch,cha2023tighter}.
\end{proof}

Theorem~\ref{thm:lower bound2} provides the lower bounds in the convex cases with any arbitrary weighted average parameter $\bar\rvx^{(R)}$ for small learning rates $\eta\lesssim \frac{1}{LMK}$. Although the constraint on small learning rates seems stringent, Theorem~\ref{thm:SFL} indicates that such small learning rates $\tilde \eta = \eta MK \lesssim \frac{1}{L}$ also exist in our upper bounds. Therefore, it is justified to use the lower bounds to assess the tightness of our upper bounds.

In the strongly convex case in Corollary~\ref{cor:SFL}, if $R \gtrsim \kappa$, the learning rate choice becomes
\[
\tilde\eta \asymp \min\left\{ \frac{1}{L}, \frac{1}{\mu R} \right\} \asymp \frac{1}{\mu R} \not\asymp \frac{1}{L} ,
\]
which yields the upper bound of $\tilde\gO\left(\frac{\sigma^2}{\mu MKR} + \frac{L\sigma^2}{\mu^2MKR^2} + \frac{L\zeta^2}{\mu^2MR^2}\right)$ for SFL. This upper bound exactly matches the lower bound in Theorem~\ref{thm:lower bound2} (ignoring polylogarithmic factors). Notably, the upper bound differs from the original bound in Corollary~\ref{cor:SFL} due to the optimization term (the last term) in Corollary~\ref{cor:SFL} being present only when $\tilde\eta \asymp \frac{1}{L}$ in the convex cases (see the proofs of \citealt{li2023convergence}'s Lemmas~7 and 8). Moreover, it is appropriate to compare the upper bounds (with $\zeta_\ast$) and the lower bounds (with $\zeta$), since Assumption~\ref{asm:heterogeneity:average} is stronger than Assumption~\ref{asm:heterogeneity:optimum} \citep{cha2023tighter}.

In the general convex case, with the same logic as in the strongly convex case, if $R\gtrsim \max\left\{ \frac{\sigma}{LM^{1/2}K^{1/2}D},  \frac{L^2 MKD^2}{\sigma^2}, \frac{\zeta}{LM^{1/2}D},  \frac{L^2 MD^2}{\zeta^2}\right\}$, then the upper bound of SFL becomes $\gO\left(\frac{\sigma D}{\sqrt{MKR}} + \frac{\left(L\sigma^2D^4\right)^{1/3}}{(MK)^{1/3}R^{2/3}} + \frac{\left(L\zeta_\ast^2D^4\right)^{1/3}}{M^{1/3}R^{2/3}}\right)$. It matches the lower bound in Theorem~\ref{thm:lower bound2}.

\textit{Key points and limitations of Subsection~\ref{subsec:lower bound}.} The matching lower bounds in Theorem~\ref{thm:lower bound} and Theorem~\ref{thm:lower bound2} verify that our upper bounds are tight in the convex cases for the sufficiently large number of training rounds $R$. However, the lower bounds for small $R$ are loose and the lower bounds for the non-convex case are still lacking, for both SGD-RR and SFL.

\section{Comparison Between PFL and SFL}\label{sec:comparison}
Unless otherwise stated, our comparisons are in terms of training rounds, which is also adopted in \cite{gao2021evaluation}. This comparison (running for the same number of total training rounds $R$) is fair when considering the same total computation cost for both methods. We summarize the existing convergence results of PFL in Table~\ref{tab:comparison}.

\begin{table}[phtb]
	\vspace{-2ex}
	\centering
	\renewcommand{\arraystretch}{1}
	\resizebox{\linewidth}{!}{
		\begin{threeparttable}[b]
			\begin{tabular}{p{13.5em}l}
				\toprule
				Method &Upper Bound\\\midrule
				\multicolumn{2}{c}{Strongly Convex}\\
				\midrule
				PFL \citep{karimireddy2020scaffold} &$\frac{\sigma^2}{\mu MKR} + \frac{L\sigma^2}{\mu^2KR^2} + \frac{L\zeta^2}{\mu^2 R^2} + \mu D^2 \exp\left(\frac{-\mu R}{L}\right)
				$ \tnote{\color{black}(1)}\\[1.5ex]
				PFL \citep{koloskova2020unified} &$\frac{\sigma_\ast^2}{\mu MKR} + \frac{L\sigma_\ast^2}{\mu^2KR^2} + \frac{L\zeta_\ast^2}{\mu^2 R^2} + LKD^2 \exp\left(\frac{-\mu R}{L}\right)
				$ \tnote{\color{black}(2)}\\[1.5ex]
				PFL (Theorem~\ref{thm:PFL-1}) &$\frac{\sigma^2}{\mu MKR} + \frac{L\sigma^2}{\mu^2KR^2} + \frac{L\zeta_\ast^2}{\mu^2 R^2} + \mu D^2 \exp\left(\frac{-\mu R}{L}\right)
				$\\[1.5ex]
				\rowcolor{noteblue!20}
				PFL \citep{woodworth2020minibatch} &$\frac{\sigma^2}{\mu MKR} + \frac{L\sigma^2}{\mu^2KR^2} + \frac{L\hat\zeta^2}{\mu^2R^2} + \mu D^2 \exp\left(\frac{-\mu {\color{red}\boldsymbol{K}}R}{L}\right)$ \tnote{\color{black}(3)}\\[1ex]
				\rowcolor{notegreen!20}
				SFL (Theorem~\ref{thm:SFL}) &$\frac{\sigma^2}{\mu MKR} + \frac{L\sigma^2}{\mu^2{\color{red}\boldsymbol{M}}KR^2} + \frac{L\zeta_\ast^2}{\mu^2{\color{red}\boldsymbol{M}}R^2} + \mu D^2 \exp\left(\frac{-\mu R}{L}\right)$\\\midrule
				
				\multicolumn{2}{c}{Convex}\\
				\midrule
				PFL \citep{karimireddy2020scaffold} &$\frac{\sigma D}{\sqrt{MKR}} + \frac{\left(L\sigma^2D^4\right)^{1/3}}{K^{1/3}R^{2/3}} + \frac{\left(L\zeta^2D^4\right)^{1/3}}{R^{2/3}} + \frac{LD^2}{R}
				$\\[1.5ex]
				PFL \citep{koloskova2020unified} &$\frac{\sigma_\ast D}{\sqrt{MKR}} + \frac{\left(L\sigma_\ast^2D^4\right)^{1/3}}{K^{1/3}R^{2/3}} + \frac{\left(L\zeta_\ast^2D^4\right)^{1/3}}{R^{2/3}} + \frac{LD^2}{R}
				$\\[1.5ex]
				\rowcolor{noteblue!20}
				PFL \citep{woodworth2020minibatch} &$\frac{\sigma D}{\sqrt{MKR}} + \frac{\left(L\sigma^2D^4\right)^{1/3}}{K^{1/3}R^{2/3}} + \frac{\left(L\hat\zeta^2D^4\right)^{1/3}}{R^{2/3}} + \frac{LD^2}{{\color{red}\boldsymbol{K}}R}
				$\\[1.5ex]
				\rowcolor{notegreen!20}
				SFL (Theorem~\ref{thm:SFL}) &$\frac{\sigma D}{\sqrt{MKR}} + \frac{\left(L\sigma^2D^4\right)^{1/3}}{{\color{red}\boldsymbol{M^{1/3}}}K^{1/3}R^{2/3}} + \frac{\left(L\zeta_\ast^2D^4\right)^{1/3}}{{\color{red}\boldsymbol{M^{1/3}}}R^{2/3}} + \frac{LD^2}{R}$\\\midrule

				\multicolumn{2}{c}{Non-convex}\\
				\midrule
				PFL \citep{karimireddy2020scaffold, koloskova2020unified} &$\frac{ \left(L\sigma^2A\right)^{1/2}}{\sqrt{MKR}} + \frac{\left(L^2\sigma^2A^2\right)^{1/3}}{K^{1/3}R^{2/3}} + \frac{\left(L^2\zeta^2A^2\right)^{1/3}}{R^{2/3}} + \frac{LA}{R}
				$ \tnote{\color{black}(4)}\\
				\rowcolor{notegreen!20}
				SFL (Theorem~\ref{thm:SFL}) &$\frac{ \left(L\sigma^2A\right)^{1/2}}{\sqrt{MKR}} + \frac{\left(L^2\sigma^2A^2\right)^{1/3}}{{\color{red}\boldsymbol{M^{1/3}}}K^{1/3}R^{2/3}}+\frac{\left(L^2\zeta^2A^2\right)^{1/3}}{{\color{red}\boldsymbol{M^{1/3}}}R^{2/3}} + \frac{LA}{R}$\tnote{\color{black}(5)}\\
				\bottomrule
			\end{tabular}
			\begin{tablenotes}
				\begin{footnotesize}
					\item [\tnote{\color{black}(1)}] (i) We use $\frac{3L\eta^3 K^3 \sigma^2}{K}$ (see the last inequality of the proof of their Lemma~8) while \cite{karimireddy2020scaffold} use $\frac{\eta^2 K^2 \sigma^2}{2K}$ with $\eta \leq 8LK$ (their Lemma~8), which causes the difference between their original bounds and our recovered bounds. (ii) This difference also exists in the other two cases. (iii) Their Assumption~A1 is essentially equivalent to Assumption~\ref{asm:heterogeneity:average}. For simplicity, we let $B=1$ in their Assumption~A1 for all three cases.
					\item [\tnote{\color{black}(2)}] Even the weaker Assumption~\ref{asm:stochasticity:optimum} is used in \cite{koloskova2020unified}, we do not consider it is a improvement over ours in this paper, given the discussions in Subsection~\ref{subsec:assumption}.
					\item [\tnote{\color{black}(3)}] Applying \cite{karimireddy2020scaffold}'s Lemma~1 instead of their Theorem~3 yields this bound. Notably, \cite{woodworth2020minibatch} assume the average of the local parameters for all iterations can be obtained, which is in fact impractical in FL. Similar assumptions are made in \cite{khaled2020tighter, koloskova2020unified}. In this paper, we omit this difference.
					\item [\tnote{\color{black}(4)}] We let $P=1, M=0$ in \cite{koloskova2020unified}'s Assumption~3b.
					\item [\tnote{\color{black}(5)}] We let $\beta=0$ in Assumption~\ref{asm:heterogeneity:average}.
				\end{footnotesize}
			\end{tablenotes}
		\end{threeparttable}}
	\caption{Upper bounds of PFL and SFL with absolute constants and polylogarithmic factors omitted. We highlight the upper bounds of ``PFL under Assumption~\ref{asm:heterogeneity:max}''/``SFL'' with a blue/green background. Main differences are marked in red fonts.}
	\label{tab:comparison}
\end{table}

\subsection{Comparison under Assumption~\ref{asm:heterogeneity:optimum}}\label{subsec:comparison-1}
\begin{theorem}\label{thm:PFL-1}
	Under the same conditions as those of the strongly convex case in Theorem~\ref{thm:SFL}, there exist $\tilde\eta= \eta K\asymp \min \left\{\frac{1}{L},\frac{1}{\mu R}\right\}$ and $w_r=(1-\frac{\mu\tilde\eta}{2})^{-(r+1)}$, such that for $R\gtrsim \kappa$,
	\begin{align*}
		\E\left[F(\bar\rvx^{(R)})-F(\rvx^\ast)\right] = \tilde\gO\left(\frac{\sigma^2}{\mu MKR} + \frac{L\sigma^2}{\mu^2KR^2} + \frac{L\zeta_\ast^2}{\mu^2R^2} + \mu D^2 \exp\left(\frac{-\mu R}{L}\right)\right).
	\end{align*}
\end{theorem}
\begin{proof}
	Applying \cite{karimireddy2020scaffold}'s Lemma~1 instead of \cite{koloskova2020unified}'s Lemma~15 to the final recursion in \cite{koloskova2020unified} yields this theorem. The detailed proofs (specialized for PFL) are in \cite{li2023convergence}.
\end{proof}
To the best of our knowledge, the existing tightest upper bounds that uses Assumption~\ref{asm:heterogeneity:optimum} to catch the heterogeneity for PFL are introduced in \cite{koloskova2020unified}. Many works \citep{woodworth2020minibatch, yun2022minibatch, glasgow2022sharp} have constructed lower bounds to show these bounds are almost the tightest for the convex cases. \cite{glasgow2022sharp} has shown that this upper bound for the general convex case is not improvable.

For the following comparisons in this subsection, we mainly focus on the strongly convex case. For fairness, we slightly improve the bound of PFL in the strongly convex case in Theorem~\ref{thm:PFL-1} by combining the works of \cite{karimireddy2020scaffold,koloskova2020unified}. Unless otherwise stated, the conclusions also hold for the other two cases.
\begin{itemize}[leftmargin=1.5em]
	\item \textit{The upper bounds of SFL are better than PFL on heterogeneous data.} As shown in Table~\ref{tab:comparison} (Theorems~\ref{thm:SFL} and~\ref{thm:PFL-1}), the upper bound of SFL is better than that of PFL, with an advantage of $1/M$ on the second and third terms (marked in red). This benefits from its sequential and shuffling-based training manner of SFL.
	\item \textit{Partial client participation.} In the more challenging cross-device settings, only a small fraction of clients participate in each training round. Following the work in \cite{karimireddy2020scaffold, yang2021achieving}, we provide the upper bounds for PFL and SFL with partial client participation as follows:
	\begin{align*}
		\text{PFL:} \quad &\tilde\gO\left(\frac{\sigma^2}{\mu SKR} + {\color{black}\frac{\zeta_\ast^2}{\mu R}\frac{M-S}{S(M-1)}} + \frac{L\sigma^2}{\mu^2KR^2} + \frac{L\zeta_\ast^2}{\mu^2 R^2} + \mu D^2 \exp\left(\frac{-\mu R}{L}\right)\right), \\
		\text{SFL:} \quad&\tilde\gO\left(\frac{\sigma^2}{\mu SKR}+{\color{black}\frac{\zeta_\ast^2}{\mu R}\frac{(M-S)}{S(M-1)}} + \frac{L\sigma^2}{\mu^2{\color{red}\boldsymbol{S}}KR^2} + \frac{L\zeta_\ast^2}{\mu^2{\color{red}\boldsymbol{S}}R^2}+\mu D^2 \exp\left(\frac{-\mu R}{L}\right)\right),
	\end{align*}
	where a subset of clients $\gS$ (its size is $\abs{\gS}=S$) are selected randomly without replacement in each training round. There are additional terms (the second terms) for both PFL and SFL, which is due to partial client participation and random sampling \citep{yang2021achieving}. \textit{It can be seen that the advantage of $1/S$ (marked in red) of SFL still exists, similar to the full client participation setup.} The proofs are in \cite{li2023convergence}.

	Notably, the proofs of SFL with partial client participation are nontrivial considering $\E_{\pi}\left[ \frac{1}{SK}\sum_{m\in \gS, k} \nabla F_{\pi_m}(\rvx_{m,k}) \right] \neq \frac{1}{MK}\sum_{m, k}\E\left[\nabla F_{m}(\rvx_{m,k}) \right]$ (updates in different clients are not independent) and we cannot transform them into the full participation setup directly as done in PFL \citep{karimireddy2020scaffold, yang2021achieving}.
\end{itemize}

\textit{Key points of Subsection~\ref{subsec:comparison-1}.} The discussions above show that the upper bounds of SFL are better than PFL with both full client participation and partial client participation under Assumption~\ref{asm:heterogeneity:optimum} in the convex cases and under Assumption~\ref{asm:heterogeneity:average} in the non-convex case.

\subsection{Comparison under Assumption~\ref{asm:heterogeneity:max}}\label{subsec:comparison-2}

Since it is hard to achieve an improvement for SFL even with the stronger Assumption~\ref{asm:heterogeneity:max}, we next compare Corollary~\ref{cor:SFL} with \cite{woodworth2020minibatch}'s Theorem~3 (under Assumption~\ref{asm:heterogeneity:max}) to show that PFL can outperform SFL when the heterogeneity is very small. For comparison on bounds with different heterogeneity assumptions, we note that if Assumption~\ref{asm:heterogeneity:max} holds, then Assumption~\ref{asm:heterogeneity:optimum} holds, and $\zeta_\ast \leq \hat\zeta$.

As shown in Table~\ref{tab:comparison}, the results of PFL under Assumption~\ref{asm:heterogeneity:max} are highlighted with a blue background and the results of SFL (under Assumption~\ref{asm:heterogeneity:optimum}) are highlighted with a green background. These bounds closely resembles each other, with three error terms (the first three terms containing $\sigma$, $\zeta$) and one optimization term (the last one). To emphasize the role of heterogeneity, we let $\sigma = 0, \mu=L=D=1$ as done in \cite{woodworth2020minibatch}.

In the strongly convex case, it can be seen that the upper bound of PFL shows better on its optimization term, while worse in the error terms. Consequently, to make the upper bound of PFL smaller, one sufficient (not necessary) condition is $\frac{\hat\zeta^2}{R^2} \lesssim \exp \left(-KR\right)$, or equivalently $\hat\zeta^2 \lesssim R^2 \cdot \exp \left(-KR\right)$, which implies that $\hat\zeta$ should be very small, or the level of heterogeneity is very low. In this condition, the optimization terms become dominant for both PFL and SFL,
\[
\frac{\zeta_\ast^2}{MR^2} \lesssim \frac{\hat\zeta^2}{R^2} \lesssim \exp \left(-KR\right) \lesssim \exp \left(-R\right),
\]
and then the bound of PFL will be better than that of SFL. However, similarly, once $\zeta_\ast^2 \gtrsim MR^2\exp(-R)$, the error terms will become dominant and SFL becomes better.

In the general convex case, with the same logic as the strongly convex case, the sufficient (not necessary) condition is $\hat\zeta^2 \lesssim 1/(K^3R)$ , which still implies that $\hat\zeta$ should be very small.

\textit{Key points of Subsection~\ref{subsec:comparison-2}.} The discussions above show that the upper bounds of PFL can be better than SFL only when the heterogeneity is very small under Assumption~\ref{asm:heterogeneity:max} in the convex cases. However, it is unclear whether this superiority still exists under Assumption~\ref{asm:heterogeneity:optimum} in the convex cases, and in the non-convex case.

\section{Experiments}\label{sec:exp}
We conduct experiments on quadratic functions (Subsection~\ref{subsec:exp:QF}), logistic regression (Subsection~\ref{subsec:exp:LR}) and deep neural networks (Subsection~\ref{subsec:exp:DNN}) to validate our theoretical finding that SFL outperforms PFL in heterogeneous settings, at least when the level of heterogeneity is relatively high. The code is available at \url{https://github.com/liyipeng00/SFL}.

\subsection{Experiments on Quadratic Functions}\label{subsec:exp:QF}

According to the analyses in Subsections~\ref{subsec:comparison-1} and \ref{subsec:comparison-2}, SFL outperforms PFL in heterogeneous settings (at least when the level of heterogeneity is relatively high). Here we show that the counterintuitive result (in contrast to \citealt{gao2021evaluation}) can appear even for simple one-dimensional quadratic functions \citep{karimireddy2020scaffold}.

To further catch the heterogeneity, in addition to Assumption~\ref{asm:heterogeneity:optimum}, we also consider Hessian of objective functions \citep{karimireddy2020scaffold,glasgow2022sharp,patel2024limits}:
\begin{align*}
\max_{m} \norm{\nabla^2 F_m (\rvx) - \nabla^2 F(\rvx)} \leq \delta.
\end{align*}
Larger value of $\delta$ means higher heterogeneity on Hessian.
\begin{align*}
\norm{\nabla^2 F(\rvx) - \nabla^2 F(\rvy)} \leq H \norm{\rvx - \rvy}.
\end{align*}
Larger value of $H$ means more drastic Hessian change.

As shown in Table~\ref{tab:simulation settings}, we use ten groups of objective functions with various degrees of heterogeneity. In fact, we construct the lower bounds in Theorem~\ref{thm:lower bound} with similar functions. As suggested by our theory, we set the learning rate of SFL be half of that of PFL. The experimental results of Table~\ref{tab:simulation settings} are shown in Figure~\ref{fig:quadratic functions}.

Overall, SFL outperforms PFL in all settings except the settings $\delta=0$ and $H=0$ (Groups 1, 6), which coincides with our theoretical conclusion. We attribute the unexpected cases to the limitations of existing works under Assumptions~\ref{asm:heterogeneity:optimum},~\ref{asm:heterogeneity:average} and \ref{asm:heterogeneity:max}, which omit the function of the global aggregation and thus underestimate the capacity of PFL \citep{wang2024unreasonable}. More specifically, the second-order information (Hessian) is not fully studied in existing works \citep{patel2023still,patel2024limits}.

\begin{table}[ptbh]
	\renewcommand{\arraystretch}{1}
	\centering
	\setlength{\tabcolsep}{0.2em}
	\resizebox{\linewidth}{!}{
		\begin{tabular}{c|c|c|c|c}
			\toprule
			Group 1 &Group 2 &Group 3 &Group 4 &Group 5\\
			\midrule
			$\begin{cases}F_1=\frac{1}{2}x^2+x\\F_2 = \frac{1}{2}x^2-x\end{cases}$ 
			&$\begin{cases}F_1=\frac{3}{4}x^2 + x\\ F_2=\frac{1}{4}x^2 - x\end{cases}$ 
			&$\begin{cases}F_1=x^2 + x\\ F_2= - x\end{cases}$
			&$\begin{cases}F_1=(\frac{3}{4}\mathbbm{1}_{x<0} + \frac{1}{2} \mathbbm{1}_{x\geq 0})x^2 + x\\ F_2 = (\frac{3}{4}\mathbbm{1}_{x<0} + \frac{1}{2} \mathbbm{1}_{x\geq 0})x^2 - x\end{cases}$
			&$\begin{cases}F_1=(1\mathbbm{1}_{x<0} + \frac{1}{2} \mathbbm{1}_{x\geq 0})x^2 + x\\ F_2 = (1\mathbbm{1}_{x<0} + \frac{1}{2} \mathbbm{1}_{x\geq 0})x^2 - x\end{cases}$
			\\[3ex]
			$\zeta_\ast=1, \delta=0, H = 0$  &$\zeta_\ast=1, \delta=\frac{1}{2}, H=0$ &$\zeta_\ast=1, \delta=1, H=0$ &$\zeta_\ast=1, \delta=0, H=\frac{1}{2}$ &$\zeta_\ast=1, \delta=0, H=1$\\
			\midrule\midrule
			Group 6 &Group 7 &Group 8 &Group 9 &Group 10\\\midrule
			$\begin{cases}F_1=\frac{1}{2}x^2+10x\\F_2 = \frac{1}{2}x^2-10x\end{cases}$ 
			&$\begin{cases}F_1=\frac{3}{4}x^2 + 10x\\ F_2=\frac{1}{4}x^2 - 10x\end{cases}$ 
			&$\begin{cases}F_1=x^2 + 10x\\ F_2= - 10x\end{cases}$
			&$\begin{cases}F_1=(\frac{3}{4}\mathbbm{1}_{x<0} + \frac{1}{2} \mathbbm{1}_{x\geq 0})x^2 + 10x\\ F_2 = (\frac{3}{4}\mathbbm{1}_{x<0} + \frac{1}{2} \mathbbm{1}_{x\geq 0})x^2 - 10x\end{cases}$
			&$\begin{cases}F_1=(1\mathbbm{1}_{x<0} + \frac{1}{2} \mathbbm{1}_{x\geq 0})x^2 + 10x\\ F_2 = (1\mathbbm{1}_{x<0} + \frac{1}{2} \mathbbm{1}_{x\geq 0})x^2 - 10x\end{cases}$
			\\[3ex]
			$\zeta_\ast=10, \delta=0, H = 0$  &$\zeta_\ast=10, \delta=\frac{1}{2}, H=0$ &$\zeta_\ast=10, \delta=1, H=0$ &$\zeta_\ast=10, \delta=0, H=\frac{1}{2}$ &$\zeta_\ast=10, \delta=0, H=1$\\
			\bottomrule
	\end{tabular}}
	\caption{Settings of the experiments on quadratic functions. Each group has two local objectives ($M=2$). Strictly speaking, the functions in Groups 4, 5, 9, 10 are not quadratic functions.}
	\label{tab:simulation settings}
\end{table}

\begin{figure}[ptbh]
	\centering
	\includegraphics[width=\linewidth]{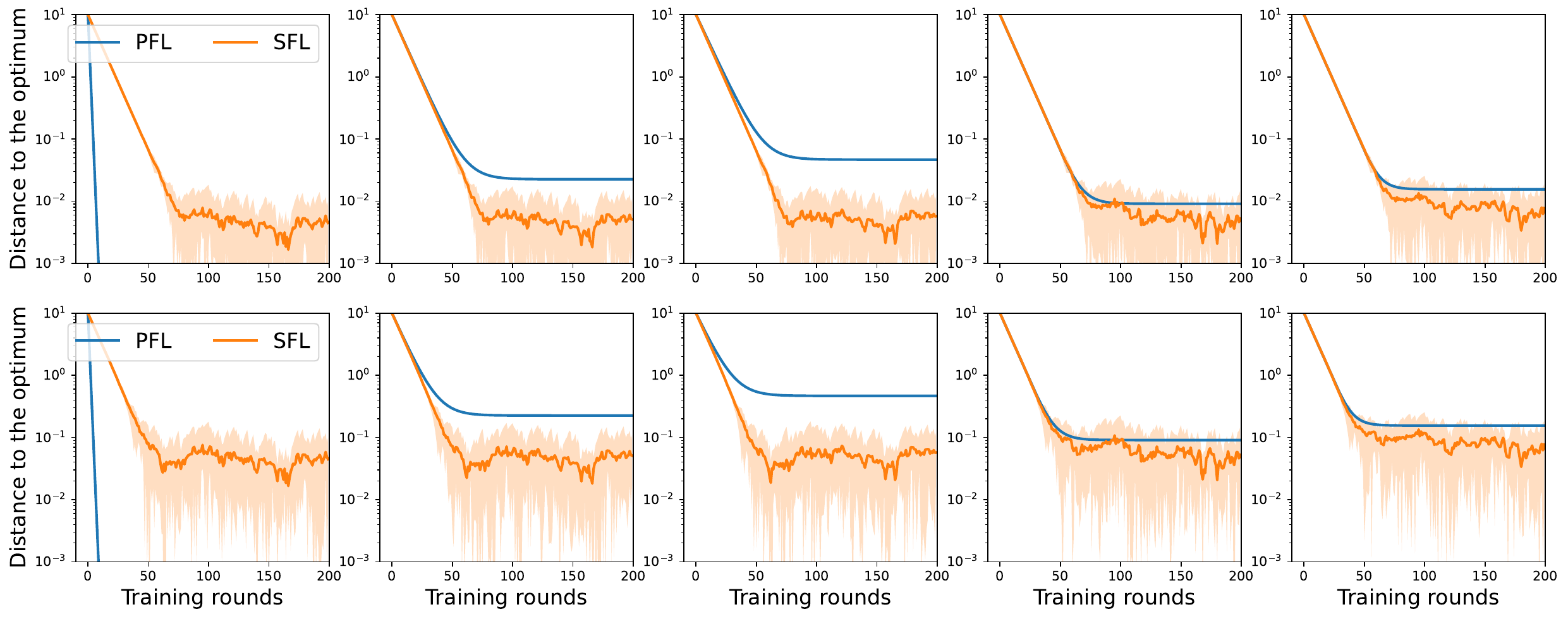}
	\caption{Results of the experiments on quadratic functions. It displays the experimental results of ten groups in Table~\ref{tab:simulation settings}. The top (bottom) row shows the first (last) five groups from left to right. We set $K=10$. The shaded areas show the min-max values across 10 random seeds.}
	\label{fig:quadratic functions}
\end{figure}

\subsection{Experiments on Logistic Regression}\label{subsec:exp:LR}

We consider the classic logistic regression for the binary classification problem \citep{khaled2020tighter,mishchenko2020random, mishchenko2022proximal, mishchenko2022proxskip, malinovsky2023federated,sadiev2023federated}. Specifically, the local objective function $F_m$ is defined as
\begin{align*}
	F_m(\rvx)&= - \left( b_m \log \left( h \left(a_m^\transpose \rvx \right)\right) + (1-b_m) \log \left( 1-h \left( a_m^\transpose \rvx\right) \right) \right) + \frac{1}{2}\omega \norm{\rvx}^2,
\end{align*}
where $\rvx\in \R^d$ is the model parameters, $a_n \in \R^d$, $b_n \in \{0,1\}$ are the data samples, $h: x\to 1/(1+e^{-x})$ is the sigmoid function and $\omega$ is the L2 regularization parameter.

We use two data sets ``a9a'' and ``w8a'' from LIBSVM library \citep{chang2011libsvm}. We partition them into $M=1000$ clients by Extended Dirichlet strategy \citep{li2023convergence}, with each client containing data samples from $C=1, 2$ labels. Larger value of $C$ means higher data heterogeneity. We set the number of local steps to $K=5$, the number of participating clients to $S=10$, and the mini-batch size to 8. The local solver is SGD with learning rate being constant, momentum being 0 and weight decay being 0. We tune the learning rate by the grid search. We run each experiment with 10 different random seeds. 

The experimental results of PFL and SFL are in Figure~\ref{fig:LR}. It can be observed that when the level of heterogeneity is relatively high ($C=1$), the performance of SFL is better than that of PFL, and when the level of heterogeneity is low ($C=2$), the performances are close. This is consistent with our theoretical finding.

\begin{figure}[ptbh]
	\centering
	\includegraphics[width=\linewidth]{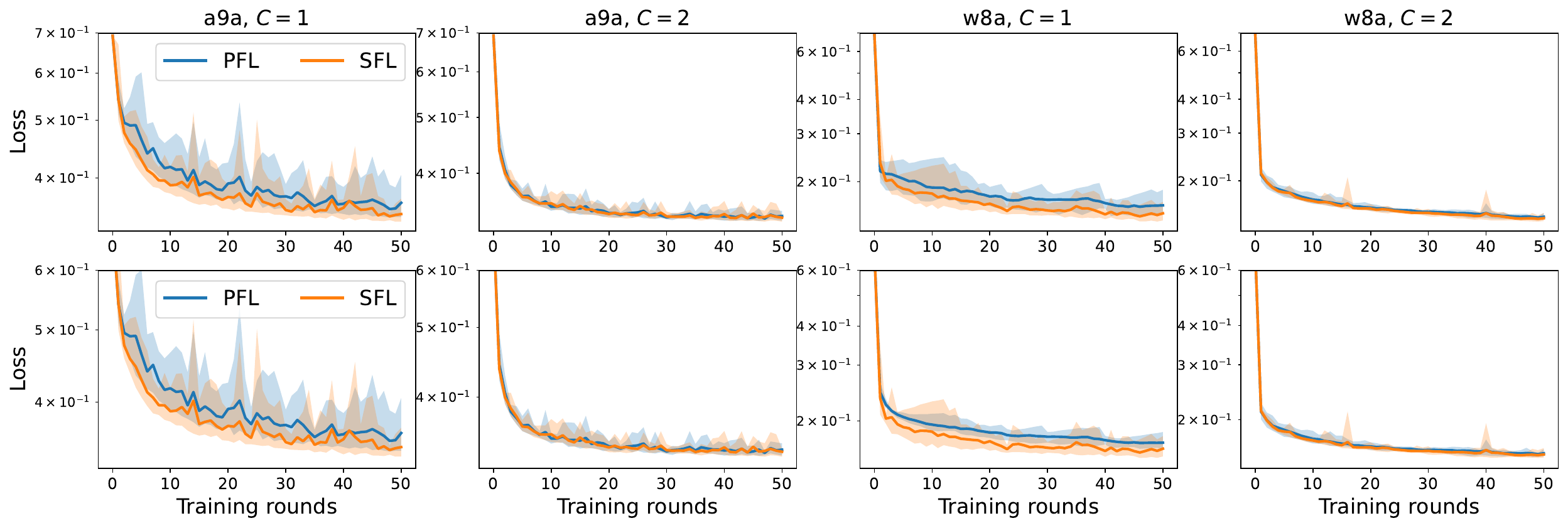}
	\caption{Training loss results of PFL and SFL. The top row shows the results when $\omega = 0.0$ and the bottom row shows the results when $\omega = 0.0001$. The shaded areas show the min-max values across 10 random seeds.}
	\label{fig:LR}
\end{figure}

\subsection{Experiments on Deep Neural Networks}\label{subsec:exp:DNN}

\textit{Setup.} We consider the common CV tasks, with data sets including Fashion-MNIST \citep{xiao2017fashion}, CIFAR-10 \citep{krizhevsky2009learning}, CINIC-10 \citep{darlow2018cinic}. Specifically, we train a CNN model from \cite{wang2022unified} on Fashion-MNIST and a VGG-9 model \citep{simonyan2014very} from \cite{lin2020ensemble} on CIFAR-10 and CINIC-10. We partition the training sets of Fashion-MNIST/CIFAR-10/CINIC-10 into 500/500/1000 clients by Extended Dirichlet strategy \citep{li2023convergence}, with each client containing data samples from $C=1, 2, 5$ labels. Larger value of $C$ means higher data heterogeneity. We spare the original test sets for computing test accuracy. We fix the number of participating clients per round to $S=10$. We fix the number of local update steps to $K=5$ and the mini-batch size to 20 (about one single pass over the local data for each client) \citep{reddi2021adaptive}. The local solver is SGD with learning rate being constant, momentem being 0 and weight decay being 0. We apply gradient clipping to both algorithms and tune the learning rate by grid search with a grid of $\{10^{-2.5}, 10^{-2.0}, 10^{-1.5}, 10^{-1.0}, 10^{-0.5}\}$.

\textit{SFL outperforms PFL on heterogeneous data.} The accuracy results on training data and test data for various tasks are collected in Table~\ref{tab:cross-device settings}. In particular, the test accuracy curves on CIFAR-10 are shown in Figure~\ref{fig:DNN}. It can be observed (i) that when the level of heterogeneity is relatively high (for example, $C=1,2$) the performance of SFL is much better than that of PFL, and (ii) that when the level of heterogeneity is low, the performances of both are close to each other. This is consistent with our theoretical finding.

\begin{figure}[ptbh]
	\centering
	\includegraphics[width=\linewidth]{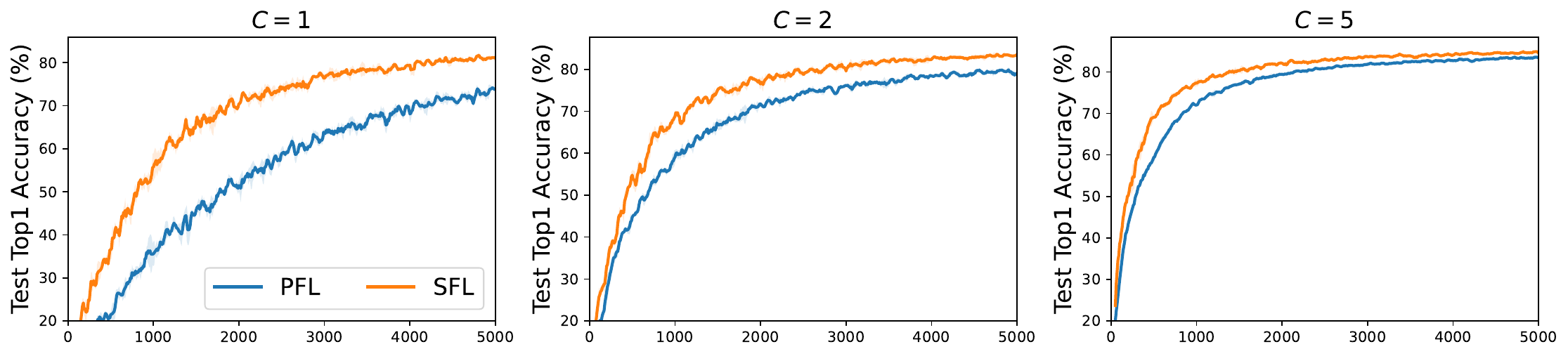}
	\caption{Test accuracy results of PFL and SFL on CIFAR-10. For visualization, we apply moving
		average over a window length of 5 data points. The shaded areas show the standard deviation across 3 random seeds.}
	\label{fig:DNN}
\end{figure}

\begin{table}[ptbh]
	\renewcommand{\arraystretch}{1}
	\centering
	\setlength{\tabcolsep}{0.5em}{
		\resizebox{\linewidth}{!}{
			\begin{tabular}{cccccccc}
				\toprule
				\multirow{2}{*}{Data set}  &\multirow{2}{*}{Method}  &\multicolumn{2}{c}{$C=1$} &\multicolumn{2}{c}{$C=2$} &\multicolumn{2}{c}{$C=5$} \\\cmidrule{3-8}
				&&Train &Test &Train &Test &Train &Test\\\midrule
				\multirow{2}{4.5em}{Fashion-MNIST}
				&PFL &86.71\tiny{$\pm$1.87} &85.77\tiny{$\pm$1.96} &89.73\tiny{$\pm$1.23} &88.55\tiny{$\pm$1.19} &92.27\tiny{$\pm$0.57} &90.70\tiny{$\pm$0.50} \\
				&SFL &88.86\tiny{$\pm$1.60} &87.60\tiny{$\pm$1.56} &91.33\tiny{$\pm$1.49} &89.66\tiny{$\pm$1.41} &92.83\tiny{$\pm$0.69} &90.92\tiny{$\pm$0.64}  \\\midrule
				
				\multirow{2}{4.5em}{CIFAR-10}
				&PFL &76.48\tiny{$\pm$2.03} &73.84\tiny{$\pm$1.90} &85.92\tiny{$\pm$1.77} &78.99\tiny{$\pm$1.43} &94.55\tiny{$\pm$0.36} &83.47\tiny{$\pm$0.48} \\
				&SFL &89.60\tiny{$\pm$2.29} &81.05\tiny{$\pm$1.78} &94.01\tiny{$\pm$0.97} &83.34\tiny{$\pm$0.68} &96.72\tiny{$\pm$0.50} &84.73\tiny{$\pm$0.44} \\\midrule
				
				\multirow{2}{4.5em}{CINIC-10}
				&PFL &53.36\tiny{$\pm$3.80} &52.27\tiny{$\pm$3.61} &65.38\tiny{$\pm$2.01} &61.96\tiny{$\pm$1.81} &74.97\tiny{$\pm$0.95} &68.45\tiny{$\pm$0.81} \\
				&SFL &65.40\tiny{$\pm$3.57} &61.52\tiny{$\pm$3.14} &73.58\tiny{$\pm$2.32} &67.31\tiny{$\pm$1.87} &79.58\tiny{$\pm$1.42} &70.82\tiny{$\pm$1.05} \\
				\bottomrule
	\end{tabular}}}
	\caption{Training and test accuracy results of PFL and SFL on Fashion-MNIST, CIFAR-10 and CINIC-10. We run PFL and SFL for 2000/5000/5000 training rounds on Fashion-MNIST/CIFAR-10/CINIC-10. Results are computed across 3 random seeds and the last 100 training rounds.}
	\label{tab:cross-device settings}
\end{table}

\section{Conclusion}
In this paper, we have derived the upper bounds of SFL for the strongly convex, general convex and non-convex objective functions on heterogeneous data. We validate that the upper bounds of SFL are tight by constructing the corresponding lower bounds of SFL in the strongly convex and general convex cases. We also make comparisons between the upper bounds of SFL and those of PFL. In the convex cases, the comparison results show a subtle difference under different heterogeneity assumptions. That is, under Assumption~\ref{asm:heterogeneity:optimum}, the upper bounds of SFL are better than those of PFL strictly, while under Assumption~\ref{asm:heterogeneity:max}, the upper bounds of SFL are still better unless the level of heterogeneity is very low. In the non-convex case under Assumption~\ref{asm:heterogeneity:average}, the upper bounds of SFL are better without exception. Experiments on quadratic functions, logistic regression and deep neural networks validate the theoretical finding that SFL outperforms PFL on heterogeneous data, at least, when the level of heterogeneity is relatively high.

Although this work has proved that SFL outperforms PFL on heterogeneous data with the standard assumptions, we believe the comparisons are still open. Are there any other conditions to overturn this conclusion? For example, new assumptions beyond the standard assumptions, new factors beyond data heterogeneity, new algorithms beyond vanilla PFL and SFL. One possible future direction is the convergence of PFL and SFL under Hessian assumptions to explain the unexpected results (Groups 1, 6) in Subsection~\ref{subsec:exp:QF}. Another promising future direction is to investigate whether the server extrapolation can be applied to SFL to achieve faster convergence as proven in PFL \citep{jhunjhunwala2023fedexp, li2024power, li2024convergence}.

\section*{Acknowledgments}
This work was supported in part by the National Natural Science Foundation of China under Grant 62371059, and in part by the Fundamental Research Funds for the Central Universities under Grant 2242022k60006.

We thank the  anonymous reviewers of NeurIPS 2023 and JMLR for their insightful suggestions.

\newpage
\appendix

\begin{center}
	\LARGE {Appendix}
\end{center}
\vskip 2ex\hrule\vskip 2ex
{
	\hypersetup{linktoc=page}
	\parskip=0.5ex
	\startcontents[sections]
	\printcontents[sections]{l}{1}{\setcounter{tocdepth}{3}}
}
\vskip 2ex\hrule\vskip 2ex

\newpage
\section{Notations}\label{app:notation}
Table~\ref{tab:notations} summarizes the notations appearing in this paper.

\begin{table}[ht]
	\renewcommand{\arraystretch}{1.1}
	\centering
	\resizebox{\linewidth}{!}{
		\begin{tabular}{cl}
			\toprule
			Symbol &Description\\ \midrule
			$R, r$ &number, index of training rounds \\
			$M, m$ &number, index of clients \\
			$K, k$ &number, index of local update steps \\
			$\gS, S$ &the set of participating clients and its size \\
			$\pi$ &$\{\pi_1, \pi_2, \ldots, \pi_M\}$ is a permutation of $[M]$\\ 
			$\eta$, $\tilde\eta$ &learning rate, effective learning rate ($\eta_{\text{SFL}} \coloneqq \eta MK$ and $\eta_{\text{PFL}} \coloneqq \eta K$)\\
			$L, \mu, \kappa$ & constants in Asm.~\ref{def:smoothness} and Asm.~\ref{def:strong convexity}; conditional number $\kappa\coloneqq L/\mu$ \\
			$\sigma_\ast, \sigma$ &constants in Asm.~\ref{asm:stochasticity:optimum} and Asm.~\ref{asm:stochasticity:uniform} for stochasticity\\
			$\zeta_\ast, \zeta (\beta), \hat\zeta$ &constants in Asm.~\ref{asm:heterogeneity:optimum}, Asm.~\ref{asm:heterogeneity:average} and Asm.~\ref{asm:heterogeneity:max} for heterogeneity\\
			$\delta, H$ &constants in Subsection~\ref{subsec:exp:QF} for Hessian\\
			$F,F_m,f_m$ &global objective, local objective and local component function\\
			$\rvx^{(r)}$ &global model parameters in the $r$-th round \\
			$\rvx_{m,k}^{(r)}$ &local model parameters of the $m$-th client after $k$ local steps in the $r$-th round \\[1ex]
			$\rvg_{\pi_m,k}^{(r)}$ &$\rvg_{\pi_m,k}^{(r)} \coloneqq \nabla f_{\pi_m}(\rvx_{m,k}^{(r)};\xi_{m,k}^{(r)})$ is the stochastic gradients of $F_{\pi_m}$ regarding $\rvx_{m,k}^{(r)}$ \\[1ex]
			$C$ &each client containing data samples from $C$ labels (Subsec.~\ref{subsec:exp:DNN})\\
			\bottomrule
	\end{tabular}}
	\caption{Summary of key notations.}
	\label{tab:notations}
\end{table}

\section{Proofs of Theorem~\ref{thm:lower bound} and Theorem~\ref{thm:lower bound2}}\label{app:lower bound}
In this section, we use the results in Appendices~\ref{app:lower bound-stochasticity} and~\ref{app:lower bound-heterogeneity} to compose the final lower bounds.

For clarity, we have summarized the lower bounds and the corresponding setups for these regimes in Tables~\ref{tab:lower bounds:stochasticity} and~\ref{tab:lower bounds:heterogeneity}. Since we use the typical objective functions in Appendices~\ref{app:lower bound-stochasticity} and~\ref{app:lower bound-heterogeneity}, we omit the step to verify that these functions satisfy the assumptions in Theorem~\ref{thm:lower bound} (see the proofs of \citealt{cha2023tighter}'s Theorem 3.1 about this step if needed).

\subsection{Proof of Theorem~\ref{thm:lower bound}}\label{app:subsec:lower bound-proof}
\begin{proof}
For the stochasticity terms, we let $\lambda = \mu$ for $\eta \leq \frac{1}{102010\lambda NR}$; we let $\lambda_1= \mu$ and $\lambda_0 = 1010\mu$ for $\frac{1}{102010\lambda_1 NR}\leq\eta\leq \frac{1}{101\lambda_0 N}$; we let $\lambda = 1010\mu$ for the other regimes in Table~\ref{tab:lower bounds:stochasticity}. There exist a 4-dimensional global objective function and an initialization point $\rvx^{(0)} = \left[ \frac{\sigma}{\mu}, \frac{1}{8160800}\frac{\sigma}{\mu N^{\frac{1}{2}}R}, 0, \frac{\sigma}{\mu} \right]^\transpose$, such that for $\kappa= \frac{L}{\mu}\geq 2020$, $M\geq 4$ and $R\geq 1$, the lower bound for stochasticity terms satisfies
\begin{align*}
	\E\left[F(\rvx^{(R)})  - F(\rvx^\ast)\right] = \Omega \left( \min \left\{\frac{\sigma^2}{\mu}, \frac{\sigma^2}{\mu NR^2}, \frac{\sigma^2}{\mu N}, \frac{\sigma^2}{\mu}\right\}\right) = \Omega \left( \frac{\sigma^2}{\mu NR^2}\right),
\end{align*}
where $\rvx^{(R)} = \left[ x_1^{(R)}, x_2^{(R)}, x_3^{(R)}, x_4^{(R)} \right]^\transpose$ and $\rvx^{\ast} = \left[ 0, 0, 0, 0\right]^\transpose$.

Notably, these dimensions are orthogonal. For one single round in SFL, with $N=MK$, the lower bound is $\Omega \left( \frac{\sigma^2}{\mu MKR^2}\right)$. It is well known that any first-order method which accesses at most $MKR$ stochastic gradients with variance $\sigma^2$ for a $\mu$-strongly convex objective will suffer error at least $\gO\left( \frac{\sigma^2}{\mu MKR} \right)$ in the worst case \citep{nemirovskij1983problem, woodworth2020local,woodworth2020minibatch}. Therefore, we get the lower bound of $\Omega \left( \frac{\sigma^2}{\mu MKR}\right)$ for stochasticity terms.

For the heterogeneity terms, we let $\lambda = \mu$ for $\eta \leq \frac{1}{102010\lambda MKR}$; we let $\lambda_1= \mu$ and $\lambda_0 = 1010\mu$ for $\frac{1}{102010\lambda_1 MKR}\leq\eta\leq \frac{1}{101\lambda_0 MK}$; we let $\lambda = 1010\mu$ for the other regimes. There exist a 5-dimensional global objective function and an initialization point $\rvx^{(0)} = \left[ \frac{\zeta}{\mu}, \frac{1}{81608000}\frac{\zeta}{\mu M^{\frac{1}{2}}R}, 0, 0, \frac{\zeta}{\mu} \right]^\transpose$, such that for $\kappa = \frac{L}{\mu}\geq 1010$, $M\geq 4$ and $R\geq 1$, the lower bound for heterogeneity terms satisfies
\begin{align*}
	\E\left[F(\rvx^{(R)})  - F(\rvx^\ast)\right] = \Omega \left( \min \left\{\frac{\zeta^2}{\mu}, \frac{\zeta^2}{\mu MR^2}, \frac{\zeta^2}{\mu M},\frac{\zeta^2}{\mu}, \frac{\zeta^2}{\mu} \right\}\right) = \Omega \left( \frac{\zeta^2}{\mu MR^2}\right).
\end{align*}

Combining these cases, we get the final lower bound of $\Omega \left( \frac{\sigma^2}{\mu MKR} + \frac{\zeta^2}{\mu MR^2}\right)$.
\end{proof}

\subsection{Proof of Theorem~\ref{thm:lower bound2}}\label{app:subsec:cor-lower bound}
\begin{proof}
	In this theorem, by using the small $\eta = \gO \left( \frac{1}{L MK} \right)$, we can extend the lower bound in Theorem~\ref{thm:lower bound} to arbitrary weighted average global parameters $\bar\rvx^{(R)} = \frac{\sum_{r=0}^{R} w_r \rvx^{(r)}}{\sum_{r=0}^{R} w_r}$, and even the general convex case.
	
	When choosing the small $\eta = \gO \left( \frac{1}{L MK} \right)$, we only need to consider the first two regimes $\eta \leq \frac{1}{102010\lambda MKR}$ and $\frac{1}{102010\lambda_1MKR} \leq \eta\leq \frac{1}{101\lambda_0 MK}$ for both stochasticity terms and heterogeneity terms. Take the heterogeneity terms as an example. In these two regimes, we can lower bound $\E\left[x^{(r)}\right]$, that is, $\E\left[x^{(r)}\right] \geq \frac{51004}{51005} D_1$ ($D_1 = \abs{x^{(0)}-x^\ast}$ is the initial distance for the first dimension) for $\eta \leq \frac{1}{102010\lambda MKR}$ and $\E\left[ x^{(r)} \right]\geq \frac{\zeta}{81608000\lambda_1 M^{\frac{1}{2}} R}$ for $\frac{1}{102010\lambda_1MKR} \leq \eta\leq \frac{1}{101\lambda_0 MK}$. With these lower bounds, the arbitrary weighted average parameters $\bar x^{(R)}$ can also be bounded
	\begin{align*}
		\bar x^{(R)} = \frac{\sum_{r=0}^{R} w_r x^{(r)}}{\sum_{r=0}^{R} w_r} \geq \frac{\sum_{r=0}^{R} w_r c}{\sum_{r=0}^{R} w_r} \geq c,
	\end{align*}
	where $c = \frac{51004}{51005} D_1$ for $\eta \leq \frac{1}{102010\lambda MKR}$ and $c=\frac{\zeta}{81608000\lambda_1 M^{\frac{1}{2}}R}$ for $\frac{1}{102010\lambda_1MKR} \leq \frac{1}{101\lambda_0 MK}$. The stochasticity terms are similar and thus be omitted here. Refer to the proofs of \cite{cha2023tighter}'s Theorem~3.3 for details if needed. In summary, we can get the same lower bound for $\bar\rvx^{(R)}$ and $\rvx^{(R)}$ in the first two regimes for stochasticity and heterogeneity.
	
	\subsubsection{Strongly Convex Case}
	For the stochasticity terms, we let $\lambda = \mu$ for $0<\eta\leq \frac{1}{102010\lambda NR}$; we let $\lambda_1 = \mu, \lambda_0 = L$ for $\frac{1}{102010\lambda_1 NR}\leq \eta\leq \frac{1}{101\lambda_0 N}$ (with $N=MK$) in Table~\ref{tab:lower bounds:stochasticity}. There exist a 2-dimensional global objective function and an initialization point $\rvx^{(0)} = \left[ \frac{\sigma}{\mu}, \frac{1}{8160800}\frac{\sigma}{\mu M^{\frac{1}{2}}K^{\frac{1}{2}}R} \right]^\transpose$, such that for $\kappa \geq 1010$, $M\geq 4$ and $R\geq \frac{1}{1010} \kappa$, the lower bound for stochasticity terms satisfies
	\begin{align*}
		\E\left[F(\bar\rvx^{(R)})  - F(\rvx^\ast)\right] = \Omega \left( \min \left\{\frac{\sigma^2}{\mu}, \frac{L\sigma^2}{\mu^2 MKR^2}\right\}\right) = \Omega \left( \frac{L\sigma^2}{\mu^2 MKR^2}\right).
	\end{align*}
	Therefore, with the same logic in the proof of Theorem~\ref{thm:lower bound}, we get the lower bound of $\Omega \left( \frac{\sigma^2}{\mu MKR} + \frac{L\sigma^2}{\mu^2 MKR^2}\right)$ for stochasticity terms.
	
	For heterogeneity terms, we let $\lambda = \mu$ for $0<\eta\leq \frac{1}{102010\lambda MKR}$; we let $\lambda_1 = \mu, \lambda_0 = L$ for $\frac{1}{102010\lambda_1 MKR}\leq \eta\leq \frac{1}{101\lambda_0 MK}$ in Table~\ref{tab:lower bounds:heterogeneity}. There exist a 2-dimensional global objective function and an initialization point $\rvx^{(0)} = \left[ \frac{\zeta}{\mu}, \frac{1}{81608000}\frac{\zeta}{\mu M^{\frac{1}{2}}R} \right]^\transpose$, such that for $\kappa \geq 1010$, $M\geq 4$ and $R\geq \frac{1}{1010} \kappa$, the lower bound for heterogeneity terms satisfies
	\begin{align*}
		\E\left[F(\bar\rvx^{(R)})  - F(\rvx^\ast)\right] = \Omega \left( \min \left\{\frac{\zeta^2}{\mu}, \frac{L\zeta^2}{\mu^2 MR^2}\right\}\right) = \Omega \left( \frac{L\zeta^2}{\mu^2 MR^2}\right).
	\end{align*}
	Therefore, we get the lower bound of $\Omega \left( \frac{L\zeta^2}{\mu^2 MR^2}\right)$ for heterogeneity terms.
	
	Combining them, we get the lower bound of $\Omega \left( \frac{\sigma^2}{\mu MKR} + \frac{L\sigma^2}{\mu^2 MKR^2} + \frac{L\zeta^2}{\mu^2 MR^2}\right)$.
	
	\subsubsection{General Convex Case}
	As done in \cite{woodworth2020minibatch}'s Theorem~2 and \cite{cha2023tighter}'s Corollary 3.5, we need to choose $\lambda$, $\lambda_0$ and $\lambda_1$ more carefully for the general convex case.

	For the stochasticity terms, we let $\lambda = \frac{L^{1/3}\sigma^{2/3}}{M^{1/3} K^{1/3} R^{2/3} D^{2/3}}$ for the first regime; we let $\lambda_1 = \frac{L^{1/3}\sigma^{2/3}}{M^{1/3} K^{1/3} R^{2/3} D^{2/3}}$, $\lambda_0 = L$ for the second regime in Table~\ref{tab:lower bounds:stochasticity}. For the heterogeneity terms, we let $\lambda = \frac{L^{1/3}\zeta^{2/3}}{M^{1/3}R^{2/3}D^{2/3}}$ for the first regime; we let $\lambda_1 = \frac{L^{1/3}\zeta^{2/3}}{M^{1/3}R^{2/3}D^{2/3}}$, $\lambda_0 = L$ for the second regime in Table~\ref{tab:lower bounds:heterogeneity}. Here $D$ is the initial distance to the optimum $\rvx^\ast = [0,0,0,0]^\transpose$.
	
	Considering that $D$ is affected by both stochasticity and heterogeneity, we consider a 4-dimensional global objective function. We let $D_1$, $D_2$, $D_3$ and $D_4$ are the initial distance in the first, second, third and forth dimensions, respectively. Then, if
	\begin{align*}
		\rvx^{(0)} = \left[ D_1, \frac{1}{8160800}\frac{\sigma^{1/3}D^{2/3}}{L^{1/3}M^{1/6}K^{1/6}R^{1/3}}, D_1, \frac{1}{81608000}\frac{\zeta^{1/3}D^{2/3}}{L^{1/3} M^{1/6}R^{1/3}} \right]^\transpose,
	\end{align*}
	then
	\begin{align*}
		\E\left[F(\bar\rvx^{(R)})-F(\rvx^\ast)\right] &= \Omega \left( \min \left\{\frac{L^{1/3}\sigma^{2/3}D_1^2}{M^{1/3}K^{1/3}R^{2/3}D^{2/3}}, \frac{L^{1/3}\sigma^{2/3}D^{4/3}}{ M^{1/3}K^{1/3} R^{2/3}}\right\} \right) \\
		&\quad +\Omega \left( \min \left\{\frac{L^{1/3}\zeta^{2/3}D_1^2}{M^{1/3}R^{2/3}D^{2/3}}, \frac{L^{1/3}\zeta^{2/3}D^{4/3}}{ M^{1/3} R^{2/3}}\right\} \right).
	\end{align*}
	Then, since
	\begin{align*}
		2D_1^2 &= D^2 - D_2^2- D_4^2 \\
		&= D^2 \left(1- \frac{1}{8160800^2} \frac{\sigma^{2/3}}{L^{2/3}M^{1/3}K^{1/3}R^{2/3}D^{2/3}} - \frac{1}{81608000^2}\frac{\zeta^{2/3}}{L^{2/3}M^{1/3}R^{2/3}D^{2/3}}\right)\\
		&\geq D^2 \left( 1- \frac{1}{8160800^2} - \frac{1}{81608000^2}\right) \geq \frac{1}{2}D^2,
	\end{align*}
	where we use the conditions $R\geq \frac{ \sigma }{LM^{1/2}K^{1/2}D}$ and $R\geq \frac{ \zeta }{LM^{1/2}D}$. Thus, we get
	\begin{align*}
		\E\left[F(\bar\rvx^{(R)})-F(\rvx^\ast)\right] &= \Omega \left( \frac{L^{1/3}\sigma^{2/3}D^{4/3}}{ M^{1/3}K^{1/3} R^{2/3}} + \frac{L^{1/3}\zeta^{2/3}D^{4/3}}{ M^{1/3} R^{2/3}}\right).
	\end{align*}
	Adding the classic bound $\Omega \left(\frac{\sigma D}{\sqrt{MKR}} \right)$ for the general convex case \citep{woodworth2020local,woodworth2020minibatch} yields the final result.
	
	To ensure that the objective functions satisfy the assumptions, we use the conditions
	\begin{align*}
		2\lambda \leq L &\implies R\geq 2^{\frac{3}{2}}\cdot  \frac{\sigma}{LM^{1/2}K^{1/2}D} ,\\
		R\geq \frac{1}{1010} \frac{\lambda_0}{\lambda_1} &\implies R\geq \frac{1}{1010^3} \cdot \frac{L^2 MKD^2}{\sigma^2}.
	\end{align*}
	for stochasticity terms and the conditions
	\begin{align*}
		2\lambda \leq L &\implies R\geq 2^{\frac{3}{2}}\cdot  \frac{\zeta}{LM^{1/2}D} ,\\
		R\geq \frac{1}{1010} \frac{\lambda_0}{\lambda_1} &\implies R\geq \frac{1}{1010^3} \cdot \frac{L^2 MD^2}{\zeta^2}.
	\end{align*}
	for heterogeneity terms. Note that the choices of $\lambda$, $\lambda_0$ and $\lambda_1$ are different for stochasticity terms and heterogeneity terms. For simplicity, we can use a stricter condition 
	\begin{align*}
		R\geq 4\max\left\{\frac{\sigma}{LM^{1/2}K^{1/2}D},  \frac{L^2 MKD^2}{\sigma^2},\frac{\zeta}{LM^{1/2}D},  \frac{L^2 MD^2}{\zeta^2} \right\}.
	\end{align*}
\end{proof}

\section{Proof of Stochasticity Terms in Theorem~\ref{thm:lower bound}}\label{app:lower bound-stochasticity}
\begin{proof}
	We assume $F_1= F_2= \cdots=F_M=F$, and then the task is to construct the lower bound of vanilla SGD, where one objective function is sampled with replacement for updates in each step: $x_{n+1} = x_n - \eta \nabla f_{\pi_n} (x_n)$. The results are summarized in Table~\ref{tab:lower bounds:stochasticity}. Notably, $\lambda_0, \lambda_1$ and $\lambda$ in different regimes can be different.
	\begin{table}[h]
		\centering
		\resizebox{\linewidth}{!}{
			\begin{threeparttable}[b]
				\renewcommand{\arraystretch}{1.2}
				\begin{tabular}{lllll}
					\toprule
					Bound &Regime &Objective functions &Initialization &$\nabla^2 F_m \in$\\
					\midrule
					$\Omega\left(\frac{\sigma^2}{\lambda}\right)$ &$\eta\leq \frac{1}{102010\lambda NR}$ &$f_{\pi_n} (x) = \lambda x^2$ &$x^{(0)}=\frac{\sigma}{\lambda}$ &$[2\lambda, 2\lambda]$\\\cmidrule{1-5}
					
					$\Omega\left(\frac{\lambda_0\sigma^2}{\lambda_1^2 NR^2}\right)$ &\makecell[l]{$\frac{1}{102010\lambda_1 NR}\leq\eta$\\[2ex] and $\eta \leq\frac{1}{101\lambda_0 N}$} &\makecell[l]{$f_{\pi_n} (x) $\\ $= (\lambda_0\mathbbm{1}_{x<0} + \lambda \mathbbm{1}_{x\geq 0})\frac{x^2}{2} + \sigma\tau_n x$} &\makecell[l]{$x^{(0)}=$\\$\frac{1}{8160800}\frac{\sigma}{\lambda_1 N^{\frac{1}{2}}R}$} &$\left[\frac{\lambda_0}{1010}, \lambda_0\right]$\\\cmidrule{1-5}
				
					$\Omega\left(\frac{\sigma^2}{\lambda N}\right)$ &$\frac{1}{101\lambda N}\leq\eta\leq\frac{1}{\lambda}$ &$f_{\pi_n} (x) = \frac{\lambda}{2}x^2 + \sigma\tau_n x$ &$x^{(0)}=0$ &$[\lambda, \lambda]$\\\cmidrule{1-5}
					
					$\Omega\left(\frac{\sigma^2}{\lambda}\right)$ &$\eta\geq\frac{1}{\lambda}$ &$f_{\pi_n} (x) = \lambda x^2$ &$x^{(0)}=\frac{\sigma}{\lambda}$ &$[2\lambda, 2\lambda]$\\
					\bottomrule
				\end{tabular}
		\end{threeparttable}}
		\caption{Lower bounds of SFL for stochasticity terms. It requires that $\lambda = \frac{\lambda_0}{1010}$ and $R\geq \frac{1}{1010} \frac{\lambda_0}{\lambda_1}$ in the regime $\frac{1}{102010\lambda_1 NR} \leq \eta \leq \frac{1}{101\lambda_0N}$. We set $N=MK$ in SFL.}
		\label{tab:lower bounds:stochasticity}
	\end{table}

\subsection{Lower Bounds for $0< \eta \leq \frac{1}{102010\lambda NR}$ and $\eta \geq \frac{1}{\lambda}$}
In this regime, we consider the following objective functions
\begin{align*}
	&f_{\pi_n} (x) = \lambda x^2, \\
	&f(x) = \E\left[f_{\pi_n} (x)\right] = \lambda x^2.
\end{align*}
We can soon build the relationship between $x^{(R)}$ and $x^{(0)}$: $x^{(R)} = (1-2\lambda\eta)^{NR} x^{(0)}$.

\subsubsection{Lower Bound for $0<\eta \leq \frac{1}{102010\lambda NR}$}
Since $\eta \leq \frac{1}{\lambda NR}$ and $(1-\frac{1}{51005}\cdot \frac{1}{x})^x$ is monotonically increasing when $x \geq 1$, we have
\begin{align*}
	&x^{(R)} = (1-2\lambda\eta)^{NR} x^{(0)} \geq (1-\frac{1}{51005}\cdot \frac{1}{NR})^{NR} x^{(0)} \geq \frac{51004}{51005}x^{(0)} \tag{$\because N\geq 1$}\\
	&F(x^{(R)}) = \lambda \left(x^{(R)}\right)^2 \geq \frac{51004^2}{51005^2}\lambda\left(x^{(0)}\right)^2.
\end{align*}
If $x^{(0)} = \frac{\sigma}{\lambda}$, we can get
\begin{align*}
	\E\left[F(x^{(R)})-F^\ast\right] = \E\left[F(x^{(R)})\right] \geq \frac{51004^2}{51005^2}\lambda\left(x^{(0)}\right)^2 = \Omega\left(\frac{\sigma^2}{\lambda}\right).
\end{align*}

\subsubsection{Lower Bound for $\eta \geq \frac{1}{\lambda}$}
Since $\eta \geq \frac{1}{\lambda}$ implies that $(1-2\lambda\eta)^2 \geq 1$, we have
\begin{align*}
	F(x^{(R)}) = \lambda \left(x^{(R)}\right)^2 = \lambda(1-2\lambda\eta)^{2NR} \left(x^{(0)}\right)^2 \geq \lambda \left(x^{(0)}\right)^2.
\end{align*}
If $x^{(0)} = \frac{\sigma}{\lambda}$, we can get
\begin{align*}
	\E\left[F(x^{(R)})-F^\ast\right] = \E\left[F(x^{(R)})\right] \geq \lambda\left(x^{(0)}\right)^2 = \frac{\sigma^2}{\lambda} = \Omega\left(\frac{\sigma^2}{\lambda}\right).
\end{align*}

\subsection{Lower Bound for $\frac{1}{102010\lambda_1 NR}\leq \eta \leq \frac{1}{101\lambda_0 N}$}
In this regime, we consider the following functions
\begin{align*}
	&f_{\pi_n} (x) = (\lambda_0\mathbbm{1}_{x<0} + \lambda \mathbbm{1}_{x\geq 0})\frac{x^2}{2} + \tau_{n}\sigma x, \\
	&f(x) = \E \left[f_{\pi_n}(x)\right] = (\lambda_0\mathbbm{1}_{x<0} + \lambda \mathbbm{1}_{x\geq 0})\frac{x^2}{2} \quad \text{($\lambda_0/\lambda\geq 1010$)},
\end{align*}
where $\tau_n$ is a random variable with equal probabilities of being either ``$+1$'' or ``$-1$''. Next, we focus on a single round $r$, including $N$ local steps in total, and thus we drop the superscripts $r$ for a while, for example, replacing $x_{n}^{(r)}$ with $x_{n}$. Unless otherwise stated, the expectation is conditioned on $x_{0}$ when we focus on one single round.

The relationship between the current parameter $x_n$ and the initial parameter $x_0$ satisfies
\begin{align}
	x_n = x_0 - \eta \sum_{i=0}^{n-1} (\lambda_0\mathbbm{1}_{x_i<0} + \lambda \mathbbm{1}_{x_i\geq 0}) x_{i} - \eta\sigma\gE_n.\label{eq:proof:lower bound-stochasticity:curr-init}
\end{align}

\subsubsection{Lower Bound of $\E\left[ x^{(r+1)}\mid x^{(r)} \geq 0 \right]$}

We first bound $\E\left[ x^{(r+1)}\mid x^{(r)} \geq 0 \right]$. Since $(\lambda_0\mathbbm{1}_{x<0} + \lambda \mathbbm{1}_{x\geq 0}) x \leq \lambda_0 x$ and $(\lambda_0\mathbbm{1}_{x<0} + \lambda \mathbbm{1}_{x\geq 0}) x \leq \lambda x$, we have
\begin{align}
	\E\left[ (\lambda_0\mathbbm{1}_{x_n<0} + \lambda \mathbbm{1}_{x_n\geq 0}) x_{n} \right] &= \Pr(\gE_n > 0) \E\left[ (\lambda_0\mathbbm{1}_{x_n<0} + \lambda \mathbbm{1}_{x_n\geq 0}) x_{n} \mid \gE_n >0\right] \nonumber\\
	&\quad+ \Pr(\gE_n \leq 0) \E\left[ (\lambda_0\mathbbm{1}_{x_n<0} + \lambda \mathbbm{1}_{x_n\geq 0}) x_{n} \mid \gE_n \leq0\right]\nonumber\\
	&\leq \lambda_0 \Pr(\gE_n > 0) \E\left[ x_{n} \mid \gE_n >0\right] + \lambda \Pr(\gE_n \leq 0) \E\left[ x_{n} \mid \gE_n \leq0\right].\label{eq:proof:lower bound-stochasticity:curr lower bound}
\end{align}
According to Eq.~\eqref{eq:proof:lower bound-stochasticity:curr-init}, we can get
\begin{align*}
	\E\left[ x_{n} \mid \gE_n >0\right] &= \E\left[ x_0 - \eta \sum_{i=0}^{n-1} (\lambda_0\mathbbm{1}_{x_i<0} + \lambda \mathbbm{1}_{x_i\geq 0}) x_{i} - \eta\sigma\gE_n \mid \gE_n >0\right]\\
	&\leq x_0 + \lambda_0 \eta \sum_{i=0}^{n-1} \E\left[\abs{x_i - x_0}\mid \gE_n >0\right] - \eta \sigma \E\left[\gE_n \mid \gE_n >0\right].
\end{align*}
Then using $\E\left[\abs{x_i - x_0}\right] \geq \E\left[\abs{x_i - x_0}\mid \gE_n >0\right] \Pr(\gE_n >0)$ with $\Pr(\gE_n >0) \geq \frac{1}{4}$ for the second term, and $\E\left[\abs{\gE_n}\right] = 2\E\left[\gE_n>0\mid \gE_n>0\right]\Pr(\gE_n>0)$ with $\Pr(\gE_n >0) \leq \frac{1}{2}$, we can get
\begin{align*}
	\E\left[ x_{n} \mid \gE_n >0\right] &\leq x_0 + 4\lambda_0 \eta  \sum_{i=0}^{n-1} \E\left[\abs{x_i - x_0}\right] - \frac{1}{2}\eta \sigma \E\left[\abs{\gE_n}\right]\\
	&\leq (1+\frac{1}{2525})x_0 - \frac{6}{100}\eta \sigma \sqrt{n},
\end{align*}
where Lemmas~\ref{lem:lower bound-stochasticity:diff curr-init},~\ref{lem:lower bound-stochasticity:partial sum} and $\lambda_0 \eta N \leq \frac{1}{101}$ are applied in the last inequality. We bound $\E\left[ x_{n} \mid \gE_n \leq 0\right]$ with a looser bound as
\begin{align*}
	\E\left[ x_{n} \mid \gE_n \leq 0\right] &\leq x_0 + \E\left[ \abs{x_{n}-x_{0}} \mid \gE_n \leq 0\right]\\
	&\leq x_0 + 2\E\left[ \abs{x_{n}-x_{0}} \right] \tag{$\because \Pr(\gE_n\leq 0) \geq \frac{1}{2}$}\\
	&\leq \frac{51}{50}\lambda_0 x_0 + \frac{101}{50} \eta \sigma \sqrt{n}.
\end{align*}
Then, back to Ineq.~\eqref{eq:proof:lower bound-stochasticity:curr lower bound}, we have
\begin{align*}
	\E\left[ (\lambda_0\mathbbm{1}_{x_n<0} + \lambda \mathbbm{1}_{x_n\geq 0}) x_{n} \right] 
	&\leq \frac{1}{2}\lambda_0 \left( (1+\frac{1}{2525})x_0 - \frac{6}{100}\eta \sigma \sqrt{n} \right) + \frac{3}{4} \lambda \left( \frac{51}{50}x_0 + \frac{101}{50} \eta \sigma \sqrt{n} \right)\\
	&\leq \frac{253}{505}x_0 - \frac{1}{40} \eta \sigma\sqrt{n}.
\end{align*}
Now, we have
\begin{align*}
	\E[x_N] &= x_0 - \eta \sum_{n=0}^{N-1} (\lambda_0\mathbbm{1}_{x_n<0} + \lambda \mathbbm{1}_{x_n\geq 0}) x_{n} \tag{$\because \E[\gE_N] = 0$}\\
	&\geq x_0 - \eta \sum_{n=0}^{N-1} \left(\frac{253}{505}\lambda_0 x_0 - \frac{1}{40} \eta \sigma\sqrt{n} \right)\\
	&\geq \left( 1-\frac{2}{3}\lambda_0 \eta N \right)x_0 + \frac{1}{60} \eta N^{\frac{3}{2}}\sigma.
\end{align*}
That is,
\begin{align*}
	\E\left[x^{(r+1)}\mid x^{(r)}\geq 0\right] \geq \left(1-\frac{2}{3}\lambda_0 N\eta\right)\E\left[x^{(r)}\mid x^{(r)}\geq 0\right]+ \frac{1}{60} \eta N^{\frac{3}{2}}\sigma.
\end{align*}

\subsubsection{Lower Bound of $\E\left[ x^{(r+1)}\mid x^{(r)} < 0 \right]$}
With similar analyses in Subsection~\ref{subsubsection:intermediate func-1} and \cite{cha2023tighter}' Lemma B.4, we get
\begin{align*}
	\E\left[x^{(r+1)}\mid x^{(r)}< 0\right] \geq \left(1-\frac{2}{3}\lambda_0 N\eta\right)\E\left[x^{(r)}\mid x^{(r)}< 0\right].
\end{align*}

\subsubsection{Relationship Between $\Pr(x^{(r)}) \geq 0$ and $\Pr(x^{(r)}) < 0$.}
With similar analyses in Subsection~\ref{subsubsection:intermediate func-2} and \cite{cha2023tighter}' Lemma B.4, we get
\begin{align*}
	\Pr\left(x^{(r)}\geq 0\right) \geq \frac{1}{2},
\end{align*}
when $x^{(0)} \geq 0$.

\subsubsection{Lower Bound for $\frac{1}{102010\lambda_1 NR}\leq \eta \leq \frac{1}{101\lambda_0 N}$}
With the above bounds for $\E\left[ x^{(r+1)}\mid x^{(r)} \geq 0 \right]$ and $\E\left[ x^{(r+1)}\mid x^{(r)} < 0 \right]$, we have
\begin{align*}
	\E\left[x^{(r+1)}\right] &= \E\left[x^{(r+1)}\mid x^{(r)}\geq 0\right]\Pr\left(x^{(r)}\geq 0\right) + \E\left[x^{(r+1)}\mid x^{(r)}< 0\right]\Pr\left(x^{(r)}< 0\right)\\
	&\geq \left(1-\frac{2}{3}\lambda_0 N\eta\right)x^{(r)} + \frac{1}{120}\lambda_0 N^{\frac{3}{2}}\eta^2\sigma  .\tag{$\because \Pr\left(x^{(r)}\geq 0\right)\geq \frac{1}{2}$}
\end{align*}
If $x^{(r)} \geq \frac{1}{8160800}\cdot \frac{\sigma}{\lambda_1 N^{\frac{1}{2}}R}$, then using $\eta \geq \frac{1}{102010\lambda_1 NR}$, we have
\begin{align*}
	x^{(r+1)} &\geq \left(1-\frac{2}{3}\lambda_0 N\eta\right)x^{(r)} + \frac{1}{120}\lambda_0 N^{\frac{3}{2}}\eta^2\sigma \\
	&\geq \left(1-\frac{2}{3}\lambda_0 N\eta\right)\cdot \frac{1}{8160800}\cdot \frac{\sigma}{\lambda_1 N^{\frac{1}{2}}R} + \frac{1}{120}\lambda_0 N^{\frac{3}{2}}\eta\sigma \cdot \frac{1}{102010\lambda_1 NR} \\
	&\geq \frac{1}{8160800}\cdot \frac{\sigma}{\lambda_1 N^{\frac{1}{2}}R}.
\end{align*}
Therefore, if we set $x^{(0)}\geq \frac{1}{8160800}\cdot \frac{\sigma}{\lambda_1 N^{\frac{1}{2}}R}$, then the final parameters will also maintain $x^{(R)}\geq \frac{1}{8160800}\cdot \frac{\sigma}{\lambda_1 N^{\frac{1}{2}}R}$. Then, noting that $\frac{\lambda_0}{\lambda} \geq 1010$, we can choose $\frac{\lambda_0}{\lambda} = 1010$. Then,
\begin{align*}
	\E\left[ F(x^{(R)})-F(x^\ast) \right] = \E\left[ F(x^{(R)})\right]
	\geq \frac{1}{2}\cdot \frac{\lambda_0}{1010} \E\left[ \left(x^{(R)}\right)^2 \right]
	&=\Omega\left( \frac{\lambda_0\sigma^2}{\lambda_1^2 NR^2} \right).
\end{align*}

\subsection{Lower Bound for $\frac{1}{101\lambda N} \leq \eta < \frac{1}{\lambda}$}
In this regime, we consider the following functions
\begin{align*}
	&f_{\pi_n} (x) = \frac{1}{2}\lambda x^2 + \tau_{n}\sigma x, \\
	&f(x) = \E \left[f_{\pi_n}(x)\right] = \frac{1}{2}\lambda x^2,
\end{align*}
where $\tau_n$ is a random variable with equal probabilities of being either ``$+1$'' or ``$-1$''.

According to the result in \cite{cha2023tighter}'s Appendix~B.3, we have
\begin{align*}
	&x_N = (1-\lambda \eta)^N x_0 - \eta \sigma \sum_{n=1}^{N} (1-\lambda\eta)^{N-n} \tau_{n} \\
	&\E\left[ x_N^2 \right] = (1-\lambda \eta)^{2N} x_0^2 + \eta^2 \sigma^2 \E\left( \sum_{n=1}^{N} (1-\lambda \eta)^{N-n} \tau_{n} \right)^2.
\end{align*}
Similar to \cite{safran2020good}'s Lemma~1, we have
\begin{align*}
	\E\left( \sum_{n=1}^{N} (1-\lambda \eta)^{N-n} \tau_{n} \right)^2 &= \sum_{n=1}^{N} (1-\lambda\eta)^{2(N-n)} \E\left[\tau_{n}^2\right] + \sum_{n=1}^{N} \sum_{i\neq n}^{N} (1-\lambda\eta)^{2N-n-i}\E\left[\tau_{n}\tau_{i}\right]\\
	&=\sum_{n=1}^{N} (1-\lambda\eta)^{2(N-n)} .\tag{$\because \E[\tau_n \tau_i] = \E[\tau_n]\E[\tau_i]=0$}
\end{align*}
Thus, we get
\begin{align*}
	&\eta^2\sigma^2\sum_{n=1}^{N} (1-\lambda\eta)^{2(N-n)}= \frac{1-(1-\lambda\eta)^{2N}}{1-(1-\lambda\eta)^2}\eta^2\sigma^2
	\geq \frac{1-\exp(-2\lambda\eta N)}{\lambda\eta (2-\lambda\eta)}\eta^2\sigma^2 \geq 0.009\frac{\eta\sigma^2}{\lambda},
\end{align*}
where we use $\frac{1}{101\lambda N} \leq \eta \leq \frac{1}{\lambda}$. Then,
\begin{align*}
	&\E\left[\left(x^{(R)}\right)^2\right] \geq (1-\lambda\eta)^{2NR} \left(x^{(0)}\right)^2 + \sum_{r=0}^{R-1}(1-\lambda\eta)^{2Nr}0.009\frac{\eta\sigma^2}{\lambda}  \geq 0.009\frac{\sigma^2}{\lambda^2N},\\
	&\E\left[F(x^{(R)})-F^\ast\right] = \E\left[F(x^{(R)})\right] = \frac{\lambda}{2} \E\left[\left(x^{(R)}\right)^2\right] = \Omega\left(\frac{\sigma^2}{\lambda N}\right).
\end{align*}

Now we complete the proofs for all regimes for stochasticity terms in Theorem~\ref{thm:lower bound}. The setups and final results are summarized in Table~\ref{tab:lower bounds:stochasticity}. 
\end{proof}

\subsection{Helpful Lemmas for Stochasticity Terms}

\begin{lemma}
	\label{lem:lower bound-stochasticity:partial sum}
	Let $\tau_1, \tau_2,\dots,\tau_n$ be independent random variables, each with equal probabilities of being either ``$+1$'' or ``$-1$''. Let $\gE_n \coloneqq \sum_{i=1}^n \tau_i$ (with $\gE_0 = 0$). Then for any $n\geq 0$,
	\begin{align*}
		\frac{\sqrt{n}}{5}\leq \E\abs{\gE_n} \leq \sqrt{n}.
	\end{align*}
\end{lemma}
\begin{proof}
	For the upper bound, similar to \cite{rajput2020closing}'s Lemma 12 and \cite{cha2023tighter}'s Lemma B.5, we have
	\begin{align*}
		\E\abs{\gE_{n}}
		&= \E\left[\abs{\sum_{i=1}^{n}\tau_i}\right]
		&\leq \sqrt{\E\left[\left(\sum_{i=1}^{n}\tau_i\right)^2\right]}
		&\leq \sqrt{\sum_{i=1}^{n}\E[\left(\tau_i\right)^2]+2\sum_{i<j\leq m-1}\E[\tau_i\tau_j]}
		&= \sqrt{n},
	\end{align*}
	where $\E[\tau_i\tau_j] = 0$ for $i\neq j$, due to independence.
	
	For the lower bound, similar to \cite{rajput2020closing}'s Lemma 12, we have
	\begin{align*}
		\E\abs{\gE_{n}} = \E\abs{\gE_{n-1}} + \Pr(\gE_{n-1} \tau_n = 0) + \Pr(\gE_{n-1} \tau_n > 0) - \Pr(\gE_{n-1} \tau_n < 0).
	\end{align*}
	It can be seen that the last two terms can be canceled out,
	\begin{align*}
		\Pr(\gE_{n-1} \tau_n > 0) &= \Pr(\gE_{n-1}> 0, \tau_n > 0) + \Pr(\gE_{n-1}< 0, \tau_n < 0) \\
		&=\Pr(\gE_{n-1}> 0) \cdot \Pr(\tau_n > 0) + \Pr(\gE_{n-1}< 0) \Pr(\tau_n < 0)\\
		&= \frac{1}{2}\Pr(\gE_{n-1}> 0) + \frac{1}{2}\Pr(\gE_{n-1}< 0).
	\end{align*}
	Then, since $\Pr(\gE_{n-1} \tau_n = 0) = \Pr(\gE_{n-1} = 0) = \mathbbm{1}_{n-1 \ \text{is even}} \frac{{n-1 \choose n-1/2}}{2^{n-1}}$, we have
	\begin{align*}
		\E\abs{\gE_{n}} = \E\abs{\gE_{n-1}} + \mathbbm{1}_{n-1 \ \text{is even}} \frac{{n-1 \choose n-1/2}}{2^{n-1}}.
	\end{align*}
	According to \cite{cha2023tighter}'s Lemma B.8, ${n \choose n/2}$ can be estimated as $2^{n} \cdot \frac{\sqrt{2n + \alpha_n}}{\sqrt{\pi} (n + \alpha_{n/2})}$ with $0.333 \leq \alpha_{n/2}, \alpha_n \leq 0.354$ \citep{mortici2011gospers}, so we can get ${n \choose n/2}/2^{n} \geq \frac{1}{2\sqrt{n}}$.
	
	When $n\geq 2$ is an even integer,
	\begin{align*}
		\E\abs{\gE_{n}} &= \E\abs{\gE_{n-1}} + \mathbbm{1}_{n-1 \ \text{is even}} \frac{{n-1 \choose n-1/2}}{2^{n-1}}\\
		&= \E\abs{\gE_{n-2}} + \mathbbm{1}_{n-2\ \text{is even}} \frac{{n-1 \choose n-1/2}}{2^{n-1}}\\
		&\ \vdots\\
		&=\E\abs{\gE_2} + \sum_{p = 1}^{n/2-1}\frac{1}{2\sqrt{n-2p}} \tag{$\because \E\abs{\gE_2}=1$}\\
		&\geq \frac{n}{2} \cdot \frac{1}{2\sqrt{n}} \geq \frac{\sqrt{n}}{5}.
	\end{align*}
	When $n\geq 1$ is an odd integer,
	\begin{align*}
		\E\abs{\gE_{n+1}} = \E\abs{\gE_{n}}+\mathbbm{1}_{n \ \text{is even}} \frac{{n \choose n/2}}{2^{n}} \implies \E\abs{\gE_{n}} = \E\abs{\gE_{n+1}} \geq \frac{\sqrt{n}}{5}.
	\end{align*}
	Now we complete the proof of the bounds of $\E\abs{\gE_{n}}$ for any $n\geq 0$.
\end{proof}

\begin{lemma}
	\label{lem:lower bound-stochasticity:partial sum probability}
	Let $\tau_1, \tau_2,\dots,\tau_n$ be independent random variables, each with equal probabilities of being either ``$+1$'' or ``$-1$''. Let $\gE_n \coloneqq \sum_{i=1}^n \tau_i$ (with $\gE_0 = 0$). Then, for any $n\geq 0$, the probability distribution of $\gE_n$ is symmetric with respect $0$, and for any $n\geq 1$,
	\begin{align*}
		\frac{1}{4}\leq \Pr(\gE_{n}>0) = \Pr(\gE_{n}<0) \leq \frac{1}{2}.
	\end{align*}
\end{lemma}
\begin{proof}
	When $n=0$, $\Pr(\gE_{n}>0) = \Pr(\gE_{n}<0)=0$. The distribution is symmetric trivially. When $n>1$, $\Pr(\gE_{n}=p) = \begin{cases} \frac{{n \choose (n+p)/2}}{2^{n}} =  \frac{{n \choose (n-p)/2}}{2^{n}}, & n+p \mod 2 = 0\\0,&n+p \mod 2 \neq 0\end{cases}$ where the integer $p$ satisfies $-n \leq p \leq n$. Thus, the distribution is symmetric with respect to $0$.
	
	When any integer $n>1$, we have $\Pr(\gE_{n}=0) = \mathbbm{1}_{n \ \text{is even}} \frac{{n \choose n/2}}{2^{n}}$. Letting $g(n) = \frac{{n \choose n/2}}{2^{n}}$, it can be validated that $\frac{g(n+2)}{g(n)} = \frac{n+1}{n+2}<1$, so $\Pr(\gE_{n}=0) = g(n) \leq g(n-2) \leq \cdots \leq g(2) = \frac{1}{2}$ when $n$ is even. Therefore, $\frac{1}{2} \geq \Pr(\gE_{n}>0) = \Pr(\gE_{n}<0) = \frac{1}{2}(1- \Pr(\gE_{n}=0)) \geq \frac{1}{4}$.
\end{proof}

\begin{lemma}
	\label{lem:lower bound-stochasticity:diff curr-init}
	Suppose that $x_{0}\geq 0$, $\frac{\lambda_0}{\lambda} \geq 1010$ and $\eta \leq \frac{1}{101\lambda N}$. Then for $0\leq n \leq N$,
	\begin{align*}
		\E\abs{x_{n}- x_{0}} \leq \frac{1}{100} x_{0} + \frac{101}{100}\eta\sigma\sqrt{n}
	\end{align*}
\end{lemma}
\begin{proof}
	With result of Lemma~\ref{lem:lower bound-stochasticity:partial sum}, the proof is identical to \cite{cha2023tighter}'s Lemma B.7, except the numerical factors, to be consistent with the constants used in Appendix~\ref{app:lower bound-heterogeneity}.
\end{proof}

\section{Proof of Heterogeneity Terms in Theorem~\ref{thm:lower bound}}\label{app:lower bound-heterogeneity}
\begin{proof}
	We assume $F_m= f_m(\cdot;\xi_m^1)= \cdots=f_m(\cdot;\xi_m^{\abs{\gD_m}})$ for each $m$, and then extend the works of SGD-RR to SFL, from performing one update step to performing multiple update steps on each local objective function. The results are summarized in Table~\ref{tab:lower bounds:heterogeneity}. Notably, $\lambda_0, \lambda_1$ and $\lambda$ in different regimes can be different.
	
	\begin{table}[h]
	\centering
	\resizebox{\linewidth}{!}{
		\begin{threeparttable}[b]
			\renewcommand{\arraystretch}{1.2}
			\begin{tabular}{lllll}
				\toprule
				Bound &Regime &Objective functions &Initialization &$\nabla^2 F_m\in $\\
				\midrule
				$\Omega\left(\frac{\zeta^2}{\lambda}\right)$ &$\eta\leq \frac{1}{102010\lambda MKR}$ &$F_m(x) = \lambda x^2$ &$x^{(0)}=\frac{\zeta}{\lambda}$ &$[2\lambda, 2\lambda]$\\\cmidrule{1-5}
				
				$\Omega\left(\frac{\lambda_2\zeta^2}{\lambda_1^2 MR^2}\right)$ &\makecell[l]{$\frac{1}{102010\lambda_1 MKR}\leq\eta$\\[2ex] and $\eta \leq\frac{1}{101\lambda_0 MK}$} &\makecell[l]{$F_m (x) = $\\$\begin{cases}
						(\lambda_0\mathbbm{1}_{x<0} + \lambda \mathbbm{1}_{x\geq 0})\frac{x^2}{2} + \zeta x, &\text{if} \ m\leq \frac{M}{2}\\
						(\lambda_0\mathbbm{1}_{x<0} + \lambda \mathbbm{1}_{x\geq 0})\frac{x^2}{2} - \zeta x, &\text{otherwise}
					\end{cases}$} &\makecell[l]{$x^{(0)}$ \\$=\frac{1}{81608000}\frac{\zeta}{\lambda_1 M^{\frac{1}{2}}R}$} &$\left[\frac{\lambda_0}{1010}, \lambda_0\right]$\\\cmidrule{1-5}
				
				$\Omega\left(\frac{\zeta^2}{\lambda M}\right)$ &$\frac{1}{101\lambda MK}\leq\eta\leq\frac{1}{\lambda K}$ &$F_m (x) = \begin{cases}
					\frac{\lambda}{2}x^2 + \zeta x, &\text{if} \ m\leq \frac{M}{2}\\
					\frac{\lambda}{2}x^2 - \zeta x, &\text{otherwise}
				\end{cases}$ &$x^{(0)}=0$ &$[\lambda, \lambda]$\\\cmidrule{1-5}
				
				$\Omega\left(\frac{\zeta^2}{\lambda}\right)$ &$\frac{1}{\lambda K}\leq \eta\leq\frac{1}{\lambda}$ &$F_m (x) = \begin{cases}
					\frac{\lambda}{2}x^2 + \zeta x, &\text{if} \ m\leq \frac{M}{2}\\
					\frac{\lambda}{2}x^2 - \zeta x, &\text{otherwise}
				\end{cases}$ &$x^{(0)}=0$ &$[\lambda, \lambda]$\\\cmidrule{1-5}
				
				$\Omega\left(\frac{\zeta^2}{\lambda}\right)$ &$\eta\geq\frac{1}{\lambda}$ &$F_m(x) = \lambda x^2$ &$x^{(0)}=\frac{\zeta}{\lambda}$ &$[2\lambda, 2\lambda]$\\
				\bottomrule
			\end{tabular}
	\end{threeparttable}}
	\caption{Lower bounds of SFL for heterogeneity terms. It requires that $\lambda_0 = 1010\lambda$ and $R\geq \frac{1}{1010} \frac{\lambda_0}{\lambda_1}$ in the regime $\frac{1}{102010\lambda_1 MKR}\leq\eta\leq\frac{1}{101\lambda_0 MK}$.}
	\label{tab:lower bounds:heterogeneity}
\end{table}

	\subsection{Lower Bounds for $0< \eta\leq \frac{1}{102010\lambda MKR}$ and $\eta \geq \frac{1}{\lambda}$}
	In this regime, we consider the following objective functions
	\begin{align*}
		&F_m (x) = \lambda x^2, \\
		&F(x) = \frac{1}{M} \sum_{m=1}^M F_m (x) = \lambda x^2.
	\end{align*}
	We can soon build the relationship between $x^{(R)}$ and $x^{(0)}$: $x^{(R)} = (1-2\lambda\eta)^{MKR} x^{(0)}$.
	
	\subsubsection{Lower Bound for $0<\eta \leq \frac{1}{102010\lambda MKR}$}
	Since $\eta\leq \frac{1}{102010\lambda MKR}$ and $\left(1-\frac{1}{51005}\cdot \frac{1}{x}\right)^x$ is monotonically increasing when $x \geq 1$, we have
	\begin{align*}
		&x^{(R)} = \left(1-2\lambda\eta\right)^{MKR} x^{(0)} \geq \left(1-\frac{1}{51005MKR}\right)^{MKR}x^{(0)} \geq \frac{51004}{51005} x^{(0)},  \tag{$\because M\geq 1$}\\
		&F(x^{(R)}) = \lambda \left(x^{(R)}\right)^2 \geq \frac{51004^2}{51005^2}\lambda\left(x^{(0)}\right)^2.
	\end{align*}
	If $x^{(0)} = \frac{\zeta}{\lambda}$, we can get
	\begin{align*}
		\E\left[F(x^{(R)})-F^\ast\right] = \E\left[F(x^{(R)})\right] \geq \frac{51004^2}{51005^2}\lambda\left(x^{(0)}\right)^2 = \frac{51004^2}{51005^2} \frac{\zeta^2}{\lambda} = \Omega\left(\frac{\zeta^2}{\lambda}\right).
	\end{align*}
	
	\subsubsection{Lower Bound for $\eta \geq \frac{1}{\lambda}$}
	Since $\eta \geq \frac{1}{\lambda}$ implies that $(1-2\lambda\eta)^2 \geq 1$, we have
	\begin{align*}
		F(x^{(R)}) = \lambda \left(x^{(R)}\right)^2 = \lambda(1-2\lambda\eta)^{2MKR} \left(x^{(0)}\right)^2 \geq \lambda \left(x^{(0)}\right)^2.
	\end{align*}
	If $x^{(0)} = \frac{\zeta}{\lambda}$, we can get
	\begin{align*}
		\E\left[F(x^{(R)})-F^\ast\right] = \E\left[F(x^{(R)})\right] \geq \lambda\left(x^{(0)}\right)^2 = \frac{\zeta^2}{\lambda} = \Omega\left(\frac{\zeta^2}{\lambda}\right).
	\end{align*}
	
	\subsection{Lower Bound for $\frac{1}{102010\lambda_1 MKR}\leq \eta \leq \frac{1}{101\lambda_0 MK}$}
	In this regime, we consider the following functions
	\begin{align*}
		&F_m (x) = \begin{cases}
			(\lambda_0\mathbbm{1}_{x<0} + \lambda \mathbbm{1}_{x\geq 0})\frac{x^2}{2} + \zeta x, &\text{if} \ m\leq \frac{M}{2}\\
			(\lambda_0\mathbbm{1}_{x<0} + \lambda \mathbbm{1}_{x\geq 0})\frac{x^2}{2} - \zeta x, &\text{otherwise}
		\end{cases},\\
		&F(x) = \frac{1}{M} \sum_{m=1}^M F_{m} (x) = (\lambda_0\mathbbm{1}_{x<0} + \lambda \mathbbm{1}_{x\geq 0})\frac{x^2}{2} \quad \text{($\lambda_0/\lambda\geq 1010$)}.
	\end{align*}
	Next, we focus on a single round $r$, and thus we drop the superscripts $r$ for a while. Unless otherwise stated, the expectation is conditioned on $x_{1,0}$ when considering one single round. The proofs in this regime have a similar structure as \cite{cha2023tighter}'s Theorem~3.1.
	
	In each training round, we sample a random permutation $\pi=(\pi_1, \pi_2, \ldots, \pi_M)$ ($\pi_m$ is its $m$-th element) from $\{1,2,\ldots, M\}$ as the clients' training order. Then, we can denote $F_{\pi_m}$ for $m \in \{1,2,\ldots,M\}$ as $F_{\pi_m}(x) = (\lambda_0\mathbbm{1}_{x<0} + \lambda \mathbbm{1}_{x\geq 0})\frac{x^2}{2} + \zeta \tau_m x$, where $\tau=(\tau_1, \tau_2, \ldots,\tau_M)$ ($\tau_m$ is its $m$-th element) is a random permutation of $\frac{M}{2}$ $+1$'s and $\frac{M}{2}$ $-1$'s. For example, assuming that $\pi=(4,2,3,1)$ (with $M=4$), we can get the corresponding coefficients $\tau=(-1, +1, -1, +1)$.
	
	Then, the relationship between $x_{m,k}$ and $x_{1,0}$ satisfies
	\begin{align*}
		x_{m,k}
		&= x_{1,0} - \eta \sum_{i=1}^{m}\sum_{j=0}^{K-1} \left((\lambda_0\mathbbm{1}_{x_{i,j}<0} + \lambda \mathbbm{1}_{x_{i,j}\geq 0}) x_{i,j}\right) - K\zeta\sum_{i=1}^{m-1}\tau_i\\
		&\quad\quad\quad-\eta \sum_{j=0}^{k-1}\left((\lambda_0\mathbbm{1}_{x_{i,j}<0} + \lambda \mathbbm{1}_{x_{i,j}\geq 0}) x_{m,j}  \right) - k\zeta\tau_{m}.
	\end{align*}
	For convenience, we write it as
	\begin{align}
		x_{m,k} = x_{1,0} -\eta \sum_{i=0}^{\gB_{m,k}-1} \left((\lambda_0\mathbbm{1}_{x_{b_1(i),b_2(i)} <0} + \lambda \mathbbm{1}_{x_{b_1(i),b_2(i)} \geq 0}) x_{b_1(i),b_2(i)}  \right) - K\eta\zeta\gA_{m,k},\label{eq:proof:lower bound-heterogeneity:curr-init}
	\end{align}
	where $b_1(i) \coloneqq \lfloor \frac{i}{K} \rfloor+1$, $b_2(i) \coloneqq i-K \lfloor \frac{i}{K} \rfloor$, $\gA_{m,k} \coloneqq \gE_{m-1}+ a_k \tau_m = \tau_1+\tau_2\cdots + a_k \tau_{m}$ ($\gE_m \coloneqq \tau_1+\tau_2\cdots + \tau_{m}$ and $a_k \coloneqq k/K$) and $\gB_{m,k} = \sum_{i=0}^{(m-1)K+k-1}1 = (m-1)K+k$.
	In particular, when $m=M+1$ and $k=0$, it follows that
	\begin{align}
		x_{M+1,0} 
		&= x_{1,0} -\eta \sum_{m=1}^{M} \sum_{k=0}^{K-1} \left((\lambda_0\mathbbm{1}_{x_{m,k}<0} + \lambda \mathbbm{1}_{x_{m,k}\geq 0}) x_{m,k}\right).\label{eq:proof:lower bound-heterogeneity:final-init}
	\end{align}
	Notably, the following notations $x_{m+1,0}$ and $x_{m,K}$ are equivalent. In this paper, we use $x_{m+1,0}$ instead of $x_{m,K}$.
	
	\subsubsection{Lower Bound of $\E\left[x^{(r+1)}\mid x^{(r)}\geq 0\right]$.}
	We first give a stricter upper bound of $\E\left[(\lambda_0\mathbbm{1}_{x_{m,k}<0} + \lambda \mathbbm{1}_{x_{m,k}\geq 0}) x_{m,k}\right]$ for $1\leq m \leq \frac{M}{2}+1$, and then give a general upper bound for $1\leq m \leq M$. These two upper bounds are then plugged into Eq.~\eqref{eq:proof:lower bound-heterogeneity:final-init}, yielding the targeted lower bound of $\E\left[x_{M+1,0} \mid x_{1,0}\geq 0 \right]$.
	
	Since $(\lambda_0\mathbbm{1}_{x<0} + \lambda \mathbbm{1}_{x\geq 0}) x \leq \lambda_0 x$ and $(\lambda_0\mathbbm{1}_{x<0} + \lambda \mathbbm{1}_{x\geq 0}) x \leq \lambda x$, we have
	\begin{align}
		&\E\left[(\lambda_0\mathbbm{1}_{x_{m,k}<0} + \lambda \mathbbm{1}_{x_{m,k}\geq 0}) x_{m,k}\right]\nonumber \\
		&= \E\left[(\lambda_0\mathbbm{1}_{x_{m,k}<0} + \lambda \mathbbm{1}_{x_{m,k}\geq 0}) x_{m,k}\mid \gA_{m,k}>0\right] \Pr(\gA_{m,k}>0) \nonumber\\
		&\quad+ \E\left[(\lambda_0\mathbbm{1}_{x_{m,k}<0} + \lambda \mathbbm{1}_{x_{m,k}\geq 0}) x_{m,k}\mid \gA_{m,k}\leq 0\right] \Pr(\gA_{m,k}\leq 0)\nonumber\\
		&\leq \lambda_0\E\left[x_{m,k}\mid \gA_{m,k}>0\right] \cdot \Pr(\gA_{m,k}>0) + \lambda\E\left[x_{m,k}\mid \gA_{m,k}\leq 0\right] \cdot \Pr(\gA_{m,k}\leq 0).\label{eq:proof:lower bound-heterogeneity:local gradient stricter bound-1}
	\end{align}
	Intuitively, this is a trick, using $\lambda_0$ for the former term and $\lambda$ for the latter term, so we can make the former term (which has a stricter bound) dominant by controlling the value of ${\lambda_0}/{\lambda}$. Next, we bound the terms on the right hand side in Ineq.~\eqref{eq:proof:lower bound-heterogeneity:local gradient stricter bound-1}.
	For the first term in Ineq.~\eqref{eq:proof:lower bound-heterogeneity:local gradient stricter bound-1}, according to Ineq.~\eqref{eq:proof:lower bound-heterogeneity:curr-init}, we have
	\begin{align*}
		&\E\left[x_{m,k}\mid \gA_{m,k}>0\right]\\
		&= \E\left[x_{1,0} -\eta \sum_{i=0}^{\gB_{m,k}-1} \left((\lambda_0\mathbbm{1}_{x_{b_1(i),b_2(i)} <0} + \lambda \mathbbm{1}_{x_{b_1(i),b_2(i)} \geq 0}) x_{b_1(i),b_2(i)}  \right) - K\eta\zeta\gA_{m,k}\mid \gA_{m,k}>0\right] \\
		&= x_{1,0}+\eta \sum_{i=0}^{\gB_{m,k}-1} \E\left[(\lambda_0\mathbbm{1}_{x_{b_1(i),b_2(i)} <0} + \lambda \mathbbm{1}_{x_{b_1(i),b_2(i)}\geq 0}) \abs{x_{b_1(i),b_2(i)}-x_{1,0}} \mid \gA_{m,k}>0\right] \\
		&\quad-\eta \sum_{i=0}^{\gB_{m,k}-1} \E\left[(\lambda_0\mathbbm{1}_{x_{b_1(i),b_2(i)} <0} + \lambda \mathbbm{1}_{x_{b_1(i),b_2(i)}\geq 0}) x_{1,0} \mid \gA_{m,k}>0\right] - K\eta\zeta\E\left[\gA_{m,k}\mid \gA_{m,k}>0\right]\\
		&\leq x_{1,0}+\lambda_0\eta \sum_{i=0}^{\gB_{m,k}-1} \E\left[\abs{x_{b_1(i),b_2(i)}-x_{1,0}} \mid \gA_{m,k}>0\right] - K\eta\zeta\E\left[\gA_{m,k}\mid \gA_{m,k}>0\right],
	\end{align*}
	where we use $\lambda\leq (\lambda_0\mathbbm{1}_{x<0} + \lambda \mathbbm{1}_{x\geq 0}) \leq \lambda_0$ and $x_{1,0} \geq 0$ in the last inequality. Then, we have
	\begin{align*}
		\Term{1}{\eqref{eq:proof:lower bound-heterogeneity:local gradient stricter bound-1}} 
		&\leq \lambda_0 x_{1,0}\Pr(\gA_{m,k}>0) 
		+\lambda_0^2\eta \sum_{i=0}^{\gB_{m,k}-1} \E\left[\abs{x_{b_1(i),b_2(i)}-x_{1,0}} \mid \gA_{m,k}>0\right] \Pr(\gA_{m,k}>0) \\
		&\quad - \lambda_0 K\eta\zeta\E\left[\gA_{m,k}\mid \gA_{m,k}>0\right]\Pr(\gA_{m,k}>0).
	\end{align*}
	Since $\E\left[\abs{x_{i,j}-x_{1,0}} \mid \gA_{m,k}>0\right]\Pr(\gA_{m,k}>0)$ and $\E\left[\abs{x_{i,j}-x_{1,0}} \mid \gA_{m,k}\leq0\right]\Pr(\gA_{m,k}\leq0)$ are non-negative, the term $\E\left[\abs{x_{i,j}-x_{1,0}} \mid \gA_{m,k}>0\right]\Pr(\gA_{m,k}>0)$ appearing in the second term in the preceding inequality can be bounded by $\E\left[\abs{x_{i,j}-x_{1,0}}\right]$ for any integers $i,j$. Since the probability distribution of $\gA_{m,k}$ is symmetric with respect to $0$ (Lemma~\ref{lem:lower bound:partial sum probability bounds}), we get
	\begin{align*}
		\E\left[\abs{\gA_{m,k}}\right]
		&=\E\left[\gA_{m,k}\mid \gA_{m,k}>0\right]\Pr(\gA_{m,k}>0) + \E\left[-\gA_{m,k}\mid \gA_{m,k}<0\right]\Pr(\gA_{m,k}<0)\\
		&=\E\left[\gA_{m,k}\mid \gA_{m,k}>0\right]\Pr(\gA_{m,k}>0) + \E\left[-\gA_{m,k}\mid -\gA_{m,k}>0\right]\Pr(-\gA_{m,k}>0)\\
		&=2\E\left[\gA_{m,k}\mid \gA_{m,k}>0\right]\Pr(\gA_{m,k}>0). \tag{by symmetry}
	\end{align*}
	After using $\Pr(\gA_{m,k}>0)\leq \frac{1}{2}$ by symmetry, $\E\left[\abs{x_{i,j}-x_{1,0}} \mid \gA_{m,k}>0\right]\Pr(\gA_{m,k}>0) \leq\E\left[\abs{x_{i,j}-x_{1,0}}\right]$ and $\E\left[\abs{\gA_{m,k}}\right] =2\E\left[\gA_{m,k}\mid \gA_{m,k}>0\right]\Pr(\gA_{m,k}>0)$ for the first, the second and the last terms on the right hand side, respectively, we have
	\begin{align}
		\Term{1}{\eqref{eq:proof:lower bound-heterogeneity:local gradient stricter bound-1}}
		&\leq \frac{1}{2}\lambda_0 x_{1,0}+\lambda_0^2\eta \sum_{i=0}^{\gB_{m,k}-1} \E\left[\abs{x_{b_1(i),b_2(i)}-x_{1,0}}\right]- \frac{1}{2}\lambda_0 K\eta\zeta\E\left[\abs{\gA_{m,k}}\right]. \label{eq:proof:lower bound-heterogeneity:local gradient stricter bound-2}
	\end{align}
	For the second term on the right hand side in Ineq.~\eqref{eq:proof:lower bound-heterogeneity:local gradient stricter bound-1}, we have
	\begin{align}
		\Term{2}{\eqref{eq:proof:lower bound-heterogeneity:local gradient stricter bound-1}}&\leq \lambda\E\left[\abs{x_{m,k}-x_{1,0}}\mid \gA_{m,k}\leq 0\right] \Pr(\gA_{m,k}\leq 0) + \lambda\E\left[x_{1,0}\mid \gA_{m,k}\leq 0\right] \Pr(\gA_{m,k}\leq 0)\nonumber\\
		&\leq \lambda \E\left[\abs{x_{m,k}-x_{1,0}}\right] + \frac{5}{6}\lambda x_{1,0},\label{eq:proof:lower bound-heterogeneity:local gradient stricter bound-3}
	\end{align}
	where we use $\Pr(\gA_{m,k}< 0) = \Pr(\gA_{m,k}> 0) \geq \frac{1}{6}$ (Lemma~\ref{lem:lower bound:partial sum probability bounds}) for the second term on the right hand side in the last inequality. Plugging Ineq.~\eqref{eq:proof:lower bound-heterogeneity:local gradient stricter bound-2} and Ineq.~\eqref{eq:proof:lower bound-heterogeneity:local gradient stricter bound-3} into Ineq.~\eqref{eq:proof:lower bound-heterogeneity:local gradient stricter bound-1}, we have
	\begin{align*}
		\E\left[(\lambda_0\mathbbm{1}_{x_{m,k}<0} + \lambda \mathbbm{1}_{x_{m,k}\geq 0}) x_{m,k}\right]&\leq \frac{1}{2}\lambda_0 x_{1,0} +\lambda_0^2\eta \sum_{i=0}^{\gB_{m,k}-1} \E\left[\abs{x_{b_1(i),b_2(i)}-x_{1,0}}\right]\\
		&\quad- \frac{1}{2}\lambda_0 K\eta\zeta\E\left[\abs{\gA_{m,k}}\right] + \lambda \E\left[\abs{x_{m,k}-x_{1,0}}\right] + \frac{5}{6}\lambda x_{1,0}.
	\end{align*}
	Using $\E\left[\abs{x_{i,j}-x_{1,0}}\right] \leq \frac{1}{100}x_{1,0} + \frac{101}{100}K\eta\zeta\sqrt{m-1+a_k^2}$ for $(i-1)K + j \leq (m-1)K + k$ (Lemma~\ref{lem:lower bound:diff current-initial parameter upper bound}), $\E\left[\abs{\gA_{m,k}}\right] \geq \frac{1}{20} \sqrt{m-1+a_k^2}$ (Lemma~\ref{lem:lower bound:partial sum absolute value bounds}) and ${\lambda_0}/{\lambda_0}\geq {1010}$ and $\lambda_0 \eta \gB_{m,k} \leq \lambda_0 MK\eta\leq \frac{1}{101}$, we can simplify it as
	\begin{align}
		\E\left[(\lambda_0\mathbbm{1}_{x_{m,k}<0} + \lambda \mathbbm{1}_{x_{m,k}\geq 0}) x_{m,k}\right]\leq \frac{501}{1000}\lambda_0 x_{1,0} + \frac{14}{1000}\lambda_0 K\eta\zeta\sqrt{m-1+a_k^2}.\label{eq:proof:lower bound-heterogeneity:local gradient stricter bound}
	\end{align}
	Ineq.~\eqref{eq:proof:lower bound-heterogeneity:local gradient stricter bound} holds for $1\leq m \leq \frac{M}{2}+1$ ($M\geq 4$) and $0\leq k \leq K-1$, due to the constraints of Lemmas~\ref{lem:lower bound:partial sum absolute value bounds}, \ref{lem:lower bound:partial sum probability bounds} and \ref{lem:lower bound:diff current-initial parameter upper bound}. Even though the constraint of Lemma~\ref{lem:lower bound:partial sum probability bounds} excludes the case $m=1,k=0$, we can verify $\E\left[(\lambda_0\mathbbm{1}_{x_{1,0}<0} + \lambda \mathbbm{1}_{x_{1,0}\geq 0}) x_{1,0}\right] = \lambda x_{1,0} \leq \frac{1}{1010}\lambda_0 x_{1,0}$ ($x_{1,0}\geq 0$).
	
	Next, we give a general upper bound of $\E\left[(\lambda_0\mathbbm{1}_{x_{m,k}<0} + \lambda \mathbbm{1}_{x_{m,k}\geq 0}) x_{m,k}\right]$ for $1\leq m \leq M$.
	\begin{align}
		\E\left[(\lambda_0\mathbbm{1}_{x_{m,k}<0} + \lambda \mathbbm{1}_{x_{m,k}\geq 0}) x_{m,k}\right] \tag{$\because (\lambda_0\mathbbm{1}_{x<0} + \lambda \mathbbm{1}_{x\geq 0}) x \leq \lambda x$ for all $x\in \R$}
		&\leq \E\left[\lambda x_{m,k}\right] \nonumber\\
		&\leq \lambda \E\left[\abs{x_{m,k}-x_{1,0}}\right] + \lambda x_{1,0} \nonumber\\
		&\leq \frac{1}{1000}\lambda_0 x_{1,0} + \frac{1}{1000}\lambda_0 K\eta\zeta\sqrt{m-1+a_k^2},\label{eq:proof:lower bound-heterogeneity:local gradient general bound}
	\end{align}
	where we use Lemma~\ref{lem:lower bound:diff current-initial parameter upper bound} and ${\lambda_0}/{\lambda}\geq {1010}$ in the last inequality. Notably, this inequality holds for $1\leq m \leq M$ and $0\leq k \leq K-1$ ($\because$ Lemma~\ref{lem:lower bound:diff current-initial parameter upper bound}).
	
	Recalling Eq.~\eqref{eq:proof:lower bound-heterogeneity:final-init}, we first separate the sum of $\E\left[(\lambda_0\mathbbm{1}_{x_{m,k}<0} + \lambda \mathbbm{1}_{x_{m,k}\geq 0}) x_{m,k}\right]$ into two parts, and then use the tighter bound, Ineq.~\eqref{eq:proof:lower bound-heterogeneity:local gradient stricter bound}, for $1\leq m\leq \frac{M}{2}+1$ (the first part) and the general bound, Ineq.~\eqref{eq:proof:lower bound-heterogeneity:local gradient general bound}, for $\frac{M}{2}+2\leq m\leq M$ (the second part).
	\begin{align*}
		\E\left[x_{M+1,0}-x_{1,0}\right]
		&= -\eta \sum_{m=1}^{\frac{M}{2}+1} \sum_{k=0}^{K-1} \E\left[(\lambda_0\mathbbm{1}_{x_{m,k}<0} + \lambda \mathbbm{1}_{x_{m,k}\geq 0}) x_{m,k}\right] \\
		&\quad- \eta \sum_{m=\frac{M}{2}+2}^{M} \sum_{k=0}^{K-1} \E\left[(\lambda_0\mathbbm{1}_{x_{m,k}<0} + \lambda \mathbbm{1}_{x_{m,k}\geq 0}) x_{m,k}\right]\\
		&\geq -\eta \sum_{m=1}^{\frac{M}{2}+1} \sum_{k=0}^{K-1} \left(\frac{501}{1000}\lambda_0 x_{1,0} - \frac{14}{1000} \lambda_0 K\eta\zeta\sqrt{m-1+a_k^2}\right) \\
		&\quad- \eta \sum_{m=\frac{M}{2}+2}^{M} \sum_{k=0}^{K-1} \left( \frac{1}{1000}\lambda_0 x_{1,0} + \frac{1}{1000}\lambda_0 K\eta\zeta \sqrt{m-1+a_k^2} \right).
	\end{align*}
	Then, after simplifying the preceding inequality, we get
	\begin{align*}
		&\E\left[x_{M+1,0}\right] \geq \left(1-\frac{2}{3}\lambda_0 MK\eta\right) x_{1,0} + \frac{1}{600}\lambda_0 M^{\frac{3}{2}}K^2\eta^2\zeta.
	\end{align*}
	Taking unconditional expectations and putting back the superscripts, we can get
	\begin{align}
		&\E\left[x^{(r+1)}\mid x^{(r)}\geq 0\right] \geq \left(1-\frac{2}{3}\lambda_0 MK\eta\right)\E\left[x^{(r)}\mid x^{(r)}\geq 0\right] + \frac{1}{600}\lambda_0 M^{\frac{3}{2}}K^2\eta^2\zeta.\label{eq:proof:lower bound-heterogeneity:bound-1}
	\end{align}
	where we note that the following notations are interchangeable: $x_{1,0}^{(r)}$ and $x^{(r)}$, $x_{M+1,0}^{(r)}$ and $x^{(r+1)}$. Notably, Ineq.~\eqref{eq:proof:lower bound-heterogeneity:bound-1} holds for $M\geq 4$, because of the constraints of Ineq.~\eqref{eq:proof:lower bound-heterogeneity:local gradient stricter bound}.
	
	\subsubsection{Lower Bound of $\E\left[x^{(r+1)}\mid x^{(r)}< 0\right]$.}\label{subsubsection:intermediate func-1}
	We introduce a new function $H(x)$ (see Section~\ref{subsection:lower bound:regime3} for details) as follows:
	\begin{align*}
		&H_m (x) = \begin{cases}
			\frac{\lambda_0}{2}x^2 + \zeta x, &\text{if} \ m\leq \frac{M}{2}\\
			\frac{\lambda_0}{2}x^2 - \zeta x, &\text{otherwise}
		\end{cases}\\
		&H(x) = \frac{1}{M} \sum_{m=1}^M H_m (x) = \frac{\lambda_0}{2}x^2.
	\end{align*}
	Let Algorithm~\ref{algorithm1} run on the two functions $F(x)$ and $H(x)$, where both algorithms start from the same initial point and share the same random variables $\{\tau_{m,k}^{(r)}\}_{k,m,r}$ for all training rounds. Then, according to Part 1 of \cite{cha2023tighter}'s Lemma B.4, the model parameters generated on $F(x)$ and $H(x)$ satisfy $(x_{m,k})_{F} \geq (x_{m,k})_{H}$. Here, we let both cases share the same initial point $x_{1,0}$ and the same random variables $\pi$ (accordingly, the same $\tau$).

	The relationship between $(x_{M+1,0})_H$ and $x_{1,0}$ satisfies:
	\begin{align*}
		\E\left[(x_{M+1,0})_H\right] &= \E\left[(1-\lambda_0\eta)^{MK}x_{1,0} - \eta\zeta \sum_{m=0}^{M-1} (1-\lambda_0\eta)^{mK}\tau_{M-m} \sum_{k=0}^{K-1}(1-\lambda_0\eta)^k\right]\\
		&= (1-\lambda_0\eta)^{MK}x_{1,0} - \eta\zeta \sum_{m=0}^{M-1} (1-\lambda_0\eta)^{mK}\E\left[\tau_{M-m}\right] \sum_{k=0}^{K-1}(1-\lambda_0\eta)^k\\
		&= (1-\lambda_0\eta)^{MK}x_{1,0}. \tag{$\because \E[\tau_{M-m}] = 0$}
	\end{align*}
	Then since $(1-z)^K\leq 1-Kz+K^2z^2$, $\forall x\in [0,1]$, we can get
	\begin{align*}
		(1-\lambda_0\eta)^{MK} &\leq 1-\lambda_0 MK\eta + \lambda_0^2 M^2K^2\eta^2 \leq 1-\frac{2}{3} \lambda_0 MK\eta.\tag{$\because \lambda_0 MK\eta \leq \frac{1}{101}$}
	\end{align*}
	Then using $(x_{m,k})_{F} \geq (x_{m,k})_{H}$ and $x_{1,0}<0$, we have
	\begin{align*}
		\E\left[(x_{M+1,0})_F\right] \geq \E\left[(x_{M+1,0})_H\right] = (1-\lambda_0\eta)^{MK}x_{1,0} \geq \left(1-\frac{2}{3} \lambda_0 MK\eta\right)x_{1,0}.
	\end{align*}
	Taking unconditional expectations and putting back the superscripts, we can get
	\begin{align}
		&\E\left[x^{(r+1)}\mid x^{(r)}< 0\right] \geq \left(1-\frac{2}{3}\lambda_0 MK\eta\right)\E\left[x^{(r)}\mid x^{(r)}< 0\right].\label{eq:proof:lower bound-heterogeneity:bound-2}
	\end{align}

	\subsubsection{Relationship Between $\Pr(x^{(r)}) \geq 0$ and $\Pr(x^{(r)}) < 0$.}\label{subsubsection:intermediate func-2}
	We still use the function $H(x)$ for comparison. This time, we let both cases share the same initial point $x^{(0)}$ and the same $\tau^{(0)}, \tau^{(1)},\ldots,\tau^{(r-1)}$ for the first $r$ rounds. For any round $r$, we can build the relationship of $(x^{(r)})_H$ with $x^{0}$:
	\begin{align*}
		(x^{(r)})_H &= (1-\lambda_0\eta)^{rMK}x^{(0)} - \eta\zeta\sum_{k=0}^{K-1}(1-\lambda_0\eta)^k\sum_{m=0}^{M-1} (1-\lambda_0\eta)^{mK} \sum_{s=0}^{r-1}(1-\lambda_0\eta)^{sMK}\tau_{M-m}^{(r-1-s)}.
	\end{align*}
	
	For all possible permutations $\tau^{(0)}, \tau^{(1)},\ldots,\tau^{(r-1)}$, we can find the corresponding permutations $(\tau^{(0)})', (\tau^{(1)})', \ldots, (\tau^{(r-1)})'$ satisfy $\tau_m^{(s)} = -(\tau_m^{(s)})'$ for all $m \in \{1,2,\ldots,M\}$ and $s \in \{0,1,\ldots,r-1\}$. Denoting the parameters obtained with $\tau^{(0)}, \tau^{(1)},\ldots,\tau^{(r-1)}$ and $(\tau^{(0)})', (\tau^{(1)})', \ldots, (\tau^{(r-1)})'$ as $(x^{(r)})_H$ and $(x^{(r)})_H'$, respectively, we have
	\begin{align*}
		\frac{1}{2}\left( (x^{(r)})_H + (x^{(r)})_H' \right) &= (1-\lambda_0\eta)^{rMK}x^{(0)}.
	\end{align*}
	where the second terms of $(x^{(r)})_H$ and $(x^{(r)})_H'$ are canceled out with $\frac{\tau_{M-m}^{(r-1-s)} + (\tau_{M-m}^{(r-1-s)})'}{2} = 0$.
	This means that for any possible parameter $(x^{(r)})_H$, there exits one corresponding parameter $(x^{(r)})_H'$ to make their average be $(1-\lambda_0\eta)^{rMK} x^{(0)}$, and further implies
	\begin{align*}
		\Pr\left((x^{(r)})_H\geq (1-\lambda_0\eta)^{rMK} x^{(0)}\right) \geq \frac{1}{2}.
	\end{align*}
	Then, for the same initial point $x^{(0)} \geq 0$, we have
	\begin{align*}
		\Pr\left((x^{(r)})_F\geq 0\right) \geq \Pr\left((x^{(r)})_H\geq 0\right) \geq \Pr\left((x^{(r)})_H\geq (1-\lambda_0\eta)^{rMK} x^{(0)}\right) \geq \frac{1}{2}.
	\end{align*}
	Intuitively, the total possible number of events (the permutations) are identical for both $F$ and $H$. Since $(x^{(r)})_F\geq (x^{(r)})_H$ for the same permutations, the permutations that make $(x^{(r)})_H\geq 0$ always make $(x^{(r)})_F\geq 0$, causing $\Pr\left((x^{(r)})_F\geq 0\right) \geq \Pr\left((x^{(r)})_H\geq 0\right)$. The reasoning for the second inequality is similar. The permutations that make $(x^{(r)})_H\geq (1-\lambda_0\eta)^{rMK}x^{(0)}$ always make $(x^{(r)})_H\geq 0$ for $x^{(0)} \geq 0$.

	\subsubsection{Lower Bound for $\frac{1}{102010\lambda_1 MKR}\leq \eta \leq \frac{1}{101\lambda_0 MK}$}
	Using Ineq.~\eqref{eq:proof:lower bound-heterogeneity:bound-1}, Ineq.~\eqref{eq:proof:lower bound-heterogeneity:bound-2} and $\Pr\left(x^{(r)}\geq 0\right)\geq \frac{1}{2}$ (when $x^{(0)} \geq 0$), we have
	\begin{align*}
		\E\left[x^{(r+1)}\right] &= \E\left[x^{(r+1)}\mid x^{(r)}\geq 0\right]\Pr\left(x^{(r)}\geq 0\right) + \E\left[x^{(r+1)}\mid x^{(r)}< 0\right]\Pr\left(x^{(r)}< 0\right)\\
		&\geq \left( \left(1-\frac{2}{3}\lambda_0 MK\eta\right)x^{(r)} + \frac{1}{600}\lambda_0 M^{\frac{3}{2}}K^2\eta^2\zeta \right)\Pr\left(x^{(r)}\geq 0\right) \\
		&\quad+ \left( \left(1-\frac{2}{3} \lambda_0 MK\eta\right)x^{(r)} \right) \Pr\left(x^{(r)}< 0\right)\\
		&\geq \left(1-\frac{2}{3}\lambda_0 MK\eta\right)x^{(r)} + \frac{1}{1200}\lambda_0 M^{\frac{3}{2}}K^2\eta^2\zeta  .\tag{$\because \Pr\left(x^{(r)}\geq 0\right)\geq \frac{1}{2}$}
	\end{align*}
	If $x^{(r)} \geq \frac{1}{81608000}\cdot \frac{\zeta}{\lambda_1 M^{\frac{1}{2}}R}$, then using $\eta \geq \frac{1}{102010\lambda_1 MKR}$, we have
	\begin{align*}
		x^{(r+1)} &\geq \left(1-\frac{2}{3}\lambda_0 MK\eta\right)x^{(r)} + \frac{1}{1200}\lambda_0 M^{\frac{3}{2}}K^2\eta^2\zeta \\
		&\geq \left(1-\frac{2}{3}\lambda_0 MK\eta\right) \cdot \frac{1}{81608000}\cdot \frac{\zeta}{\lambda_1 M^{\frac{1}{2}}R} + \frac{1}{1200}\lambda_0 M^{\frac{3}{2}}K^2\eta \zeta \cdot \frac{1}{102010\lambda_1 MKR}\\
		&\geq \frac{1}{81608000}\cdot \frac{\zeta}{\lambda_1 M^{\frac{1}{2}}R}.
	\end{align*}
	Therefore, if we set $x^{(0)}\geq \frac{1}{81608000}\cdot \frac{\zeta}{\lambda_1 M^{\frac{1}{2}}R}$, then the final parameters will also maintain $x^{(R)}\geq \frac{1}{81608000}\cdot \frac{\zeta}{\lambda_1 M^{\frac{1}{2}}R}$. Then, noting that $\frac{\lambda_0}{\lambda} \geq 1010$, we can choose $\frac{\lambda_0}{\lambda} = 1010$. Then,
	\begin{align*}
		\E\left[ F(x^{R})-F(x^\ast) \right] = \E\left[ F(x^{R})\right]
		&\geq \frac{1}{2}\cdot \lambda \E\left[ \left(x^{R}\right)^2 \right]\\
		&\geq \frac{1}{2}\cdot \frac{\lambda_0}{1010}\cdot \left(\frac{1}{81608000}\cdot \frac{\zeta}{\lambda_1 M^{\frac{1}{2}}R}\right)^2\\
		&=\Omega\left( \frac{\lambda_0\zeta^2}{\lambda_1^2 MR^2} \right).
	\end{align*}
	Notably, this inequality holds for $M\geq 4$ (see Ineq.~\ref{eq:proof:lower bound-heterogeneity:bound-1}).
	
	\subsection{Lower Bound for $\frac{1}{101\lambda MK}\leq \eta \leq \frac{1}{\lambda K}$ and $\frac{1}{\lambda K} \leq \eta \leq \frac{1}{\lambda}$}
	\label{subsection:lower bound:regime3}
	In these two regimes, we consider the following functions
	\begin{align*}
		&F_m (x) = \begin{cases}
			\frac{\lambda}{2}x^2 + \zeta x, &\text{if} \ m\leq \frac{M}{2}\\
			\frac{\lambda}{2}x^2 - \zeta x, &\text{otherwise}
		\end{cases},\\
		&F(x) = \frac{1}{M} \sum_{m=1}^M F_{m} (x) = \frac{\lambda}{2}x^2.
	\end{align*}
	
	In each training round, we sample a random permutation $\pi=(\pi_1, \pi_2, \ldots, \pi_M)$ ($\pi_m$ is its $m$-th element) from $\{1,2,\ldots, M\}$ as the clients' training order. Thus, we can denote $F_{\pi_m}$ for $m \in \{1,2,\ldots,M\}$ as $F_{\pi_m}(x) = \frac{\lambda}{2}x^2 + \zeta \tau_m x$,
	where $\tau=(\tau_1, \tau_2, \ldots,\tau_M)$ is a random permutation of $\frac{M}{2}$ $+1$'s and $\frac{M}{2}$ $-1$'s.
	
	For a single training round, we can get
	\begin{align*}
		x^{(r)} = (1-\lambda\eta)^{MK} x^{(r-1)} - \eta\zeta\sum_{m=0}^{M-1} (1-\lambda\eta)^{mK}\tau_{M-m} \sum_{k=0}^{K-1}(1-\lambda\eta)^k.
	\end{align*}
		
	Taking expectation conditional on $x^{(r-1)}$, we can get
	\begin{align*}
		\E\left(x^{(r)}\right)^2&= \E\left[\left((1-\lambda\eta)^{MK} x^{(r-1)} - \eta\zeta\sum_{m=0}^{M-1} (1-\lambda\eta)^{mK}\tau_{M-m} \sum_{k=0}^{K-1}(1-\lambda\eta)^k\right)^2\right]\\
		&=(1-\lambda\eta)^{2MK}\left(x^{(r-1)}\right)^2 + \eta^2\zeta^2\left(\sum_{k=0}^{K-1}(1-\lambda\eta)^k\right)^2\E\left[\left(\sum_{m=0}^{M-1}(1-\lambda\eta)^{mK}\tau_{m}\right)^2\right],
	\end{align*}
	where we note that the cross terms on the right hand side equal zero since $\E[\tau_{M-m}] = 0$. Following \cite{safran2020good}'s Lemma 1, we first focus on the term
	\begin{align*}
		\E\left[\left(\sum_{m=0}^{M-1}(1-\lambda\eta)^{mK}\tau_{m}\right)^2\right]
		&=\sum_{m=0}^{M-1} (1-\lambda\eta)^{2mK}\E\left[\tau_m^2\right] + \sum_{i=0}^{M-1}\sum_{j\neq i}^{M-1}(1-\lambda\eta)^{(i+j)K}\E\left[\langle \tau_{i},\tau_{j}\rangle \right].
	\end{align*}
	Since $\tau_m^2=1$ and $\E\left[\langle \tau_{i},\tau_{j}\rangle \right]=-\frac{1}{M-1}$ (see \cite{safran2020good}'s Lemma 2), we get
	\begin{align*}
		\E\left[\left(\sum_{m=0}^{M-1}(1-\lambda\eta)^{mK}\tau_{m}\right)^2\right]&=\sum_{m=0}^{M-1} (1-\lambda\eta)^{2mK} -\frac{1}{M-1} \sum_{i=0}^{M-1}\sum_{j\neq i}^{M-1}(1-\lambda\eta)^{(i+j)K}\\
		=&\left(1+\frac{1}{M-1}\right)\sum_{m=0}^{M-1} (1-\lambda\eta)^{2mK} - \frac{1}{M-1}\left(\sum_{m=0}^{M-1}(1-\lambda\eta)^{mK} \right)^2.
	\end{align*}
	Returning back to $\E\left[\left(x^{(r)}\right)^2\right]$ and using $\sum_{m=0}^{M-1}(1-\lambda\eta)^{mK} = \frac{1-(1-\lambda\eta)^{MK}}{1-(1-\lambda\eta)^{K}}$ and $\sum_{m=0}^{M-1} (1-\lambda\eta)^{2mk} = \frac{1-(1-\lambda\eta)^{2MK}}{1-(1-\lambda\eta)^{2K}}$, we can get
	\begin{align}
		&\E\left(x^{(r)}\right)^2 \nonumber\\
		&=(1-d)^{2MK}\left(x^{(r-1)}\right)^2 + \eta^2\zeta^2 \frac{M}{M-1} \cdot \frac{1}{d^2}\cdot \frac{1-(1-d)^{K}}{1+(1-d)^K}\cdot\left(1-(1-d)^{MK}\right)T(d),\label{eq:proof:lower bound-heterogeneity:parameter square recursion}
	\end{align}
	where we define $d = \lambda \eta$ and $T(d) = 1+(1-d)^{MK} - \frac{1}{M}\cdot \frac{1+(1-d)^{K}}{1-(1-d)^K }\cdot \left(1-(1-d)^{MK}\right)$. For convenience, we also define a intermediate variable $t= 1-(1-d)^K$. Then
	\begin{align}
		T(t) = \left(1-\frac{1}{M}\cdot \frac{2-t}{t}\right) + \left(1+\frac{1}{M}\cdot \frac{2-t}{t}\right)(1-t)^M.\label{eq:proof:lower bound-heterogeneity:composite function}
	\end{align}
	According to Lemma~\ref{lem:composite function}, $T(t)$ is monotonically increasing on the interval $0<t<1$, and $T(d)$ is monotonically increasing on the interval $0<d<1$.
	
	\subsubsection{Lower Bound for $\frac{1}{101\lambda MK} \leq \eta \leq \frac{1}{\lambda K}$} 
	In this regime, $T(d)$ is lower bounded by $T(\frac{1}{101MK})$ for $d\in [\frac{1}{101MK}, \frac{1}{K}]$. Here, we first lower bound $t(d)$, and then lower bound $T(d)$. According to the fact $(1-x)^n \leq 1-nx+\frac{1}{2}n^2x^2$ when $x\in (0,1)$ (it can be proved with Taylor expansion of $(1-x)^n$ at $x=0$.), we get
	\begin{align*}
		t \geq 1-(1-d)^K \geq 1-(1-dK + \frac{1}{2}d^2K^2) \geq dK - \frac{1}{2}d^2K^2 \geq \frac{1}{2}dK \geq \frac{1}{202M}
	\end{align*}
	Then following the proofs of \cite{safran2020good}'s Lemma~1, we deal with the lower bound of $T(t)$ on $t \in [\frac{1}{202M}, 1]$,
	\begin{align*}
		T &=405 \left( 1-\frac{1}{202M} \right)^M - 403 + \left(1- \left(1-\frac{1}{202M}\right)^M \right) \cdot \frac{1}{M}
	\end{align*}
	Since the first two terms are increasing as $M$ increases and the third term is positive, $T$ is lower bounded by one numerical constant, as long as there exists $M_0$ such that $405 \left( 1-\frac{1}{202M_0} \right)^{M_0} - 403> 0$ and $T(t)> 0$ for all $2\leq t\leq M_0$ (This can be done by a simple code). We get that $405 \left( 1-\frac{1}{202\cdot 1212} \right)^{1212} - 403> 1.3 \cdot 10^{-11}$ when $M_0 = 1212$. Hence, $T$ is lower bounded by some numerical constant $c$. Returning to $\E\left[\left(x^{(r)}\right)^2\right]$, we get
	\begin{align*}
		\E\left[\left(x^{(r)}\right)^2\right] &\geq (1-d)^{2MK}\left(x^{(r-1)}\right)^2+ c\cdot \eta^2\zeta^2 \frac{M}{M-1} \cdot \frac{1}{d^2}\cdot \frac{1-(1-d)^{K}}{1+(1-d)^K}\cdot\left(1-(1-d)^{MK}\right)\\
		&\geq (1-d)^{2MK}\left(x^{(r-1)}\right)^2 + \left(1-\exp\left(-\frac{1}{101}\right)\right)\frac{c}{2}  \cdot \frac{1}{202M}\cdot \eta^2\zeta^2\frac{1}{d^2} \\
		&\geq (1-\lambda\eta)^{2MK}\left(x^{(r-1)}\right)^2 + \frac{(1-\exp(-1/101))c}{404} \frac{\zeta^2}{\lambda^2M}\\
		&\geq (1-\lambda\eta)^{2MK}\left(x^{(r-1)}\right)^2 + c' \frac{\zeta^2}{\lambda^2M} \tag{$c' = \frac{(1-\exp(-1/101))}{404}\cdot c$},
	\end{align*}
	where we use $\frac{M}{M-1}\geq 1$, $1-(1-d)^{MK} \geq 1-\exp(-dMK) \geq 1-\exp(-1/101)$, $\frac{1}{1+(1-d)^K} \geq \frac{1}{2}$ and $1-(1-d)^K \geq \frac{1}{202M}$ in the second inequality. Then,
	\begin{align*}
		&\E\left[\left(x^{(R)}\right)^2\right] \geq (1-\lambda\eta)^{2MKR} \left(x^{(0)}\right)^2 + \sum_{r=0}^{R-1} (1-\lambda\eta)^{2MKr}c' \frac{\zeta^2}{\lambda^2M} \geq c'\frac{\zeta^2}{\lambda^2M},\\
		&\E\left[F(x^{(R)})-F^\ast\right] = \E\left[F(x^{(R)})\right] = \frac{\lambda}{2} \E\left[\left(x^{(R)}\right)^2\right] \geq \frac{c'}{2}\frac{\zeta^2}{\lambda M}= \Omega\left(\frac{\zeta^2}{\lambda M}\right).
	\end{align*}
	
	\subsubsection{Lower Bound for $\frac{1}{\lambda K} \leq \eta \leq \frac{1}{\lambda}$}
	We still start from Eq.~\eqref{eq:proof:lower bound-heterogeneity:parameter square recursion} and Eq.~\eqref{eq:proof:lower bound-heterogeneity:composite function}. For $d=\lambda\eta = 1$, we have $\E\left[\left(x^{(r)}\right)^2\right] = \eta^2\zeta^2 = \frac{\zeta^2}{\lambda^2}$. For $\frac{1}{K}\leq d <1$, we have $t = 1-(1-d)^K \geq 1-\exp(-dK) \geq 1-\exp(-1) \approx 0.63 > 0.5$. Then we can get the lower bound of $T(t)$ on $\frac{1}{2}<t <1$,
	\begin{align*}
		T 
		\geq \left(1-\frac{1}{M}\cdot \frac{2-\frac{1}{2}}{\frac{1}{2}}\right) + \left(1+\frac{1}{M}\cdot \frac{2-\frac{1}{2}}{\frac{1}{2}}\right)(1-\frac{1}{2})^M
		\geq \left( 1-\frac{3}{M} \right) + \left(1+\frac{3}{M}\right)\frac{1}{2^M}.
	\end{align*}
	It can be seen that $T > 1-\frac{3}{M}\geq \frac{1}{4}$ when $M\geq 4$, $T=\frac{1}{4}$ when $M=3$, and $T = \frac{1}{8}$ when $M=2$. Thus, we can obtain that $T \geq c$ for some numerical constant $c$. In fact, since $t \geq 1-\exp(-1) > \frac{1}{202M}$, we can also use the conclusion of the lower bound for $\frac{1}{101 \lambda MK} \leq \eta \leq \frac{1}{\lambda K}$. Then,
	\begin{align*}
		\E\left(x^{(r)}\right)^2 &\geq (1-d)^{2MK}\left(x^{(r-1)}\right)^2+ c\cdot \eta^2\zeta^2 \frac{M}{M-1} \cdot \frac{1}{d^2}\cdot \frac{1-(1-d)^{K}}{1+(1-d)^K}\cdot\left(1-(1-d)^{MK}\right)\\
		&\geq (1-\lambda\eta)^{2MK}\left(x^{(r-1)}\right)^2 + \frac{c}{8}\cdot \frac{\zeta^2}{\lambda^2},
	\end{align*}
	where we use $\frac{M}{M-1}\geq 1$, $1-(1-d)^{MK} \geq 1-(1-d)^{K} \geq 1-\exp(-dK)\geq \frac{1}{2}$, $\frac{1}{1+(1-d)^K} \geq \frac{1}{2}$. Then, for $\frac{1}{K}\leq d = \lambda\eta < 1$, we have
	\begin{align*}
		&\E\left[\left(x^{(R)}\right)^2\right] \geq (1-\lambda\eta)^{2MKR} \left(x^{(0)}\right)^2 + \sum_{r=0}^{R-1} (1-\lambda\eta)^{2MKr}\frac{c}{8} \frac{\zeta^2}{\lambda^2} \geq \frac{c}{8}\frac{\zeta^2}{\lambda^2},\\
		&\E\left[F(x^{(R)})-F^\ast\right] = \E\left[F(x^{(R)})\right] = \frac{\lambda}{2} \E\left[\left(x^{(R)}\right)^2\right] = \Omega\left(\frac{\zeta^2}{\lambda}\right).
	\end{align*}
	
	Now we complete the proofs for all regimes for heterogeneity terms in Theorem~\ref{thm:lower bound}. The setups and final results are summarized in Table~\ref{tab:lower bounds:heterogeneity}. 
\end{proof}

\subsection{Helpful Lemmas for Heterogeneity Terms}\label{app:subsec:helpful lemmas-heterogeneity}

\begin{lemma}
	\label{lem:composite function}
	The function $T(d)$ defined below is monotonically increasing on the interval $0<d<1$, for integers $M\geq 2$ and $K \geq 1$.
	\begin{align*}
		T(d) = 1+(1-d)^{MK} - \frac{1}{M}\cdot \frac{1+(1-d)^{K}}{1-(1-d)^K }\cdot \left(1-(1-d)^{MK}\right).
	\end{align*}
\end{lemma}
\begin{proof}
	Here we introduce an intermediate variable $t = 1-(1-d)^K$ (it implies that $(1-d)^K = 1-t$, $(1-d)^{MK} = (1-t)^M$) and analyze the function $T(t)$ on $0<t<1$ at first.
	\begin{align*}
		T &= 1+(1-t)^M - \frac{1}{M} \cdot \frac{2-t}{t} \cdot \left(1-(1-t)^M\right)\\
		&=\left(1-\frac{1}{M}\cdot \frac{2-t}{t}\right) + \left(1+\frac{1}{M}\cdot \frac{2-t}{t}\right)(1-t)^M.
	\end{align*}
	Then, we follow a similar way to th proof of Lemma 1 in \cite{safran2020good} to prove $T(t)$ is increasing. The derivative of the function $T(t)$ is
	\begin{align*}
		T(t)' &= \frac{2}{Mt^2} - \frac{2}{Mt^2}(1-t)^M - \left( 1+\frac{1}{M}\cdot \frac{2-t}{t}\right)\cdot M (1-t)^{M-1} \\
		&= \frac{2}{Mt^2} \left( 1-(1-t)^{M-1} \cdot \left( 1+(M-1)t + \frac{1}{2}M(M-1)t^2 \right) \right).
	\end{align*}
	The Taylor expansion of $\left((1-t)^{1-M}\right)$ at $t=0$ is
	\begin{align*}
		(1-t)^{1-M} = 1 + (M-1)t + \frac{1}{2}M(M-1) t^2 + \frac{1}{3!} (M+1)M(M-1)(1-\xi)^{-M-2}t^3,
	\end{align*}
	where $\xi \in [0,t]$. When $0<t< 1$, the remainder $\frac{1}{3!} (M+1)M(M-1)(1-\xi)^{-M-2}t^3> 0$ for $M\geq 2$. So we can get $(1-t)^{1-M} > 1 + (M-1)t + \frac{1}{2}M(M-1) t^2$. It follows that
	\begin{align*}
		T' > \frac{2}{Mt^2} (1-(1-t)^{M-1} \cdot (1-t)^{1-M}) = 0.
	\end{align*}
	Thus, $T(t)$ is monotonically increasing on $0<t<1$. Since $t= 1-(1-d)^K$ is monotonically increasing on $0<d<1$, we can get that $T(d)$ is monotonically increasing on $0<d<1$.
\end{proof}

\begin{lemma}
	\label{lem:lower bound:partial sum absolute value bounds}
	Let $\tau=(\tau_1, \tau_2, \ldots,\tau_M)$ be a random permutation of $\frac{M}{2}$ $+1$'s and $\frac{M}{2}$ $-1$'s. Let $\gA_{m,k} \coloneqq \gE_{m-1} + a_k \tau_m$, where $\gE_{m-1} = \sum_{i=1}^{m-1}\tau_{i}$ and $a_k = k/K$ ($0\leq k\leq K-1$). Then,
	\begin{align*}
		\frac{1}{20}\sqrt{m-1+a_k^2} \leq \E\left[\abs{\gA_{m,k}}\right] \leq \sqrt{m-1+a_k^2}  .
	\end{align*}
	Notably, the lower bound holds for $1\leq m \leq \frac{M}{2}+1$ ($M\geq 4$) and the upper bound holds for $1\leq m \leq M$ ($M \geq 2$).
\end{lemma}
\begin{proof}
	We consider the upper and lower bounds as follows.
	For the upper bound, similar to \cite{rajput2020closing}'s Lemma 12 and \cite{cha2023tighter}'s Lemma B.5, we have
	\begin{align*}
		\E\left[\abs{\gA_{m,k}}\right]
		&= \E\left[\abs{\sum_{i=1}^{m-1}\tau_i + a_k \tau_m}\right]\\
		&\leq \sqrt{\E\left[\left(\sum_{i=1}^{m-1}\tau_i + a_k \tau_m\right)^2\right]}\\
		&\leq \sqrt{\sum_{i=1}^{m-1}\E[\left(\tau_i\right)^2]+2\sum_{i<j\leq m-1}\E[\tau_i\tau_j]+a_k^2+2a_k\sum_{i=1}^{m-1}\E(\tau_i\tau_m) }\\
		&\leq \sqrt{m-1+a_k^2} \tag{$\because \E[\tau_i\tau_j]<0$, $\forall i\neq j$} .
	\end{align*}
	For the lower bound, suppose that $3 \leq m\leq \frac{M}{2}+1$ (i.e., $2 \leq m-1\leq \frac{M}{2}$ and $M\geq 4$). Then,
	\begin{align*}
		\E\left[\abs{\gA_{m,k}}\right]
		&= \E\left[\abs{\gE_{m-1} + a_k \tau_m}\right]\\
		&= \E\left[\abs{\gE_{m-1} + a_k \tau_m}\mid \gE_{m-1}\tau_m\geq 0 \right]\cdot \Pr(\gE_{m-1}\tau_m\geq 0) \\
		&\quad+ \E\left[\abs{\gE_{m-1} + a_k \tau_m}\mid \gE_{m-1}\tau_m< 0 \right]\cdot \Pr(\gE_{m-1}\tau_m< 0)\\
		&= \E\left[\abs{\gE_{m-1}}\mid \gE_{m-1}\tau_m\geq 0 \right]\cdot \Pr(\gE_{m-1}\tau_m\geq 0) + a_k \cdot \Pr(\gE_{m-1}\tau_m\geq 0) \\
		&\quad+ \E\left[\abs{\gE_{m-1}}\mid \gE_{m-1}\tau_m< 0 \right]\cdot \Pr(\gE_{m-1}\tau_m< 0) - a_k\cdot \Pr(\gE_{m-1}\tau_m< 0)\\
		&=\E\left[\abs{\gE_{m-1}}\right] + a_k\Pr(\gE_{m-1}\tau_m\geq 0)- a_k\Pr(\gE_{m-1}\tau_m< 0)\\
		&=\E\left[\abs{\gE_{m-1}}\right] + a_k\Pr(\gE_{m-1}\tau_m=0) + a_k\Pr(\gE_{m-1}\tau_m>0)- a_k\Pr(\gE_{m-1}\tau_m< 0)\\
		&=\E\left[\abs{\gE_{m-1}}\right] + a_k\Pr(\gE_{m-1}\tau_m=0) \\
		&\quad+a_k\sum_{i=1}^{m-1}\Pr(\gE_{m-1}\tau_m>0\mid \abs{\gE_{m-1}=i})\Pr(\abs{\gE_{m-1}=i})\\
		&\quad- a_k\sum_{i=1}^{m-1}\Pr(\gE_{m-1}\tau_m<0\mid \abs{\gE_{m-1}=i})\Pr(\abs{\gE_{m-1}=i}) .
	\end{align*}
	Since $\Pr(\gE_{m-1}\tau_m>0\mid \abs{\gE_{m-1}=i}) = \frac{(M-m+1-i)/2}{M-m+1}$ and $\Pr(\gE_{m-1}\tau_m<0\mid \abs{\gE_{m-1}=i}) = \frac{(M-m+1+i)/2}{M-m+1}$, we get
	\begin{align*}
		\E\left[\abs{\gA_{m,k}}\right] &= \E\left[\abs{\gE_{m-1}}\right] +a_k\Pr(\gE_{m-1}\tau_m=0) - a_k \cdot \frac{1}{M-m+1} \sum_{i=1}^{m-1} i\cdot \Pr(\abs{\gE_{m-1}}=i)\\
		&=\left(1-\frac{a_k}{M-m+1}\right)\E\left[\abs{\gE_{m-1}}\right] + a_k\Pr(\gE_{m-1}\tau_m=0) \\
		&\geq \left(1-\frac{1}{M-m+1}\right)\E\left[\abs{\gE_{m-1}}\right]\tag{$\because 0\leq a_k < 1$}\\
		&\geq \left(1-\frac{2}{M}\right)\E\left[\abs{\gE_{m-1}}\right] \tag{$\because m-1\leq \frac{M}{2}$} \\
		&\geq \frac{1}{2}\E\left[\abs{\gE_{m-1}}\right], \tag{$\because M\geq 4$}
	\end{align*}
	where we use $\E\left[\abs{\gE_{m-1}}\right] = \sum_{i=1}^{m-1}i\cdot \Pr(\abs{\gE_{m-1}}=i)$ in the second equality.
	
	For the convenience of subsequent proofs, we need a tighter lower bound for $\E\left[\abs{\gE_{m-1}}\right]$, which can be achieved with a few modifications to \cite{cha2023tighter}'s Lemma B.5.
	\begin{mdframed}[roundcorner=1mm]
		Let us start from Ineq.~(21) in \cite{cha2023tighter}'s Lemma B.5.
		
		For the even integers $m\geq 4$, we have
		\begin{align*}
			\E\left[\abs{\gE_m}\right]
			&\geq \left(\frac{M-2}{M-1} \cdot \frac{\sqrt{m}}{\sqrt{m-2}}\right) \cdot \left(\frac{M}{M+2m}\right) \cdot \left(\frac{\sqrt{m}}{5}\right)\\
			&= \left(\frac{M-2}{M-1} \cdot \frac{{\sqrt{m-1}}}{\sqrt{m-2}}\right) \cdot \left(\frac{M}{M+2m}\right) \cdot \left(\frac{{\sqrt{m+1}}}{5}\right)
		\end{align*}
		It can be shown that $\frac{M-2}{M-1} \cdot \frac{\sqrt{m-1}}{\sqrt{m-2}}\geq 1 \iff (2M-3) m\leq M^2-2$. Since $m\leq \frac{M}{2}$ (Note that the constraint $m\leq \frac{M}{2}$ is for $\E\left[\abs{\gE_m}\right]$ in \cite{cha2023tighter}'s Lemma B.5), it follows that $(2M-3) m \leq M^2 - \frac{3}{2}M \leq M^2 -2$ when $M \geq 8$. Then, we can get $\E\left[\gE_m\right]\geq \frac{\sqrt{m+1}}{10}$ for $4\leq m\leq \frac{M}{2}$. 
		
		For the even integers $m=2$, we have $\E\left[\abs{\gE_2}\right] = 1-\frac{1}{M-1} \geq \frac{2}{3}\geq \frac{\sqrt{3}}{10}$ ($M\geq 2m \geq 4$). Now, we complete the proof for the even cases $2\leq m \leq \frac{M}{2}$.
		
		In fact, the lower bound holds in odd cases $1\leq m\leq \frac{M}{2}$ in \cite{cha2023tighter}'s Lemma B.5 without any modification (see their last inequality $\E\left[\gE_m\right]\geq \frac{M-m}{M-m-1} \cdot \frac{\sqrt{m+1}}{10}\geq \frac{\sqrt{m+1}}{10}$). We can also prove it with the same steps as \cite{cha2023tighter}'s Lemma B.5.
		
		Note that the lower bound does not hold in the last case $m=0$. 
		
		As a summary, we can get a tighter bound $\E\left[\abs{\gE_{m}}\right] \geq \frac{\sqrt{m+1}}{10}$ for $1\leq m \leq \frac{M}{2}$.
	\end{mdframed}
	
	Returning to $\E\left[\abs{\gA_{m,k}}\right]$ and using the tighter lower bound for $\E\left[\abs{\gE_{m-1}}\right]$, we have
	\begin{align*}
		\E\left[\abs{\gA_{m,k}}\right] \geq \frac{1}{2}\E\left[\abs{\gE_{m-1}}\right] \geq \frac{\sqrt{m}}{20} .
	\end{align*}
	The above lower bound does not hold for $m=1$ since it requires the false argument $\gE_0=0\geq \frac{\sqrt{1}}{10}$. To incorporate the case where $m=1$, we consider a looser bound
	\begin{align*}
		\E\left[\abs{\gA_{m,k}}\right] \geq \frac{\sqrt{m}}{20} \geq \frac{1}{20}\sqrt{m-1+a_k^2} .
	\end{align*}
	At last, let us verify whether this lower bound holds for the remaining cases where $m=1,2$. When $m=1$, $\E\left[\abs{\gA_{1,k}}\right] = a_k \geq \frac{a_k}{20\sqrt{2}}$. When $m=2$ ($M\geq 2m \geq 4$), it follows that
	\begin{align*}
		\E\left[\abs{\gA_{2,k}}\right] = \E\left[\abs{\tau_1 + a_k \tau_2}\right]
		&= (1+a_k) \Pr(\tau_1\tau_2=+1) + (1-a_k) \Pr(\tau_1\tau_2=-1)\\
		&= (1+a_k)\cdot \frac{2\cdot {\frac{M}{2} \choose 2}}{{M \choose 2}} + (1-a_k) \cdot \frac{{\frac{M}{2} \choose 1}\cdot {\frac{M}{2} \choose 1}}{{M \choose 2}}\\
		&= 1-\frac{a_k}{M-1}.
	\end{align*}
	Here we adopt $M \geq 4$ for $m=2$, and then $\E\left[\abs{\gA_{2,k}}\right] =1-\frac{a_k}{M-1} \geq 1-\frac{a_k}{2} \geq \frac{1}{2} = \frac{1}{2\sqrt{2}}\sqrt{1+1^2} \geq \frac{1}{20}\sqrt{1+a_k^2}$. Now we complete the proof of the lower bound of $\E\left[\abs{\gA_{2,k}}\right]$, which holds for $1<m\leq \frac{M}{2}+1$ and $M\geq 4$.
\end{proof}

\begin{lemma}
	\label{lem:lower bound:partial sum probability bounds}
	Let $\tau=(\tau_1, \tau_2, \ldots, \tau_M)$ be a random permutation of $\frac{M}{2}$ $+1$'s and $\frac{M}{2}$ $-1$'s. Let $\gA_{m,k} \coloneqq \gE_{m-1} + a_k \tau_m$, where $\gE_{m-1} = \sum_{i=1}^{m-1}\tau_{i}$ and $a_k = k/K$ ($0\leq k\leq K-1$). The probability distribution of $\gA_{m,k}$ is symmetric with respect to $0$. And For $1\leq m\leq M$ and $0\leq k \leq K-1$ (excluding the case $m=1, k=0$), it holds that
	\begin{align*}
		\frac{1}{6}\leq \Pr(\gA_{m,k}>0) = \Pr(\gA_{m,k}<0) \leq \frac{1}{2}.
	\end{align*}
\end{lemma}
\begin{proof}
	When $m=1$ and $k=0$, $\gA_{1,0} = \gE_{0} =0 $ (defined). When $m=M+1$ and $k=0$, $\gA_{M+1,0} = 0$. In these two cases, $\Pr(\gA_{m,k}=0) = 1$ and $\Pr(\gA_{m,k}<0)=\Pr(\gA_{m,k}>0) = 0$.
	
	When $2 \leq m \leq M$ and $k=0$, we get $\Pr(\gA_{m,k}>0) = \Pr(\gA_{m,k}<0)\geq \frac{1}{6}$ according to \cite{yun2022minibatch}' Lemma 14.
	
	When $1 \leq m \leq M$ and $0<k\leq K-1$, similarly, we can first prove that $\gA_{m,k}$ is symmetric and then compute $\Pr(\gA_{m,k}=0)$. As shown in Table~\ref{tab:probability distribution table}, we conclude all cases into four categories $\gA_{m,k} = -p-a_k$, $\gA_{m,k} = -p+a_k$, $\gA_{m,k} = p-a_k$ and $\gA_{m,k} = p+a_k$. We can get that the probability distribution of $\gA_{m,k}$ is symmetric. Furthermore, since $0<a_k<1$, we can get that $\Pr(\gA_{m,k}=0) = 0$, and thus $\Pr(\gA_{m,k}>0)=\Pr(\gA_{m,k}<0) = \frac{1}{2}$.
	\begin{table}[h]
		\centering
			\begin{tabular}{l|l}
				\toprule
				Value &Probability \\\midrule
				$-p-a_k$ &$\frac{{\frac{M}{2} \choose \frac{m-1-p}{2}} {\frac{M}{2} \choose \frac{m-1+p}{2}}}{{M \choose m-1}} \cdot \frac{\frac{M-m+1-p}{2}}{M-m+1}$\\\midrule
				
				$-p+a_k$ &$\frac{{\frac{M}{2} \choose \frac{m-1-p}{2}} {\frac{M}{2} \choose \frac{m-1+p}{2}}}{{M \choose m-1}} \cdot \frac{\frac{M-m+1+p}{2}}{M-m+1}$\\\midrule
				$p-a_k$ &$\frac{{\frac{M}{2} \choose \frac{m-1-p}{2}} {\frac{M}{2} \choose \frac{m-1+p}{2}}}{{M \choose m-1}} \cdot \frac{\frac{M-m+1+p}{2}}{M-m+1}$\\\midrule
				$p+a_k$ &$\frac{{\frac{M}{2} \choose \frac{m-1-p}{2}} {\frac{M}{2} \choose \frac{m-1+p}{2}}}{{M \choose m-1}} \cdot \frac{\frac{M-m+1-p}{2}}{M-m+1}$\\
				\bottomrule
		\end{tabular}
		\caption{Probability distribution of $\gA_{m,k}$.}
		\label{tab:probability distribution table}
	\end{table}
	
	In summary, we have proved that $\Pr(\gA_{m,k}>0) = \Pr(\gA_{m,k}<0) \geq \frac{1}{6}$ for $1\leq m\leq M$ and $0\leq k \leq K-1$ (except the case $m=1, k=0$). Note that for all cases, the probability distribution is symmetric, we can get $\Pr(\gA_{m,k}>0) = \Pr(\gA_{m,k}<0) \leq \frac{1}{2}$.
\end{proof}

\begin{lemma}
	\label{lem:lower bound:diff current-initial parameter upper bound}
	Supposing that $x_{1,0}\geq 0$, ${\lambda_0}/{\lambda}\geq {1010}$ and $\eta \leq \frac{1}{101\lambda MK}$, then for $1\leq m \leq M$, $0\leq k\leq K-1$, we have
	\begin{align*}
		\E\left[\abs{x_{m,k}- x_{1,0}}\right] \leq \frac{1}{100} x_{1,0} + \frac{101}{100}K\eta\zeta\sqrt{m-1+a_k^2}
	\end{align*}
\end{lemma}
\begin{proof}
	According to Eq.~\eqref{eq:proof:lower bound-heterogeneity:curr-init},
	\begin{align*}
		x_{m,k} = x_{1,0} -\eta \sum_{i=0}^{\gB_{m,k}-1} \left((\lambda\mathbbm{1}_{x_{b_1(i),b_2(i)} <0} + \lambda_0 \mathbbm{1}_{x_{b_1(i),b_2(i)} \geq 0}) x_{b_1(i),b_2(i)}  \right) - K\eta\zeta\gA_{m,k}
	\end{align*}
	(we have dropped the superscript $r$), we can get
	\begin{align*}
		\E\left[\abs{x_{m,k}- x_{1,0}}\right] &\leq \eta \sum_{i=0}^{\gB_{m,k}-1} \E\left[\abs{(\lambda\mathbbm{1}_{x_{b_1(i),b_2(i)} <0} + \lambda_0 \mathbbm{1}_{x_{b_1(i),b_2(i)} \geq 0}) x_{b_1(i),b_2(i)}}  \right] + K\eta\zeta\E\left[\abs{\gA_{m,k}}\right]\\
		&\leq \lambda\eta \sum_{i=0}^{\gB_{m,k}-1} \E\left[\abs{x_{b_1(i),b_2(i)}}  \right] + K\eta\zeta\E\left[\abs{\gA_{m,k}}\right] \tag{$\because \lambda_0 \leq \lambda$}\\
		&\leq \lambda\eta \sum_{i=0}^{\gB_{m,k}-1} \E\left[x_{1,0}\right] + \lambda\eta \sum_{i=0}^{\gB_{m,k}-1} \E\left[\abs{x_{b_1(i),b_2(i)}-x_{1,0}}  \right] + K\eta\zeta\E\left[\abs{\gA_{m,k}}\right]. 
	\end{align*}
	For any integers $m\geq 1$, $k\geq 0$ satisfying $(m-1)K+k \leq MK$, using Lemma~\ref{lem:lower bound:partial sum absolute value bounds}, $\E\left[\abs{\gA_{m,k}}\right]\leq \sqrt{m-1+a_k^2}$, we can get
	\begin{align*}
		\E\left[\abs{x_{m,k}- x_{1,0}}\right]\leq \lambda\eta \gB_{m,k}\E\left[x_{1,0}\right] + \lambda\eta \sum_{i=0}^{\gB_{m,k}-1} \E\left[\abs{x_{b_1(i),b_2(i)}-x_{1,0}}  \right] + K\eta\zeta\sqrt{m-1+a_k^2}.
	\end{align*}
	Let $h_{m,k} \coloneqq \lambda\eta\gB_{m,k}\E\left[x_{1,0}\right] + \lambda\eta\sum_{i=0}^{\gB_{m,k}-1}h_{b_1(i),b_2(i)} + K\eta\zeta\sqrt{m-1+a_k^2}$ and $h_{1,0} = 0$. It can be verified that the sequence $h_{1,0}, \ldots,h_{m,0}, h_{m,1},h_{m,2},\ldots, h_{m,K-1},\ h_{m+1,0}, \ldots, h_{M+1,0}$ is monotonically increasing.
	When $k=0$, then
	\begin{align*}
		h_{m,0} - h_{m-1,K-1} =\lambda\eta \E\left[x_{1,0}\right] + \lambda\eta h_{m-1,K-1} + K\eta\zeta\left(\sqrt{m-1}-\sqrt{m-2+a_{K-1}^2}\right) > 0.
	\end{align*}
	When $1\leq k\leq K-1$, then
	\begin{align*}
		h_{m,k} - h_{m,k-1} = \lambda\eta \E\left[x_{1,0}\right] + \lambda\eta h_{m,k-1}+K\eta\zeta\left(\sqrt{m-1+a_{k}^2} - \sqrt{m-1+a_{k-1}^2}\right) > 0.
	\end{align*}
	This means that $h_{b_1(i),b_2(i)} < h_{m,k}$ for any integer $i < \gB_{m,k}$. So we can get
	\begin{align*}
		&h_{m,k}\leq \lambda\eta \gB_{m,k}\E\left[x_{1,0}\right] + \lambda\eta \gB_{m,k} h_{m,k} + K\eta\zeta\sqrt{m-1+a_k^2}\\
		\implies &h_{m,k} \leq \frac{\lambda\eta\gB_{m,k}\E\left[x_{1,0}\right]}{1-\lambda\eta\gB_{m,k}} +\frac{K\eta\zeta\sqrt{m-1+a_k^2}}{1-\lambda\eta\gB_{m,k}} \tag{$\because\lambda \eta \gB_{m,k} \leq \lambda \eta \gB_{M,K} =\lambda MK\eta\leq \frac{1}{101}$ }
	\end{align*}
	By mathematical induction, we can get $\E\left[\abs{x_{m,k}- x_{1,0}}\right] \leq h_{m,k} \leq \frac{\lambda\eta\gB_{m,k}x_{1,0}}{1-\lambda\eta\gB_{m,k}} +\frac{K\eta\zeta\sqrt{m-1+a_k^2}}{1-\lambda\eta\gB_{m,k}}$. When $m=1$ and $k=0$, $\E\left[\abs{x_{1,0}- x_{1,0}}\right] = h_{1,0} = 0$. Then, suppose that $\E\left[\abs{x_{i,j}- x_{1,0}}\right]\leq h_{i,j}$ for all $i,j$ satisfying $i(K-1)+j \leq m(K-1)+k$.
	\begin{itemize}
		\item When $k=0$, it follows that
		\begin{align*}
			\E\left[\abs{x_{m,0}- x_{1,0}}\right] 
			&\leq \lambda\eta \gB_{m,0}\E\left[x_{1,0}\right] + \lambda\eta\sum_{i=0}^{\gB_{m,0}-1}\E\left[\abs{x_{b_1(i),b_2(i)}-x_{1,0}}\right] + K\eta\zeta\sqrt{m-1} \\
			&\leq \lambda\eta \gB_{m,0}\E\left[x_{1,0}\right] + \lambda\eta\sum_{i=0}^{\gB_{m,0}-1}h_{b_1(i),b_2(i)} + K\eta\zeta\sqrt{m-1} \leq h_{m,0}.
		\end{align*}
		\item When $1\leq k \leq K-1$, it follows that
		\begin{align*}
			\E\left[\abs{x_{m,k}- x_{1,0}}\right] &\leq \lambda\eta \gB_{m,k}\E\left[x_{1,0}\right] + \lambda\eta\sum_{i=0}^{\gB_{m,k}-1}\E\left[\abs{x_{b_1(i),b_2(i)}-x_{1,0}}\right] + K\eta\zeta\sqrt{m+a_{k}^2} \\
			&\leq \lambda\eta \gB_{m,k}\E\left[x_{1,0}\right] + \lambda\eta\sum_{i=0}^{\gB_{m,k}-1}h_{b_1(i),b_2(i)} + K\eta\zeta\sqrt{m+a_{k}^2} \leq h_{m,k}.
		\end{align*}
	\end{itemize}
	Then considering that ${\lambda_0}/{\lambda}\geq {1010}$ and $\lambda \eta \gB_{m,k} \leq \lambda MK\eta\leq \frac{1}{101}$ (it implies $\frac{\lambda\eta\gB_{m,k}}{1-\lambda\eta\gB_{m,k}}\leq \frac{1}{100}$ and that $\frac{1}{1-\lambda\eta\gB_{m,k}}\leq \frac{101}{100}$), we get $\E\left[\abs{x_{m,k}- x_{1,0}}\right] \leq \frac{1}{100} x_{1,0} + \frac{101}{100}K\eta\zeta\sqrt{m-1+a_k^2}$.
\end{proof}

\section{The Mechanism of ``Two Learning Rates'' in SFL}
\label{app:two-step-sizes}

This section shows that the mechanism of ``two learning rates'' can also be applied to SFL. In theory, it can achieve a similar improvement to (almost the same as) that in PFL \citet{karimireddy2020scaffold}. Next, we show how to use the mechanism of ``two learning rates'' in SFL, and compare it with that of PFL. 

\subsection{The Mechanism of ``Two Learning Rates'' in SFL}

The mechanism of ``two learning rates'' includes two learning rates, the global/server learning rate and the local/client learning rate. The client learning rate is on the client-side for local updates; the server learning rate is on the server-side for global updates. The modified algorithms are provided in Algorithm~\ref{algorithm1:two-step-size} and Algorithm~\ref{algorithm2:two-step-size}. The modified lines are marked in a pink color box. Here $\gamma$ denotes the server learning rate; $\eta$ denotes the client learning rate. For SFL, one simple and practical implementation is illustrated in Figure~\ref{fig:SFL2}.
\begin{figure}[h]
	\centering
	\includegraphics[width=0.9\linewidth]{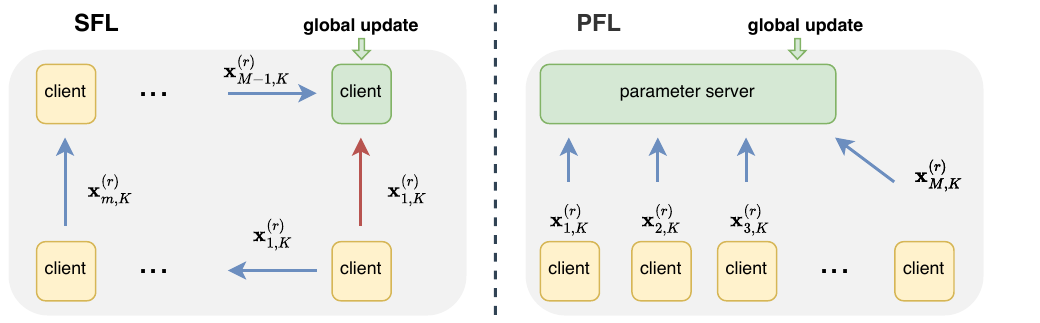}
	\caption{The mechanism of ``two learning rates'' in SFL and PFL. The global updates of SFL are performed at the last client. It performs the global updates with its parameters $\rvx_{M,K}^{(r)}$ and the initial parameters $\rvx^{(r)}$ received from the first client.}
	\label{fig:SFL2}
\end{figure}

\noindent
\begin{minipage}[t]{0.5\linewidth}
	\DecMargin{0.5em}
	\begin{algorithm}[H]
		\DontPrintSemicolon
		\caption{Sequential FL}
		\label{algorithm1:two-step-size}
		\For{$r = 0, \ldots, R-1$}{
			Sample a permutation $\pi_1, \pi_2, \ldots, \pi_{M}$ of $\{1,2,\ldots,M\}$\;
			\For{$m = 1,\ldots,M$ {\bf\textcolor{red}{in sequence}}}{
				$\rvx_{m,0}^{(r)} =
				\begin{cases}
					\rvx^{(r)}\,, &m=1\\
					\rvx_{m-1,K}^{(r)}\,, &m>1
				\end{cases}$\;
				\For{$k = 0,\ldots, K-1$}{
					$\rvx_{m,k+1}^{(r)} = \rvx_{m,k}^{(r)} - \eta \rvg_{\pi_m, k}^{(r)}$\;
				}
			}
			\colorbox{pink!50}{$\rvx^{(r+1)} = \rvx^{(r)} - \gamma \left(\rvx^{(r)} - \rvx_{M,K}^{(r)}\right)$}
		}
	\end{algorithm}
	\IncMargin{0.5em}
\end{minipage}
\hfill
\begin{minipage}[t]{0.5\linewidth}
	\DecMargin{0.5em}
	\begin{algorithm}[H]
		\DontPrintSemicolon
		\caption{Parallel FL}
		\label{algorithm2:two-step-size}
		\For{$r = 0,\ldots, R-1$}{
			\For{$m = 1,\ldots,M$ {\bf \textcolor{red}{in parallel}}}{
				$\rvx_{m,0}^{(r)} = \rvx^{(r)}$\;
				\For{$k = 0,\ldots, K-1$}{
					$\rvx_{m,k+1}^{(r)} = \rvx_{m,k}^{(r)} - \eta \rvg_{m,k}^{(r)}$\;
				}
			}
			\colorbox{pink!50}{$\rvx^{(r+1)} = \rvx^{(r)} - \gamma \left(\rvx^{(r)} - \frac{1}{M} \sum_{m=1}^M \rvx_{m,K}^{(r)}\right)$}
			
		}
	\end{algorithm}
	\IncMargin{0.5em}
\end{minipage}

According to the new update rule of SFL, we can get the upper bounds of SFL and PFL in {Theorem~\ref{thm:SFL-non-convex-server-step-size}} and {Theorem~\ref{thm:PFL-non-convex-server-step-size}}. We only consider the non-convex case for convenience. We see that the server learning rate $\gamma$ has the similar effect on the upper bounds for SFL to those for PFL. The advantage of $\frac{1}{M^{1/3}}$ of SFL still exists.

\begin{theorem}[SFL, non-convex]\label{thm:SFL-non-convex-server-step-size}
Under the same conditions as those of the non-convex case in Theorem 1, there exists $\tilde\eta = \gamma \eta MK\lesssim \frac{1}{L(1+\beta^2/M)}$ ($\gamma\geq 1$), such that
\begin{align*}
	\min_{0\leq r\leq R} \E\left[\Norm{\nabla F(\rvx^{(r)})}^2\right]
	&\lesssim \frac{A}{\tilde \eta R} + \frac{ \tilde \eta L \sigma^2}{MK} + \frac{\tilde\eta^2L^2\sigma^2}{\gamma^2MK} + \frac{\tilde\eta^2L^2\zeta^2}{\gamma^2M}.
\end{align*}
After tuning the learning rate, we get
\begin{align*}
	&\min_{0\leq r\leq R} \E\left[\Norm{\nabla F(\rvx^{(r)})}^2\right] \\
	&=\gO\left( \frac{ LA\left(1+\frac{\beta^2}{M}\right) }{R } + \frac{\left(L\sigma^2 A\right)^{1/2} }{\sqrt{MKR}} + \frac{ \left(L^2\sigma^2A^2\right)^{1/3} }{ {\color{darkgreen}\boldsymbol{\gamma^{2/3}}}{\color{red}\boldsymbol{M^{1/3}}}K^{1/3} R^{2/3}} + \frac{ \left(L^2\zeta^2A^2\right)^{1/3} }{ {\color{darkgreen}\boldsymbol{\gamma^{2/3}}}{\color{red}\boldsymbol{M^{1/3}}} R^{2/3}}\right).
\end{align*}
\end{theorem}
\begin{proof}
See Appendix~\ref{subsec:app-server-step-size}.
\end{proof}

\begin{theorem}[PFL, non-convex]\label{thm:PFL-non-convex-server-step-size}
Under the same conditions as those of the non-convex case in Theorem 1, there exists $\tilde\eta = \gamma \eta K\lesssim \frac{1}{L(1+\beta^2)}$ ($\gamma \geq 1$), such that
\begin{align*}
	\min_{0\leq r\leq R} \E\left[\Norm{\nabla F(\rvx^{(r)})}^2\right] &\lesssim \frac{A}{\tilde\eta R} + \frac{\tilde\eta L \sigma^2}{MK} + \frac{\tilde\eta^2L^2 \sigma^2}{\gamma^2 K}  + \frac{\tilde\eta^2L^2\zeta^2}{\gamma^2}.
\end{align*}
After tuning the learning rate, we get
\begin{align*}
	\min_{0\leq r\leq R} \E\left[\Norm{\nabla F(\rvx^{(r)})}^2\right] &=\gO\left( \frac{ LA(1+\beta^2) }{R } + \frac{\left(L\sigma^2 A\right)^{1/2} }{\sqrt{MKR}} + \frac{ \left(L^2\sigma^2A^2\right)^{1/3} }{ {\color{darkgreen}\boldsymbol{\gamma^{2/3}}} K^{1/3} R^{2/3}} + \frac{ \left(L^2\zeta^2A^2\right)^{1/3} }{ {\color{darkgreen}\boldsymbol{\gamma^{2/3}}}R^{2/3}}\right).
\end{align*}
\end{theorem}
\begin{proof}
See \citet{karimireddy2020scaffold}'s Theorem~1 or \citet{yang2021achieving}'s Theorem~1. Note that \citet{karimireddy2020scaffold} set $\gamma = \sqrt{M}$ and \citet{yang2021achieving} set $\gamma = \sqrt{MK}$. Here we keep $\gamma = \gamma$ for comparison.
\end{proof}

\subsection{Proofs of {Theorem~\ref{thm:SFL-non-convex-server-step-size}}}
\label{subsec:app-server-step-size}

\begin{proof}
	We consider the full client participation for simplicity. Since the global update is performed at the end of one training round, the client drift bound (that is, \citet{li2023convergence}'s Lemma~10) is unaffected. We can focus on \citet{li2023convergence}'s Lemma~9. Substituting the overall updates $\Delta \rvx = - \eta \sum_{m=1}^{M} \sum_{k=0}^{K-1} \rvg_{\pi_{m}} (\rvx_{m,k})$ with $\Delta \rvx = -\eta\gamma\sum_{m=1}^{M} \sum_{k=0}^{K-1} \rvg_{\pi_{m}} (\rvx_{m,k})$ in \citet{li2023convergence}'s Lemma~9, we get the recursion
	\begin{align*}
		&\E\left[F(\rvx^{(r+1)}) - F(\rvx^{(r)})\right]\\ &\leq-\frac{1}{6}\eta\gamma  MK\E\Norm{\nabla F(\rvx^{(r)})}^2 + 2\eta^2\gamma^2LMK\sigma^2
		+ \frac{5}{6}\eta\gamma L^2\sum_{m=1}^M\sum_{k=0}^{K-1}\E\Norm{\rvx_{m,k}^{(r)}-\rvx^{(r)}}^2.
	\end{align*}
	Plugging \citet{li2023convergence}'s Lemma~10 into it, and using $\eta\gamma  \leq \frac{1}{6LMK(1+\beta^2/M)}$, we get
	\begin{align*}
		&\E \left[F(\rvx^{(r+1)})-F(\rvx^{(r)})\right]\\
		&\leq -\frac{1}{10}\eta \gamma MK\E\|\nabla F(\rvx^{(r)})\|^2+2L\eta^2\gamma^2MK\sigma^2 + \frac{15}{8}\eta^3\gamma L^2M^2K^2\sigma^2 + \frac{15}{8}\eta^3\gamma L^2M^2K^3\zeta^2.
	\end{align*}
	Letting $\tilde \eta = \eta \gamma MK$, we can get
	\begin{align*}
		\E \left[F(\rvx^{(r+1)})-F(\rvx^{(r)})\right]\leq -\frac{1}{10}\tilde \eta \E\|\nabla F(\rvx^{(r)})\|^2+ \frac{2L\tilde\eta^2 \sigma^2}{MK} + \frac{15}{8}\frac{\tilde\eta^3L^2\sigma^2}{\gamma^2 MK} + \frac{15}{8}\frac{\tilde\eta^3L^2\zeta^2}{\gamma^2 M}.
	\end{align*}
	Then, we get
	\begin{align*}
		\min_{0\leq r\leq R} \E\left[\Norm{\nabla F(\rvx^{(r)})}^2\right] \leq \frac{10 A}{\tilde \eta R} + \frac{20 \tilde \eta L \sigma^2}{MK} + \frac{75}{4}\frac{\tilde\eta^2L^2\sigma^2}{\gamma^2MK} + \frac{75}{4}\frac{\tilde\eta^2L^2\zeta^2}{\gamma^2M}.
	\end{align*}
	Since $\eta\gamma  \leq \frac{1}{6LMK(1+\beta^2/M)}$, using \citet{li2023convergence}'s Lemma~8, we can get
	\begin{align*}
		&\min_{0\leq r\leq R} \E\left[\Norm{\nabla F(\rvx^{(r)})}^2\right]\\
		&= \gO\left(\frac{ LA\left(1+\frac{\beta^2}{M}\right) }{R } + \frac{\left(L\sigma^2 A\right)^{1/2} }{\sqrt{MKR}} + \frac{ \left(L^2\sigma^2A^2\right)^{1/3} }{ \gamma^{2/3}M^{1/3}K^{1/3} R^{2/3}} + \frac{ \left(L^2\zeta^2A^2\right)^{1/3} }{ \gamma^{2/3}M^{1/3}R^{2/3}}\right).
	\end{align*}
\end{proof}

\section{Comparison with Lower Bounds of SGD-RR}\label{app:comparison-SGD-RR}

In this section, we compare the lower bounds of SFL with those of SGD-RR. The lower bounds of SFL are stated in Theorem~\ref{thm:lower bound} and Theorem~\ref{thm:lower bound2}; Theorem~\ref{thm:lower bound} is for arbitrary learning rates $\eta>0$ and Theorem~\ref{thm:lower bound2} is for small learning rates $0 < \eta \lesssim \frac{1}{LMK}$. By comparing them with the corresponding theorems of SGD-RR in \citet{cha2023tighter}, we have the following comparison results: (1) for arbitrary learning rates $\eta>0$, the lower bound of SFL is worse than that of SGD-RR by a factor of $\kappa$; (2) for small learning rates $0<\eta\lesssim \frac{1}{LMK}$, the lower bounds of SFL match those of SGD-RR. See {Theorems~\ref{thm:SGD-RR lower bound},~\ref{thm:lower bound},~\ref{thm:SGD-RR lower bound small stepsize} and~\ref{thm:lower bound2}}.

Be attention that \textit{SGD-RR is a special case of SFL, where one single step of GD (Gradient Descent) step is performed on each local objective function (that is, $\sigma=0$ and $K=1$).}

\textit{The lower bounds for arbitrary learning rates $\eta>0$.} The lower bounds are restated in Theorem~\ref{thm:SGD-RR lower bound} (SGD-RR, \citealt{cha2023tighter}) and Theorem~\ref{thm:lower bound} (SFL) with our notations. First, we see that there are more components (including stochasticity and heterogeneity) in the lower bounds of SFL. Given that SGD-RR is a special case of SFL, by letting $\sigma=0$ and $K=1$, we next focus on the most noteworthy heterogeneity term (the last term, with $\zeta$) in Theorem~\ref{thm:lower bound}, and compare it with those of SGD-RR in Theorem~\ref{thm:SGD-RR lower bound}.

We consider two cases $R\gtrsim \kappa$ and $R\lesssim  \kappa$. When $R \gtrsim \kappa $, the lower bound $\Omega\left( \frac{L \zeta^2}{\mu^2 M R^2} \right)$ of SGD-RR is better than $\Omega\left( \frac{\zeta^2}{\mu M R^2} \right)$ of SFL with an advantage of $\kappa$. When $R \lesssim \kappa$, the lower bound $\Omega \left( \frac{\zeta^2}{\mu M R }\right)$ of SGD-RR is also better than $\Omega\left( \frac{\zeta^2}{\mu M R^2} \right)$ of SFL. Thus, we get that the lower bounds for SGD-RR in \citet{cha2023tighter} are better than ours for SFL for $\eta>0$. It is still open whether the lower bounds of SFL can be better for arbitrary $\eta>0$.

\textit{The lower bounds for small learning rates $0<\eta \lesssim \frac{1}{LMK}$.} The lower bounds are restated in Theorem~\ref{thm:SGD-RR lower bound small stepsize} (SGD-RR, \citealt{cha2023tighter}) and {Theorem~\ref{thm:lower bound2}} (SFL) with our notations. First, we see that there are more components (including stochasticity and heterogeneity) in the lower bounds of SFL. Similarly, we next focus on the most noteworthy heterogeneity term (the last term, with $\zeta$) in {Theorem~\ref{thm:lower bound2}}. It can be shown that the lower bounds of SFL completely match those of SGD-RR.

\begin{theorem}[Theorem 3.1 in \citet{cha2023tighter}]\label{thm:SGD-RR lower bound}
	For any $M\geq 2$ and $\kappa \geq 2415$, there exist a 3-dimensional function, whose local objective functions are $\mu$-strongly convex (Definition~\ref{def:strong convexity}) and $L$-smooth (Definition~\ref{def:smoothness}), and satisfy Assumption~\ref{asm:heterogeneity:average} (heterogeneity), and an initialization point $\rvx^{(0)}$ such that for any constant learning rate $\eta>0$, the last-round global parameter $\rvx^{(R)}$ satisfy
	\begin{align*}
		\E\left[ F(\rvx^{(R)}) - F^\ast\right] = \begin{cases}\Omega \left( \frac{L \zeta^2}{\mu^2 M R^2}\right) &\text{if} \ R\geq 161 \kappa, \\ \Omega \left( \frac{\zeta^2}{\mu M R }\right) &\text{if} \ R < 161\kappa.\end{cases}
	\end{align*}
\end{theorem}

\begin{theorem}[Theorem~3.3 and Corollary~3.5 in \citet{cha2023tighter}]\label{thm:SGD-RR lower bound small stepsize}
	Under the same conditions of {Theorem~\ref{thm:SGD-RR lower bound}} (unless explicitly stated), there exist a multi-dimensional global objective function and an initialization point, such that for $\eta\leq \frac{1}{161LM}$, the arbitrary weighted average global parameters $\bar\rvx^{(R)}$ satisfy the lower bounds:
	\setlist[itemize]{label=}
	\begin{itemize}[leftmargin=0.5em]
		\item \textbf{Strongly convex}: If $R\geq 161\kappa$ and $\kappa \geq 2415$, then
		\begin{flalign*}
			\E\left[F(\bar\rvx^{(R)})-F(\rvx^\ast)\right] = \Omega \left( \frac{L\zeta^2}{\mu^2 MR^2}\right).&&
		\end{flalign*}
		\item \textbf{General convex}: If $R\geq 161^3\max\left\{ \frac{\zeta}{LM^{1/2}D},  \frac{L^2 MD^2}{\zeta^2}\right\}$, then
		\begin{flalign*}
			\E\left[F(\bar\rvx^{(R)})-F(\rvx^\ast)\right] = \Omega\left(\frac{\left(L\zeta^2D^4\right)^{1/3}}{M^{1/3}R^{2/3}}\right). &&
		\end{flalign*}
	\end{itemize}
\end{theorem}

\section{Comparison with \citet{malinovsky2023federated}}\label{app:comparison-shuffling-variance}

In this section, we compare the results in \citet{malinovsky2023federated} with ours. Notably, the local solver in \citet{malinovsky2023federated} is SGD-RR, so they assume that each local component functions $f_m$ is $L$-smooth, and use the technique of Shuffling Variance \citep{mishchenko2020random}.
Since the bounds in \citet{malinovsky2023federated} are for strongly convex cases, we next compare their bound with ours in the strongly convex case. For comparison, we restate their Theorem~6.1 and our Theorem~\ref{thm:SFL} in Corollary~\ref{cor:PFL-cyclic-participation} and Corollary~\ref{cor:SFL-last-round-iterate-bound}. Since \citet{malinovsky2023federated}'s local solver is SGD-RR, while our local solver is SGD, for convenience and fairness, we next only compare the heterogeneity terms. As shown in Corollary~\ref{cor:PFL-cyclic-participation} and Corollary~\ref{cor:SFL-last-round-iterate-bound}, the upper bounds of \citet{malinovsky2023federated} almost match ours, except an advantage of $MK$ on the first term. This is because they use the advanced technique of Shuffling Variance.

\begin{corollary}[Corollary of \citet{malinovsky2023federated}'s Theorem~6.1]
	\label{cor:PFL-cyclic-participation}
	Letting $R=M$, $T=R$, $\gamma =\eta$, $N =K$, $\tilde\sigma_\star=\zeta_\ast$, $\sigma_\star = 0$ and $C=1$ in \citet{malinovsky2023federated}'s Theorem~6.1, we can get the upper bounds for SFL in our notations:
	\begin{align*}
		&\E\|\rvx^{(R)} - \rvx^\ast\|^2 = \gO\left(D^2 \exp\left(-\mu \tilde \eta R \right) + \frac{\tilde\eta^2 L\zeta_\ast^2}{\mu M} \right),\\
		&\E\|\rvx^{(R)} - \rvx^\ast\|^2 = \tilde\gO\left(D^2 \exp\left(\frac{-\mu {\color{red}\boldsymbol{MK}}R}{L} \right) + \frac{ L\zeta_\ast^2}{\mu^3 M R^2} \right),
	\end{align*}
	where $\tilde \eta = \eta MK \lesssim \frac{1}{L} \cdot MK$. Here we set $N=K$, since the number of local steps equals the size of the local data set in \citet{malinovsky2023federated}.
\end{corollary}

\begin{corollary}[Corollary of Theorem~\ref{thm:SFL}]
	\label{cor:SFL-last-round-iterate-bound}
	For the purpose of comparison, we consider the bound of $\E\|\rvx^{(R)} - \rvx^\ast\|^2$ here instead of $\E\left[F(\bar \rvx^{(R)}) - F(\rvx^\ast) \right]$:
	\begin{align*}
		&\E\|\rvx^{(R)} - \rvx^\ast\|^2 = \gO\left(D^2 \exp\left(-\mu \tilde\eta R\right) + \frac{\tilde\eta^2 L \zeta_\ast^2}{\mu M}  \right),\\
		&\E\|\rvx^{(R)} - \rvx^\ast\|^2 = \tilde\gO\left(D^2 \exp\left(\frac{-\mu R}{L}\right) + \frac{ L \zeta_\ast^2}{\mu^3 M R^2}  \right),
	\end{align*}
	where $\tilde \eta = \eta MK \lesssim \frac{1}{L}$.
\end{corollary}

\vskip 0.2in
\bibliography{refs}

\begin{thebibliography}{58}
\providecommand{\natexlab}[1]{#1}
\providecommand{\url}[1]{\texttt{#1}}
\expandafter\ifx\csname urlstyle\endcsname\relax
  \providecommand{\doi}[1]{doi: #1}\else
  \providecommand{\doi}{doi: \begingroup \urlstyle{rm}\Url}\fi

\bibitem[Ahn et~al.(2020)Ahn, Yun, and Sra]{ahn2020sgd}
Kwangjun Ahn, Chulhee Yun, and Suvrit Sra.
\newblock {SGD} with shuffling: optimal rates without component convexity and
  large epoch requirements.
\newblock In \emph{Conference on Neural Information Processing Systems
  (NeurIPS)}, 2020.

\bibitem[Cha et~al.(2023)Cha, Lee, and Yun]{cha2023tighter}
Jaeyoung Cha, Jaewook Lee, and Chulhee Yun.
\newblock Tighter lower bounds for shuffling {SGD}: Random permutations and
  beyond.
\newblock In \emph{International Conference on Machine Learning (ICML)}, 2023.

\bibitem[Chang and Lin(2011)]{chang2011libsvm}
Chih-Chung Chang and Chih-Jen Lin.
\newblock {LIBSVM}: A library for support vector machines.
\newblock \emph{ACM transactions on intelligent systems and technology (TIST)},
  2011.

\bibitem[Chang et~al.(2018)Chang, Balachandar, Lam, Yi, Brown, Beers, Rosen,
  Rubin, and Kalpathy-Cramer]{chang2018distributed}
Ken Chang, Niranjan Balachandar, Carson Lam, Darvin Yi, James Brown, Andrew
  Beers, Bruce Rosen, Daniel~L Rubin, and Jayashree Kalpathy-Cramer.
\newblock Distributed deep learning networks among institutions for medical
  imaging.
\newblock \emph{Journal of the American Medical Informatics Association}, 2018.

\bibitem[Chen et~al.(2020)Chen, Chen, Zhou, and Kailkhura]{chen2020fedcluster}
Cheng Chen, Ziyi Chen, Yi~Zhou, and Bhavya Kailkhura.
\newblock {FedCluster}: Boosting the convergence of federated learning via
  cluster-cycling.
\newblock In \emph{{IEEE} International Conference on Big Data (Big Data)},
  2020.

\bibitem[Chen et~al.(2023)Chen, Li, Ni, Zhu, and Zhang]{chen2023fedseq}
Zhikun Chen, Daofeng Li, Rui Ni, Jinkang Zhu, and Sihai Zhang.
\newblock {FedSeq}: A hybrid federated learning framework based on sequential
  in-cluster training.
\newblock \emph{IEEE Systems Journal}, 2023.

\bibitem[Cho et~al.(2023)Cho, Sharma, Joshi, Xu, Kale, and
  Zhang]{cho2023convergence}
Yae~Jee Cho, Pranay Sharma, Gauri Joshi, Zheng Xu, Satyen Kale, and Tong Zhang.
\newblock On the convergence of federated averaging with cyclic client
  participation.
\newblock In \emph{International Conference on Machine Learning (ICML)}, 2023.

\bibitem[Darlow et~al.(2018)Darlow, Crowley, Antoniou, and
  Storkey]{darlow2018cinic}
Luke~N Darlow, Elliot~J Crowley, Antreas Antoniou, and Amos~J Storkey.
\newblock {CINIC-10} is not {ImageNet} or {CIFAR-10}.
\newblock \emph{arXiv preprint arXiv:1810.03505}, 2018.

\bibitem[Eichner et~al.(2019)Eichner, Koren, McMahan, Srebro, and
  Talwar]{eichner2019semi}
Hubert Eichner, Tomer Koren, Brendan McMahan, Nathan Srebro, and Kunal Talwar.
\newblock Semi-cyclic stochastic gradient descent.
\newblock In \emph{International Conference on Machine Learning (ICML)}, 2019.

\bibitem[Gao et~al.(2021)Gao, Kim, Thapa, Abuadbba, Zhang, Camtepe, Kim, and
  Nepal]{gao2021evaluation}
Yansong Gao, Minki Kim, Chandra Thapa, Sharif Abuadbba, Zhi Zhang, Seyit
  Camtepe, Hyoungshick Kim, and Surya Nepal.
\newblock Evaluation and optimization of distributed machine learning
  techniques for internet of things.
\newblock \emph{IEEE Transactions on Computers}, 2021.

\bibitem[Glasgow et~al.(2022)Glasgow, Yuan, and Ma]{glasgow2022sharp}
Margalit~R Glasgow, Honglin Yuan, and Tengyu Ma.
\newblock Sharp bounds for federated averaging (local {SGD}) and continuous
  perspective.
\newblock In \emph{International Conference on Artificial Intelligence and
  Statistics (AISTATS)}, 2022.

\bibitem[Gupta and Raskar(2018)]{gupta2018distributed}
Otkrist Gupta and Ramesh Raskar.
\newblock Distributed learning of deep neural network over multiple agents.
\newblock \emph{Journal of Network and Computer Applications}, 2018.

\bibitem[Horv{\'a}th et~al.(2022)Horv{\'a}th, Sanjabi, Xiao, Richt{\'a}rik, and
  Rabbat]{horvath2022fedshuffle}
Samuel Horv{\'a}th, Maziar Sanjabi, Lin Xiao, Peter Richt{\'a}rik, and Michael
  Rabbat.
\newblock {FedShuffle}: Recipes for better use of local work in federated
  learning.
\newblock \emph{Transactions on Machine Learning Research (TMLR)}, 2022.

\bibitem[Huang et~al.(2024)Huang, Bert, Gomaa, Fietkau, Maier, and
  Putz]{huang2024experimental}
Yixing Huang, Christoph Bert, Ahmed Gomaa, Rainer Fietkau, Andreas Maier, and
  Florian Putz.
\newblock An experimental survey of incremental transfer learning for
  multicenter collaboration.
\newblock \emph{IEEE Access}, 2024.

\bibitem[Jhunjhunwala et~al.(2023)Jhunjhunwala, Wang, and
  Joshi]{jhunjhunwala2023fedexp}
Divyansh Jhunjhunwala, Shiqiang Wang, and Gauri Joshi.
\newblock {FedExP}: Speeding up federated averaging via extrapolation.
\newblock In \emph{International Conference on Learning Representations
  (ICLR)}, 2023.

\bibitem[Karimireddy et~al.(2020)Karimireddy, Kale, Mohri, Reddi, Stich, and
  Suresh]{karimireddy2020scaffold}
Sai~Praneeth Karimireddy, Satyen Kale, Mehryar Mohri, Sashank Reddi, Sebastian
  Stich, and Ananda~Theertha Suresh.
\newblock {SCAFFOLD}: Stochastic controlled averaging for federated learning.
\newblock In \emph{International Conference on Machine Learning (ICML)}, 2020.

\bibitem[Khaled et~al.(2020)Khaled, Mishchenko, and
  Richt{\'a}rik]{khaled2020tighter}
Ahmed Khaled, Konstantin Mishchenko, and Peter Richt{\'a}rik.
\newblock Tighter theory for local sgd on identical and heterogeneous data.
\newblock In \emph{International Conference on Artificial Intelligence and
  Statistics (AISTATS)}, 2020.

\bibitem[Koloskova et~al.(2020)Koloskova, Loizou, Boreiri, Jaggi, and
  Stich]{koloskova2020unified}
Anastasia Koloskova, Nicolas Loizou, Sadra Boreiri, Martin Jaggi, and Sebastian
  Stich.
\newblock A unified theory of decentralized sgd with changing topology and
  local updates.
\newblock In \emph{International Conference on Machine Learning (ICML)}, 2020.

\bibitem[Koloskova et~al.(2024)Koloskova, Doikov, Stich, and
  Jaggi]{koloskova2024convergence}
Anastasia Koloskova, Nikita Doikov, Sebastian~U Stich, and Martin Jaggi.
\newblock On convergence of incremental gradient for non-convex smooth
  functions.
\newblock In \emph{International Conference on Machine Learning (ICML)}, 2024.

\bibitem[Krizhevsky et~al.(2009)]{krizhevsky2009learning}
Alex Krizhevsky et~al.
\newblock Learning multiple layers of features from tiny images.
\newblock \emph{Technical report}, 2009.

\bibitem[Lee et~al.(2020)Lee, Oh, Lim, Yun, and Lee]{lee2020tornadoaggregate}
Jin-woo Lee, Jaehoon Oh, Sungsu Lim, Se-Young Yun, and Jae-Gil Lee.
\newblock Tornadoaggregate: Accurate and scalable federated learning via the
  ring-based architecture.
\newblock \emph{arXiv preprint arXiv:2012.03214}, 2020.

\bibitem[Li and Richt{\'a}rik(2024)]{li2024convergence}
Hanmin Li and Peter Richt{\'a}rik.
\newblock On the convergence of {FedProx} with extrapolation and inexact prox.
\newblock In \emph{International OPT Workshop on Optimization for Machine
  Learning at NeurIPS 2024}, 2024.

\bibitem[Li et~al.(2024)Li, Acharya, and Richt{\'a}rik]{li2024power}
Hanmin Li, Kirill Acharya, and Peter Richt{\'a}rik.
\newblock The power of extrapolation in federated learning.
\newblock In \emph{Conference on Neural Information Processing Systems
  (NeurIPS)}, 2024.

\bibitem[Li et~al.(2020)Li, Huang, Yang, Wang, and Zhang]{li2020convergence}
Xiang Li, Kaixuan Huang, Wenhao Yang, Shusen Wang, and Zhihua Zhang.
\newblock On the convergence of {FedAvg} on non-{IID} data.
\newblock In \emph{International Conference on Learning Representations
  (ICLR)}, 2020.

\bibitem[Li and Lyu(2023)]{li2023convergence}
Yipeng Li and Xinchen Lyu.
\newblock Convergence analysis of sequential federated learning on
  heterogeneous data.
\newblock In \emph{Conference on Neural Information Processing Systems
  (NeurIPS)}, 2023.

\bibitem[Lin et~al.(2020)Lin, Kong, Stich, and Jaggi]{lin2020ensemble}
Tao Lin, Lingjing Kong, Sebastian~U Stich, and Martin Jaggi.
\newblock Ensemble distillation for robust model fusion in federated learning.
\newblock In \emph{Conference on Neural Information Processing Systems
  (NeurIPS)}, 2020.

\bibitem[Lu et~al.(2022)Lu, Meng, and Sa]{lu2022general}
Yucheng Lu, Si~Yi Meng, and Christopher~De Sa.
\newblock A general analysis of example-selection for stochastic gradient
  descent.
\newblock In \emph{International Conference on Learning Representations
  (ICLR)}, 2022.

\bibitem[Malinovsky et~al.(2023)Malinovsky, Horv{\'a}th, Burlachenko, and
  Richt{\'a}rik]{malinovsky2023federated}
Grigory Malinovsky, Samuel Horv{\'a}th, Konstantin~Pavlovich Burlachenko, and
  Peter Richt{\'a}rik.
\newblock Federated learning with regularized client participation.
\newblock In \emph{Federated Learning and Analytics in Practice: Algorithms,
  Systems, Applications, and Opportunities workshop at ICML 2023}, 2023.

\bibitem[McMahan et~al.(2017)McMahan, Moore, Ramage, Hampson, and
  y~Arcas]{mcmahan2017communication}
Brendan McMahan, Eider Moore, Daniel Ramage, Seth Hampson, and Blaise~Aguera
  y~Arcas.
\newblock Communication-efficient learning of deep networks from decentralized
  data.
\newblock In \emph{International Conference on Artificial Intelligence and
  Statistics (AISTATS)}, 2017.

\bibitem[Mishchenko et~al.(2020)Mishchenko, Khaled, and
  Richt{\'a}rik]{mishchenko2020random}
Konstantin Mishchenko, Ahmed Khaled, and Peter Richt{\'a}rik.
\newblock Random reshuffling: Simple analysis with vast improvements.
\newblock In \emph{Conference on Neural Information Processing Systems
  (NeurIPS)}, 2020.

\bibitem[Mishchenko et~al.(2022{\natexlab{a}})Mishchenko, Khaled, and
  Richt{\'a}rik]{mishchenko2022proximal}
Konstantin Mishchenko, Ahmed Khaled, and Peter Richt{\'a}rik.
\newblock Proximal and federated random reshuffling.
\newblock In \emph{International Conference on Machine Learning (ICML)},
  2022{\natexlab{a}}.

\bibitem[Mishchenko et~al.(2022{\natexlab{b}})Mishchenko, Malinovsky, Stich,
  and Richt{\'a}rik]{mishchenko2022proxskip}
Konstantin Mishchenko, Grigory Malinovsky, Sebastian Stich, and Peter
  Richt{\'a}rik.
\newblock {Proxskip}: Yes! local gradient steps provably lead to communication
  acceleration! finally!
\newblock In \emph{International Conference on Machine Learning (ICML)},
  2022{\natexlab{b}}.

\bibitem[Mortici(2011)]{mortici2011gospers}
Cristinel Mortici.
\newblock On {Gospers} formula for the {Gamma} function.
\newblock \emph{Journal of Mathematical Inequalities}, 2011.

\bibitem[Nagaraj et~al.(2019)Nagaraj, Jain, and Netrapalli]{nagaraj2019sgd}
Dheeraj Nagaraj, Prateek Jain, and Praneeth Netrapalli.
\newblock {SGD} without replacement: Sharper rates for general smooth convex
  functions.
\newblock In \emph{International Conference on Machine Learning (ICML)}, 2019.

\bibitem[Nemirovskij and Yudin(1983)]{nemirovskij1983problem}
Arkadij~Semenovi{\v{c}} Nemirovskij and David~Borisovich Yudin.
\newblock \emph{Problem complexity and method efficiency in optimization}.
\newblock Wiley-Interscience, 1983.

\bibitem[Nguyen et~al.(2021)Nguyen, Tran-Dinh, Phan, Nguyen, and
  Van~Dijk]{nguyen2021unified}
Lam~M Nguyen, Quoc Tran-Dinh, Dzung~T Phan, Phuong~Ha Nguyen, and Marten
  Van~Dijk.
\newblock A unified convergence analysis for shuffling-type gradient methods.
\newblock \emph{Journal of Machine Learning Research (JMLR)}, 2021.

\bibitem[Patel et~al.(2023)Patel, Glasgow, Wang, Joshi, and
  Srebro]{patel2023still}
Kumar~Kshitij Patel, Margalit Glasgow, Lingxiao Wang, Nirmit Joshi, and Nathan
  Srebro.
\newblock On the still unreasonable effectiveness of federated averaging for
  heterogeneous distributed learning.
\newblock In \emph{Federated Learning and Analytics in Practice: Algorithms,
  Systems, Applications, and Opportunities workshop at ICML 2023}, 2023.

\bibitem[Patel et~al.(2024)Patel, Glasgow, Zindari, Wang, Stich, Cheng, Joshi,
  and Srebro]{patel2024limits}
Kumar~Kshitij Patel, Margalit Glasgow, Ali Zindari, Lingxiao Wang, Sebastian~U
  Stich, Ziheng Cheng, Nirmit Joshi, and Nathan Srebro.
\newblock The limits and potentials of local {SGD} for distributed
  heterogeneous learning with intermittent communication.
\newblock In \emph{Conference on Learning Theory (COLT)}, 2024.

\bibitem[Rajput et~al.(2020)Rajput, Gupta, and
  Papailiopoulos]{rajput2020closing}
Shashank Rajput, Anant Gupta, and Dimitris Papailiopoulos.
\newblock Closing the convergence gap of {SGD} without replacement.
\newblock In \emph{International Conference on Machine Learning (ICML)}, 2020.

\bibitem[Reddi et~al.(2021)Reddi, Charles, Zaheer, Garrett, Rush,
  Kone{\v{c}}n{\'y}, Kumar, and McMahan]{reddi2021adaptive}
Sashank~J. Reddi, Zachary Charles, Manzil Zaheer, Zachary Garrett, Keith Rush,
  Jakub Kone{\v{c}}n{\'y}, Sanjiv Kumar, and Hugh~Brendan McMahan.
\newblock Adaptive federated optimization.
\newblock In \emph{International Conference on Learning Representations
  (ICLR)}, 2021.

\bibitem[Sadiev et~al.(2023)Sadiev, Malinovsky, Gorbunov, Sokolov, Khaled,
  Burlachenko, and Richt{\'a}rik]{sadiev2023federated}
Abdurakhmon Sadiev, Grigory Malinovsky, Eduard Gorbunov, Igor Sokolov, Ahmed
  Khaled, Konstantin~Pavlovich Burlachenko, and Peter Richt{\'a}rik.
\newblock Federated optimization algorithms with random reshuffling and
  gradient compression.
\newblock In \emph{Federated Learning and Analytics in Practice: Algorithms,
  Systems, Applications, and Opportunities workshop at ICML 2023}, 2023.

\bibitem[Safran and Shamir(2020)]{safran2020good}
Itay Safran and Ohad Shamir.
\newblock How good is {SGD} with random shuffling?
\newblock In \emph{Conference on Learning Theory (COLT)}, 2020.

\bibitem[Safran and Shamir(2021)]{safran2021random}
Itay Safran and Ohad Shamir.
\newblock Random shuffling beats {SGD} only after many epochs on
  ill-conditioned problems.
\newblock In \emph{Conference on Neural Information Processing Systems
  (NeurIPS)}, 2021.

\bibitem[Simonyan and Zisserman(2014)]{simonyan2014very}
Karen Simonyan and Andrew Zisserman.
\newblock Very deep convolutional networks for large-scale image recognition.
\newblock \emph{arXiv preprint arXiv:1409.1556}, 2014.

\bibitem[Stich(2019)]{stich2019local}
Sebastian~U. Stich.
\newblock Local {SGD} converges fast and communicates little.
\newblock In \emph{International Conference on Learning Representations
  (ICLR)}, 2019.

\bibitem[Thapa et~al.(2022)Thapa, Arachchige, Camtepe, and
  Sun]{thapa2022splitfed}
Chandra Thapa, Pathum Chamikara~Mahawaga Arachchige, Seyit Camtepe, and Lichao
  Sun.
\newblock {SplitFed}: When federated learning meets split learning.
\newblock In \emph{AAAI Conference on Artificial Intelligence (AAAI)}, 2022.

\bibitem[Wang et~al.(2020)Wang, Liu, Liang, Joshi, and Poor]{wang2020tackling}
Jianyu Wang, Qinghua Liu, Hao Liang, Gauri Joshi, and H~Vincent Poor.
\newblock Tackling the objective inconsistency problem in heterogeneous
  federated optimization.
\newblock In \emph{Conference on Neural Information Processing Systems
  (NeurIPS)}, 2020.

\bibitem[Wang et~al.(2024)Wang, Das, Joshi, Kale, Xu, and
  Zhang]{wang2024unreasonable}
Jianyu Wang, Rudrajit Das, Gauri Joshi, Satyen Kale, Zheng Xu, and Tong Zhang.
\newblock On the unreasonable effectiveness of federated averaging with
  heterogeneous data.
\newblock \emph{Transactions on Machine Learning Research (TMLR)}, 2024.

\bibitem[Wang and Ji(2022)]{wang2022unified}
Shiqiang Wang and Mingyue Ji.
\newblock A unified analysis of federated learning with arbitrary client
  participation.
\newblock In \emph{Conference on Neural Information Processing Systems
  (NeurIPS)}, 2022.

\bibitem[Woodworth et~al.(2020{\natexlab{a}})Woodworth, Patel, Stich, Dai,
  Bullins, Mcmahan, Shamir, and Srebro]{woodworth2020local}
Blake Woodworth, Kumar~Kshitij Patel, Sebastian Stich, Zhen Dai, Brian Bullins,
  Brendan Mcmahan, Ohad Shamir, and Nathan Srebro.
\newblock Is local {SGD} better than minibatch {SGD}?
\newblock In \emph{International Conference on Machine Learning (ICML)},
  2020{\natexlab{a}}.

\bibitem[Woodworth et~al.(2020{\natexlab{b}})Woodworth, Patel, and
  Srebro]{woodworth2020minibatch}
Blake~E Woodworth, Kumar~Kshitij Patel, and Nati Srebro.
\newblock Minibatch vs local {SGD} for heterogeneous distributed learning.
\newblock In \emph{Conference on Neural Information Processing Systems
  (NeurIPS)}, 2020{\natexlab{b}}.

\bibitem[Xiao et~al.(2017)Xiao, Rasul, and Vollgraf]{xiao2017fashion}
Han Xiao, Kashif Rasul, and Roland Vollgraf.
\newblock {Fashion-MNIST}: A novel image dataset for benchmarking machine
  learning algorithms.
\newblock \emph{arXiv preprint arXiv:1708.07747}, 2017.

\bibitem[Yan et~al.(2024)Yan, Zuo, Fan, Hu, Shen, Zhao, and
  Luo]{yan2024sequential}
Xingrun Yan, Shiyuan Zuo, Rongfei Fan, Han Hu, Li~Shen, Puning Zhao, and Yong
  Luo.
\newblock Sequential federated learning in hierarchical architecture on
  non-{IID} datasets.
\newblock \emph{arXiv preprint arXiv:2408.09762}, 2024.

\bibitem[Yang et~al.(2021)Yang, Fang, and Liu]{yang2021achieving}
Haibo Yang, Minghong Fang, and Jia Liu.
\newblock Achieving linear speedup with partial worker participation in
  non-{IID} federated learning.
\newblock In \emph{International Conference on Learning Representations
  (ICLR)}, 2021.

\bibitem[Yuan et~al.(2023)Yuan, Ma, Su, and Wang]{yuan2023peer}
Liangqi Yuan, Yunsheng Ma, Lu~Su, and Ziran Wang.
\newblock Peer-to-peer federated continual learning for naturalistic driving
  action recognition.
\newblock In \emph{IEEE/CVF Conference on Computer Vision and Pattern
  Recognition (CVPR)}, 2023.

\bibitem[Yuan et~al.(2024)Yuan, Wang, Sun, Philip, and
  Brinton]{yuan2024decentralized}
Liangqi Yuan, Ziran Wang, Lichao Sun, S~Yu Philip, and Christopher~G Brinton.
\newblock Decentralized federated learning: A survey and perspective.
\newblock \emph{IEEE Internet of Things Journal}, 2024.

\bibitem[Yun et~al.(2022)Yun, Rajput, and Sra]{yun2022minibatch}
Chulhee Yun, Shashank Rajput, and Suvrit Sra.
\newblock Minibatch vs local {SGD} with shuffling: Tight convergence bounds and
  beyond.
\newblock In \emph{International Conference on Learning Representations
  (ICLR)}, 2022.

\bibitem[Zaccone et~al.(2022)Zaccone, Rizzardi, Caldarola, Ciccone, and
  Caputo]{zaccone2022speeding}
Riccardo Zaccone, Andrea Rizzardi, Debora Caldarola, Marco Ciccone, and Barbara
  Caputo.
\newblock Speeding up heterogeneous federated learning with sequentially
  trained superclients.
\newblock In \emph{International Conference on Pattern Recognition (ICPR)},
  2022.

\end{thebibliography}

\end{document}